\documentclass[twoside,11pt]{article}
\usepackage{jair, theapa, rawfonts}

\usepackage{graphicx}
\usepackage{booktabs}
\usepackage{amssymb}
\usepackage{amsmath}
\usepackage{diagbox}
\usepackage{siunitx} 
\usepackage[linesnumbered,ruled,vlined,algosection]{algorithm2e}
\usepackage{graphicx}
\usepackage{soul}
\usepackage{url}
\usepackage{todonotes}

\usepackage{amsthm}
\usepackage{xfrac}
\usepackage{multirow}
\usepackage{comment}
\usepackage{wrapfig}
\usepackage{enumitem}

\usepackage[T1]{fontenc} 

\DeclareMathOperator{\E}{\mathbb{E}}
\DeclareMathOperator{\I}{I}
\DeclareMathOperator{\h}{H}

\newtheorem{lemma}{Lemma}
\numberwithin{lemma}{section}

\newtheorem{theorem}{Theorem}
\newtheorem{property}{Property}
\numberwithin{property}{section}
\newtheorem*{theorem*}{Theorem}
\theoremstyle{remark}
\newtheorem{remark}{Remark}

\newcommand{\vn}[1]{{\color{blue}{VN:#1}}}

\jairheading{72}{2021}{1343-1384}{03/2021}{12/2021}
\ShortHeadings{Language Modeling and Linguistic Structures}
{Nikoulina, Tezekbayev, Kozhakhmet, Babazhanova, Gall\'e, \& Assylbekov}
\firstpageno{1343}

\begin{document}



\title{The Rediscovery Hypothesis: \\ Language Models Need to Meet Linguistics}

\author{\name Vassilina Nikoulina \email vassilina.nikoulina@naverlabs.com \\
\addr NAVER LABS Europe\\
6-8 chemin de Maupertuis, 38240 Meylan, France
\AND
\name Maxat Tezekbayev \email maxat.tezekbayev@nu.edu.kz \\
\name Nuradil Kozhakhmet \email nuradil.kozhakhmet@nu.edu.kz
\\
\name Madina Babazhanova \email madina.babazhanova@nu.edu.kz \\
\addr 
Nazarbayev University\\
53 Kabanbay Batyr ave., Nur-Sultan Z05H0P9, Kazakhstan
\AND 
\name Matthias Gall\'e \email matthias.galle@naverlabs.com \\
\addr NAVER LABS Europe\\
6-8 chemin de Maupertuis, 38240 Meylan, France
\AND
\name Zhenisbek Assylbekov \email zhassylbekov@nu.edu.kz\\
\addr 
Nazarbayev University\\
53 Kabanbay Batyr ave., Nur-Sultan Z05H0P9, Kazakhstan
}


\maketitle

\begin{abstract}
There is an ongoing debate in the NLP community whether modern language models contain linguistic knowledge, recovered through so-called \textit{probes}.
In this paper, we study whether 
linguistic knowledge is a necessary condition for the good performance of modern language models, which we call the \textit{rediscovery hypothesis}.

In the first place, we show that language models that are significantly compressed but perform well on their pretraining objectives retain good scores when probed for linguistic structures. This result supports the rediscovery hypothesis and leads to the second contribution of our paper: an information-theoretic framework that relates language modeling objectives with linguistic information. This framework also provides a metric to measure the impact of linguistic information on the word prediction task. We reinforce our analytical results with various experiments, both on synthetic and on real NLP tasks in English.
\end{abstract}

\section{Introduction}

Vector representations of words obtained from self-supervised pretraining of neural language models (LMs) on massive unlabeled data have revolutionized NLP in the last decade. This success has spurred an interest in trying to understand what type of knowledge these models actually learn~\shortcite{rogers2020primer,petroni2019}. 

Of particular interest here is ``linguistic knowledge'', which is generally measured by annotating test sets through experts (linguists) following certain pre-defined linguistic schemas.
These annotation schemas are based on language regularities manually defined by linguists. 
On the other side, we have language models which are pretrained predictors that assign a probability to the presence of a token given the surrounding context.
Neural models solve this task by finding patterns in the text.
We refer to the claim that such neural language models rediscover linguistic knowledge as the \textit{rediscovery hypothesis}.
It stipulates that the patterns of the language discovered by the model trying to solve the pretraining task correlate with the human-defined linguistic regularity.
In this work we measure the amount of linguistics rediscovered by a pretrained model through the so-called probing tasks: hidden layers of a neural LM are fed to a simple classifier---a probe---that learns to predict a linguistic structure of interest~\shortcite{ettinger2016,adi2017fine,conneau2018you,hewitt2019structural,tenney2019bert}.
In the first part of this paper (Section~\ref{sec:lth_probing}) we attempt to challenge the \textit{rediscovery hypothesis} through a variety of experiments to understand to what extent it holds. 
Those experiments aim to verify whether the path through language regularity is indeed the one taken by pretrained LMs or whether there is another way to reach good LM performance without rediscovering linguistic structure. 
Our experiments show that pretraining loss is indeed tightly linked to the amount of linguistic structure discovered by an LM. We, therefore, fail to reject the rediscovery hypothesis. 

This negative attempt, as well as the abundance of positive examples in the literature, motivates us to prove mathematically the rediscovery hypothesis.
In the second part of our paper (Section~\ref{sec:theory}) we use information theory to prove the contrapositive of the hypothesis---removal of linguistic information from an LM degrades its performance. Moreover, we show that the decline in the LM quality depends on how strongly the removed property is interdependent with the underlying text: a greater dependence leads to a greater drop.
We confirm this result empirically, both with synthetic data and with real annotations on English text.

The result that removing information that contains strong mutual information with the underlying text degrades (masked) word prediction might not seem surprising \textit{a posteriori}.
However, it is this surprise that lies at the heart of most of the work in recent years around the discovery of how easily this information can be extracted from intermediate representations. 
Our framework also provides a coefficient that measures the dependence between a probing task and the underlying text.
This measure can be used to determine more complex probing tasks, whose rediscovery by language models would indeed be surprising.

\section{Do Pretraining Objectives Correlate with Linguistic Structures?}\label{sec:lth_probing}

The first question we pose in this work is the following: is the rediscovery of linguistic knowledge mandatory for models that perform well on their pretraining tasks, typically language modeling or translation; or is it a side effect of overparameterization?\footnote{Overparamterization is defined informally as ``having more parameters than can be estimated from the data'', and therefore using a model richer than necessary for the task at hand. Those additional parameters could be responsible for the good performance of the probes.}

We analyze the correlation between linguistic knowledge and LM performance with pruned pretrained models. By compressing a network through pruning we retain the same overall architecture and can compare probing methods.
More important, we hypothesize that pruning removes all unnecessary information with respect to the pruning objective (language modeling) and that it might be that information that is used to rediscover linguistic knowledge. To complete this experiment, we also track the pruning efficiency \textit{during} pretraining.



\subsection{Pruning Method}

The lottery ticket hypothesis (LTH) of \shortciteA{frankle2018the} claims that a randomly-initialized neural network $f(\mathbf{x};\boldsymbol\theta)$ with trainable parameters $\boldsymbol\theta\in\mathbb{R}^n$ contains subnetworks $f(\mathbf{x};\mathbf{m}\odot\boldsymbol\theta)$, $\mathbf{m}\in\{0,1\}^n$, such that, when trained in isolation, they can match the performance of the original network.\footnote{$\|\mathbf{x}\|_0$ is a number of nozero elements in $\mathbf{x}\in\mathbb{R}^d$, $\odot$ is the element-wise multiplication.} The authors suggest a simple procedure for identifying such subnetworks (Alg.~\ref{alg:lth}). 
When this procedure is applied iteratively (Step~\ref{step:iter} of Alg.~\ref{alg:lth}), we get a sequence of pruned models $\{f(\mathbf{x};\mathbf{m}_i\odot\boldsymbol\theta_0)\}$ in which each model has fewer parameters than its predecessor: $\|\mathbf{m}_{i}\|_0<\|\mathbf{m}_{i-1}\|_0$. 
\citeauthor{frankle2018the} used iterative pruning for image classification models and found subnetworks that were 10\%--20\% of the sizes of the original networks and met or exceeded their validation accuracies. 
Such a compression approach retains weights important for the main task while discarding others.
We hypothesize that it might be those additional weights that contain the signals used by probes. 
But, if the subnetworks retain linguistic knowledge then this is evidence in favor of the rediscovery hypothesis.


\begin{algorithm*}[htbp]
\DontPrintSemicolon
Randomly initialize a neural network $f(\mathbf{x};\boldsymbol\theta_0)$, $\boldsymbol\theta_0\in\mathbb{R}^n$\;
Train the network for $j$ iterations, arriving at parameters $\boldsymbol\theta_j$.\label{step:train}\;
Prune $p\%$ of the parameters in $\boldsymbol\theta_j$, creating a mask $\boldsymbol{m}\in\{0,1\}^n$.\;
Reset the remaining parameters to their values in $\boldsymbol\theta_0$, creating the winning ticket $f(\mathbf{x}; \boldsymbol{m}\odot\boldsymbol{\theta}_0)$.\;
Repeat from \ref{step:train} if performing iterative pruning.\label{step:iter}\;
Train the winning ticket $f(\mathbf{x};\mathbf{m}\odot\boldsymbol\theta_0)$ to convergence.
\caption{Lottery ticket hypothesis---Identifying winning tickets \shortcite{frankle2018the}}
\label{alg:lth}
\end{algorithm*}

\subsection{Models}
\label{sec:models}
We explore one static embedding model, {\sc SGNS}, and two contextualized embedding models, {\sc CoVe} and {\sc RoBERTa}. In this manner, we cover the full spectrum of modern representations of words, from static embeddings to shallow and deep contextual embeddings.

\paragraph{\textsc{CoVe} \shortcite{mccann2017learned}} uses the top-level activations of a two-layer BiLSTM encoder from an attentional sequence-to-sequence model \shortcite{bahdanau2014neural} trained for English-to-German translation. The authors used the CommonCrawl-840B {\sc GloVe} model \shortcite{pennington2014glove} for English word vectors, which were completely fixed
during pretraining, and we follow their setup. This entails that the embedding layer on the source side is not pruned during the LTH procedure. We also concatenate the encoder output with the {\sc GloVe} embeddings as is done in the original paper \shortcite[Eq.~6]{mccann2017learned}. 

\paragraph{BERT \shortcite{devlin2019bert}} is a deep Transformer \shortcite{vaswani2017attention} encoder that has become the \textit{de facto} standard when it comes to contextualized embeddings. We pretrain the {\sc RoBERTa} variant of the {\sc BERT} model \shortcite{liu2019roberta}. {\sc RoBERTa} stands for robustly optimized {\sc BERT} which was trained with hyperparameters optimized for convergence, with dynamic masking instead of static masking, and restricted to masked LM objective only.\footnote{{\sc BERT} has next sentence prediction loss in addition.} Unlike their predecessors, such as {\sc CoVe} and {\sc ELMo} \shortcite{peters2018deep}, {\sc BERT} and other Transformer-based encoders are considered deep contextualizers.


\paragraph{Word2vec \textsc{SGNS} \shortcite{mikolov2013distributed}} is a shallow two-layer neural network that produces uncontextualized word embeddings. It is widely accepted that the {\sc SGNS} vectors capture words semantics to a certain extent, which is confirmed by the folklore examples, such as $\mathbf{w}_{\text{king}} - \mathbf{w}_{\text{man}} + \mathbf{w}_{\text{woman}} \approx \mathbf{w}_{\text{queen}}$. So, the question of the relationship between the {\sc SGNS} objective function and its ability to discover linguistics is also relevant.

\subsection{Measuring the Amount of Linguistic Knowledge}
\label{sec:probes}

To properly measure the \textit{amount of linguistic knowledge} in word vectors we define it as the performance of classifiers (probes) that take those vectors as input and are trained on linguistically annotated data. This definition has the advantage of being able to be measured exactly, at the cost of avoiding the discussion of whether POS tags or syntactic parse trees indeed denote linguistic knowledge captured by humans in their learning process.

We should note that this probing approach has received a lot of criticism recently \shortcite{hewitt2019designing,pimentel2020information,voita2020information} due to its inability to distinguish between information encoded in the pretrained vectors from the information learned by the probing classifier. However, in our study, the question is {not} \textit{How much linguistics is encoded in the presentation vector?}, but rather \textit{Does one vector contain more linguistic information than the other?} We compare different representations of the \textit{same} dimensionality using probing classifiers of the \textit{same} capacity. Even if part of the probing performance is due to the classifier itself we claim that the difference in the probing performance will be due to the difference in the amount of linguistic knowledge encoded in the representations we manipulate. This conjecture is strengthened by the findings of \shortciteA{zhang2020billions} who
analyzed the representations from pretrained {\sc miniBERTa}s\footnote{{\sc RoBERTa} models trained on a varied amount of training data} and demonstrated that the trends found through edge probing \shortcite{tenney2018what} are the same as those found through better-designed probes such as Minimum Description Length \shortcite{voita2020information}. Therefore in our work we adopt \textit{edge probing} and \textit{structural probing} for contextualized embeddings. For static embeddings, we use the traditional \textit{word similarity} and \textit{word analogy} tasks.

\paragraph{Edge probing \shortcite{tenney2018what}} formulates several linguistics tasks of different nature as text span classification tasks. The probing model is a lightweight classifier on top of the pretrained representations trained to solve those linguistic tasks. In our study, we use the part-of-speech tagging (POS), constituent labeling, named entity labeling (NE), and semantic role labeling (SRL) tasks from the suite, in which a probing classifier receives a sequence of tokens and predicts a label for it. 
For example, in the case of constituent labeling, for a sentence 
\textit{This probe }$[\textit{discovers linguistic knowledge}]$,
the sequence in square brackets should be labeled as a verb phrase. 

\paragraph{Structural probing \shortcite{hewitt2019structural}} evaluates whether syntax trees are embedded in a linear transformation of a
neural network's word representation space. The probe identifies a linear transformation under which squared Euclidean distance encodes the distance between words in the parse tree. \citeauthor{hewitt2019structural} show
that such transformations exist for both {\sc ELMo}
and {\sc BERT} but not in static baselines, providing
evidence that entire syntax trees can be easily extracted from the vector geometry of deep models.

\paragraph{Word similarity \shortcite{finkelstein2002placing} and word analogy \shortcite{DBLP:conf/naacl/MikolovYZ13}} tasks can be considered as nonparametric probes of static embeddings, and---differently from the other probing tasks---are not learned. The use of word embeddings in the word similarity task has been criticized for the instability of the results obtained \shortcite{DBLP:journals/tacl/AntoniakM18}. Regarding the word analogy task, \shortciteA{DBLP:conf/naacl/Schluter18} raised concerns on the misalignment of assumptions in generating and
testing word embeddings. However, the success of the static embeddings in performing well in these tasks was a crucial part of their widespread adoption.

\subsection{Experimental Setup}

We prune the embedding models from Section~\ref{sec:models} with the LTH algorithm (Alg.~\ref{alg:lth}) and evaluate them with probes from Section~\ref{sec:probes} at \textit{each} pruning iteration. {\sc CoVe} and {\sc SGNS} are pruned iteratively, while for {\sc RoBERTa} we perform one-shot pruning at different rates.\footnote{This is done to speedup experiments on {\sc RoBERTa}, as one-shot pruning at different rates can be run in parallel.} Assuming that $\ell^{\omega}_{i}$ is the validation loss of the embedding model $\omega\in\{${\sc CoVe}, {\sc RoBERTa}, {\sc SGNS}$\}$ at iteration $i$, and $\Delta s^{\omega,T}_{i}:=s^{\omega,T}_{i}-s^{\omega,T}_{0}$ is the drop in the corresponding score on the probing task $T\in\{$NE, POS, Const., Struct., Sim., Analogy$\}$ compared to the score $s^{\omega,T}_{0}$ of the baseline (unpruned) model, we obtain pairs $(\ell^{{\omega}}_{i}, \Delta s^{{\omega},{T}}_{i})$ for further analysis. Keep in mind that {\sc RoBERTa} and {\sc SGNS} are pruned in full, while in {\sc CoVe} we prune everything except the source-side embedding layer. This exception is due to the design of the {\sc CoVe} model, and we follow the original paper's setup \shortcite{mccann2017learned}.

\paragraph{Software and datasets.} We pretrain {\sc CoVe} on the English--German part of the IWSLT 2016 machine translation task \shortcite{cettolo2016iwslt} using the {\sc OpenNMT-py} toolkit \shortcite{klein2017opennmt}. 
{\sc RoBERTa} is pretrained on the WikiText-103 dataset (English)~\shortcite{merity2016pointer} 
using the {\sc fairseq} toolkit \shortcite{ott2019fairseq} with default training settings \shortcite{RoBERTa}. Finally, the {\sc SGNS} model is pretrained on the \texttt{text8} data (English as well) \shortcite{text8} 
using our custom implementation \shortcite{sgns}.

The edge probing classifier is trained on the standard benchmark dataset OntoNotes 5.0 \shortcite{weischedel2013ontonotes} using the {\sc jiant} toolkit \shortcite{DBLP:conf/acl/PruksachatkunYL20}. The structural probe is trained on the English UD \shortcite{silveira14gold} using the code from the authors \shortcite{structprobe}. For word similarities we use the WordSim353 dataset \shortcite{finkelstein2002placing}, while for word analogies we use the Google dataset \shortcite{mikolov2013efficient}.
All those datasets are in English.

\paragraph{Optimization} is performed in almost the same way as in the original works on {\sc CoVe}, {RoBERTa}, and {SGNS}. See Appendix~\ref{app:optimization} for details. 

\subsection{Results}
\label{sec:results1}
First, we note that the lottery ticket hypothesis is confirmed for the embedding models since pruning up to 60\% weights does not harm their performance significantly on held-out data (Fig.~\ref{fig:lth_results}).\footnote{In the case of {\sc SGNS}, pruning up to 80\% of weights does not affect its validation loss. Since solving {\sc SGNS} objective is essentially a factorization of the pointwise mutual information matrix, in the form $\mathrm{PMI}-\log k\approx\mathbf{WC}$ \cite{DBLP:conf/nips/LevyG14}, this means that a factorization with sparse $\mathbf{W}$ and $\mathbf{C}$ is possible. 
This observation complements the findings of \shortciteA{DBLP:conf/aaai/TissierGH19} who showed that near-to-optimal factorization is possible with binary $\mathbf{W}$ and $\mathbf{C}$.} 
\begin{figure}[htbp]
\begin{minipage}{\textwidth}
\includegraphics[width=.33\textwidth]{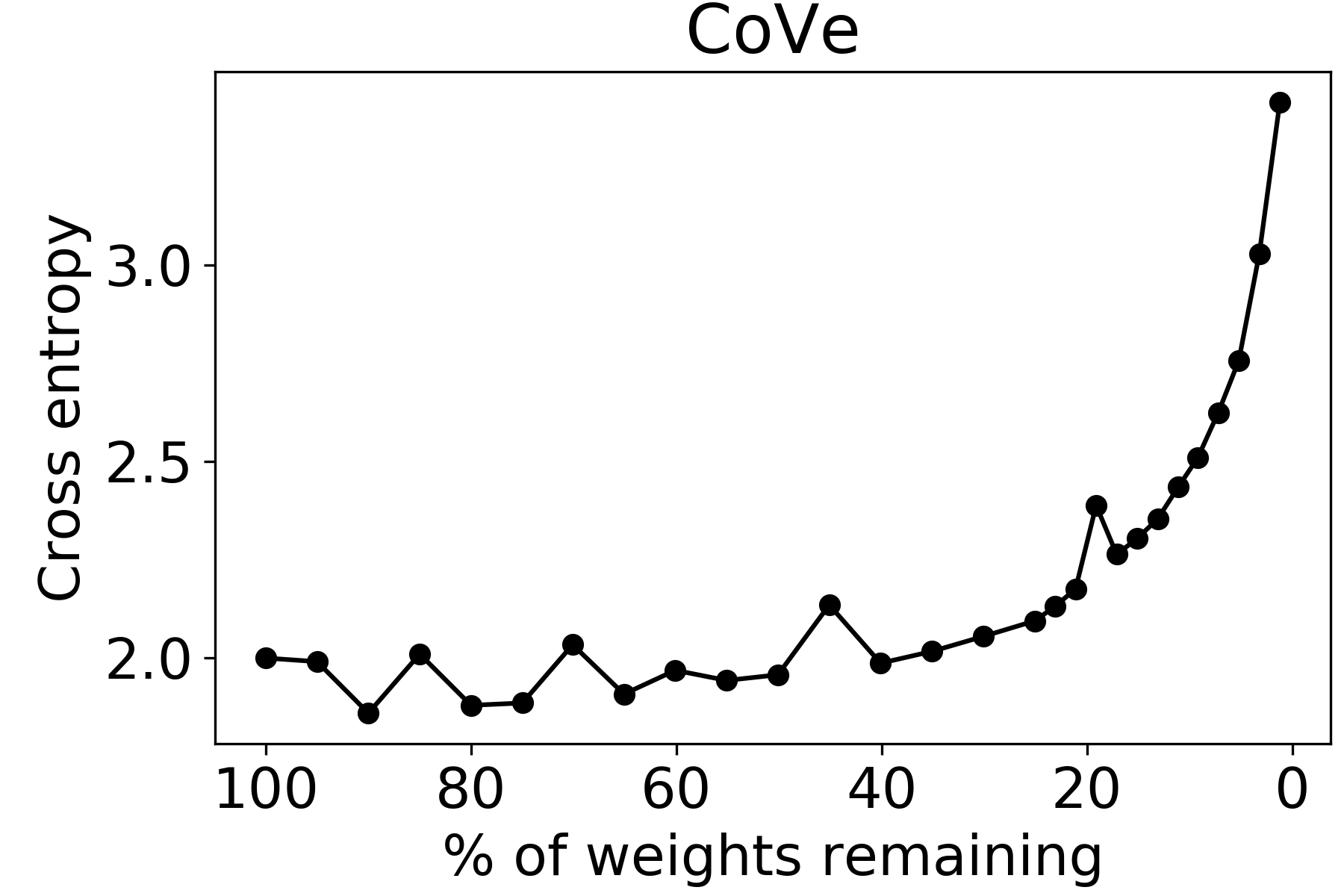}\hfill
\includegraphics[width=.33\textwidth]{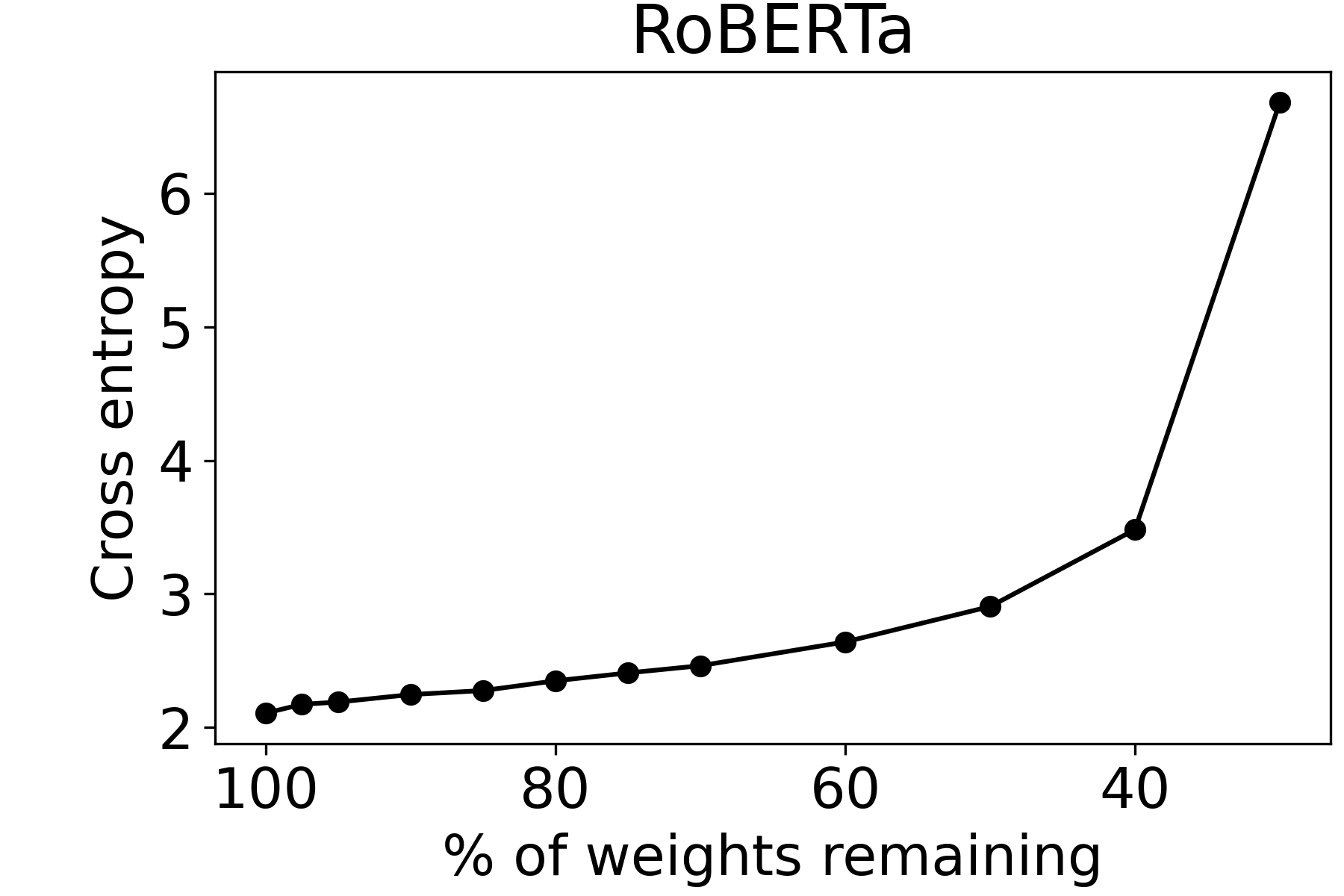}\hfill
\includegraphics[width=.33\textwidth]{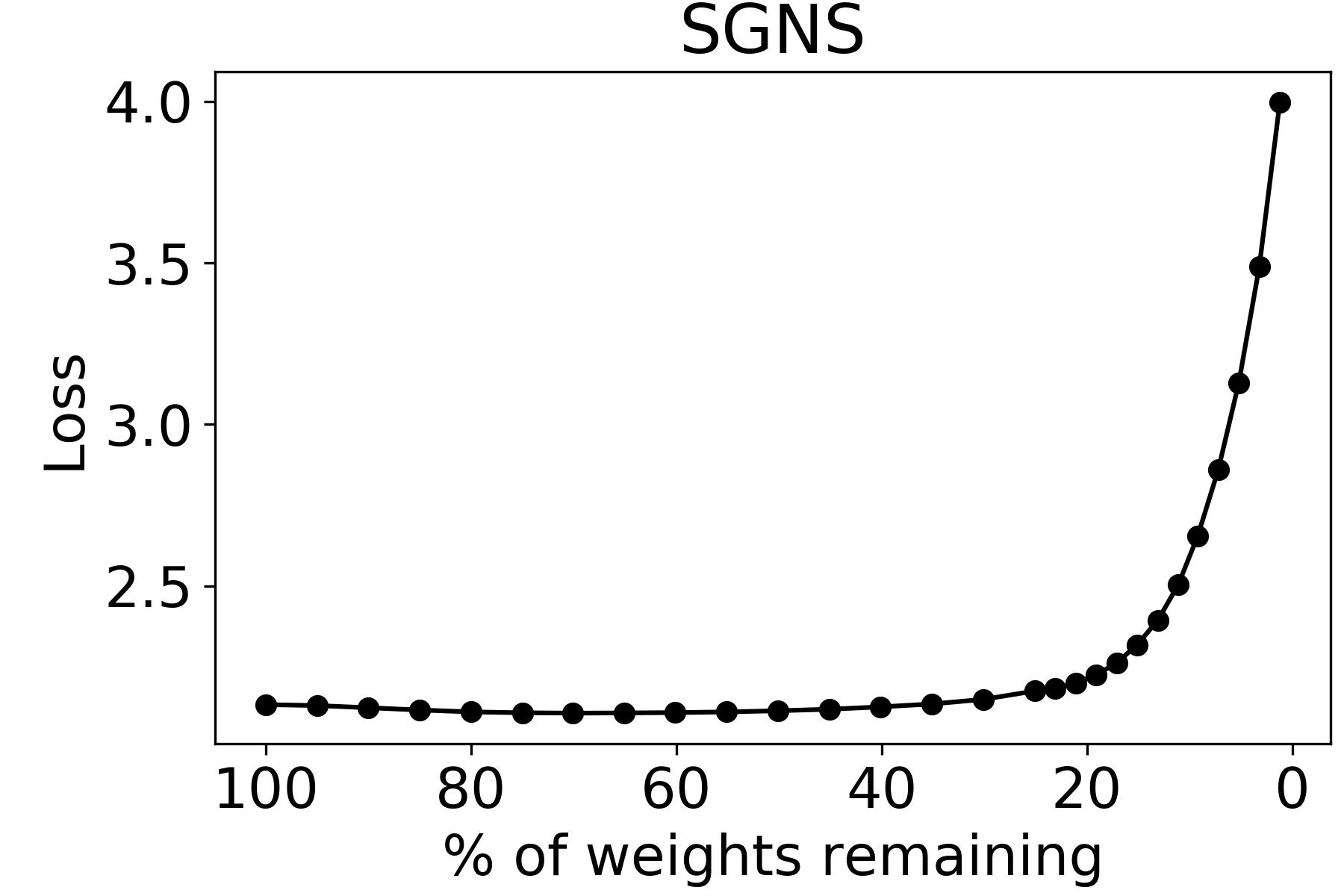}

\caption{Results of applying the LTH (Algorithm~\ref{alg:lth}) to {\sc CoVe}, {\sc RoBERTa}, and {\sc SGNS}. In each case we scatter-plot the percentage of remained weights vs validation loss, which is cross entropy for {\sc CoVe} and {\sc RoBERTa}, and a variant of negative sampling objective for {\sc SGNS}.}
\label{fig:lth_results}
\end{minipage}

\begin{minipage}{\textwidth}
\begin{center}
\includegraphics[width=.45\textwidth]{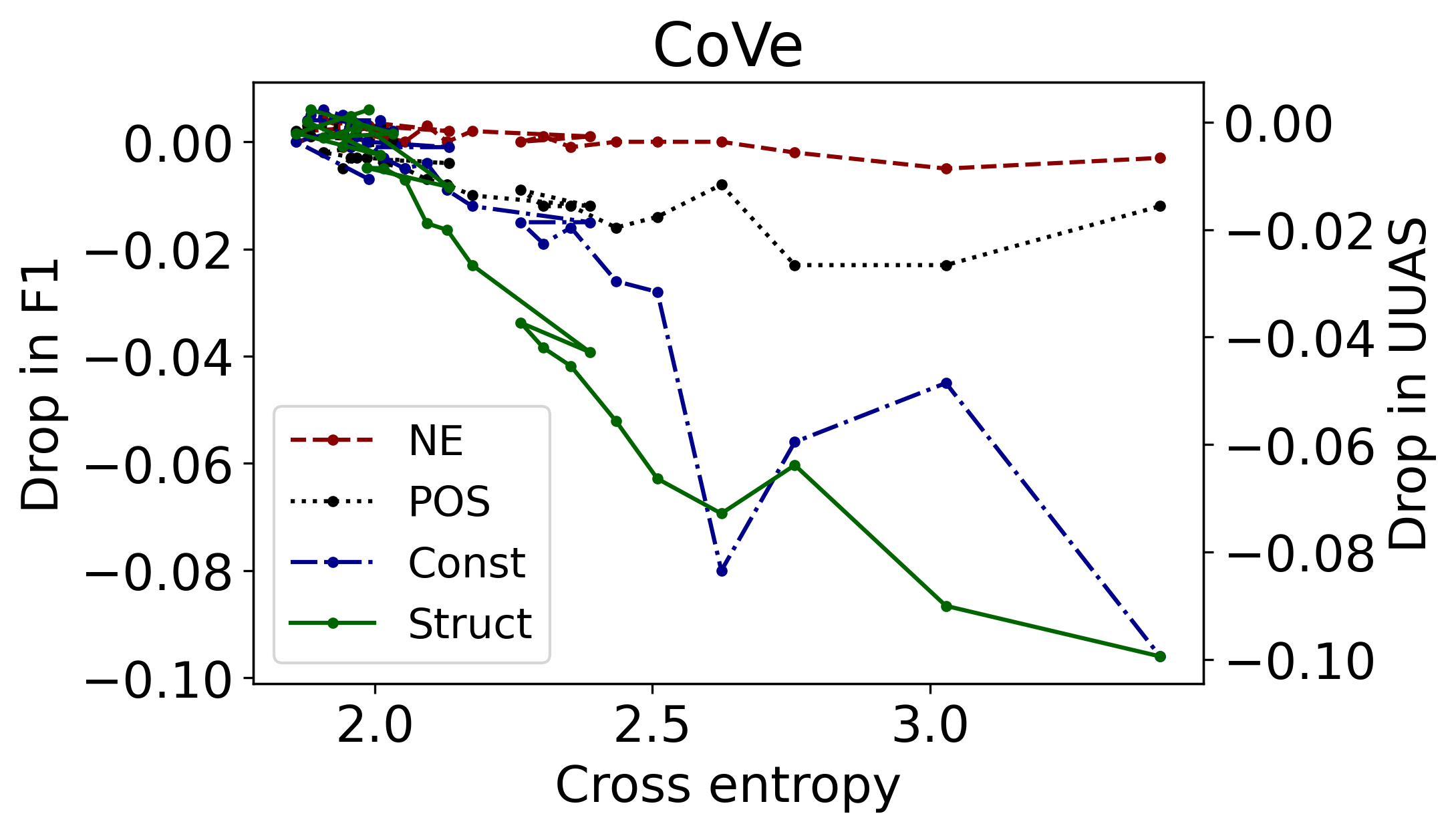}\hspace{10pt}\includegraphics[width=.45\textwidth]{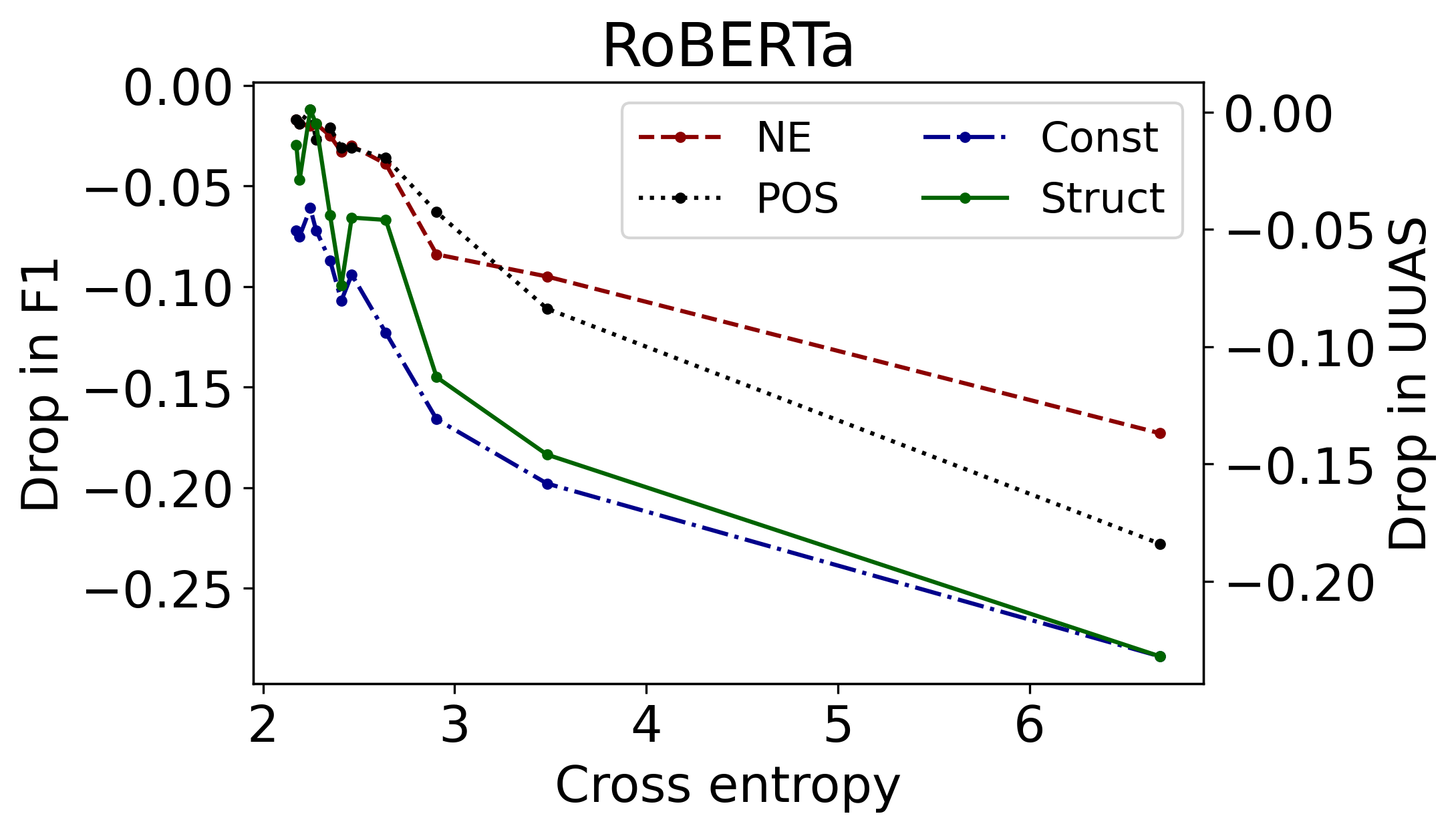}
\end{center}
\caption{Probing results for the {\sc CoVe} and {\sc RoBERTa} embeddings. Horizontal axes indicate validation loss values, which are cross-entropy values. Vertical axes indicate drops in probing performances. In case of edge probing, we use the NE, POS, and constituent labeling tasks from the suite of \shortciteA{tenney2018what} and report the drop in micro-averaged F1 score compared to the baseline (unpruned) model. In case of structural probing, we use the distance probe of \shortciteA{hewitt2019structural} and report the drop in undirected unlabeled attachment score (UUAS).}
\label{fig:edge_probing}
\end{minipage}

\begin{minipage}[t]{\textwidth}
\begin{center}
\includegraphics[width=.5\textwidth]{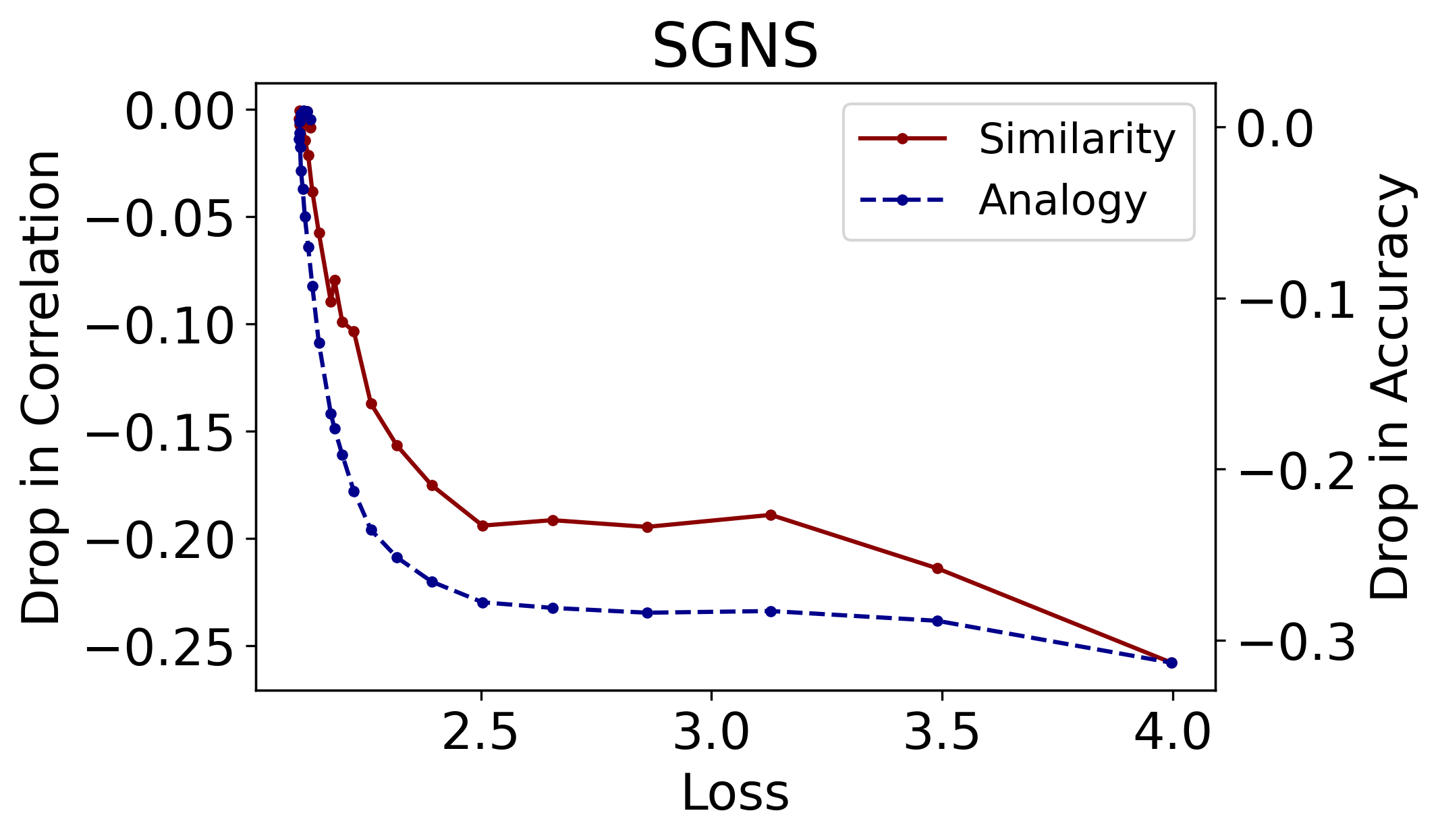}
\caption{Similarity and analogy results for the {\sc SGNS} embeddings. For similarities we report the drop in Spearman's correlation with the human ratings and for analogies in accuracy.}
\label{fig:nonparametric}
\end{center}
\end{minipage}
\end{figure}

Before proceeding to the results of probing of pruned models, we want to note that the probing scores of the baseline (unpruned) models obtained by us are close to the scores from the original papers \cite{tenney2018what,hewitt2019structural}, and are shown in Table~\ref{tab:baseline_scores}.
\begin{table}[htbp]
\begin{minipage}[t]{.55\textwidth}
\centering
\begin{tabular}{l c c c c c c}
\toprule
\multirow{2}{*}{Model} & \multicolumn{4}{c}{Task} \\
& NE & POS & Const. & Struct. \\
\midrule
{\sc CoVe} & .921 & .936 & .808 & .726 \\
{\sc RoBERTa} & .932 & .951 & .808 & .719 \\
\bottomrule
\end{tabular}
\end{minipage}%
\begin{minipage}[t]{.45\textwidth}
\centering
\begin{tabular}{l c c}
\toprule
\multirow{2}{*}{Model} & \multicolumn{2}{c}{Task} \\
& Similarity & Analogy \\
\midrule
{\sc SGNS} & .716 & .332 \\
\bottomrule
\end{tabular}
\end{minipage}
\caption{Probing scores for the baseline (unpruned) models. We report the micro-averaged F1 score for the POS, NE, and Constituents; undirected unlabeled attachment score (UUAS) for the structural probe; Spearman's correlation with the human ratings for the similarity task; and accuracy for the analogy task.}
\label{tab:baseline_scores}
\end{table}

Probing results are provided in Fig.~\ref{fig:edge_probing}~and~\ref{fig:nonparametric}, where we scatter-plot validation loss $\ell^{{\omega}}_i$ vs drop in probing performance $\Delta s^{{\omega},{T}}_i$ for each of the model-probe combinations.\footnote{Note, that cross-entropy values of {\sc CoVe} and {\sc RoBERTa} are not comparable as these are over different corpora, languages and vocabularies.} First, we note that, in most cases, the probing score correlates with the pretraining loss, which supports the rediscovery hypothesis. We note that the probing score decreases slower for some tasks (e.g.,\ POS tagging), but is much steeper for others (e.g.,\ constituents labeling). This is complementary to the findings of \shortciteA{zhang2020billions} who showed that the syntactic learning curve reaches plateau performance with less pretraining data while solving semantic tasks requires more training data. Our results suggest that similar behavior emerges with respect to the model size: simpler tasks (e.g.,\ POS tagging) can be solved with smaller models, while more complex linguistic tasks (e.g.,\ syntactic constituents or dependency parsing) require bigger model size. In addition, we note that in the case of {\sc CoVe}, and in contrast to {\sc RoBERTa}, the probing scores for more local (i.e.\ less context-dependent) tasks such as POS and NER hardly decrease with an increase in the pruning rate. 
We believe that this is because {\sc CoVe} representations by default contain unpruned static {\sc GloVe} embeddings, which by themselves already obtain good performance on the more local tasks.

\begin{figure}
\centering
\includegraphics[width=1\textwidth]{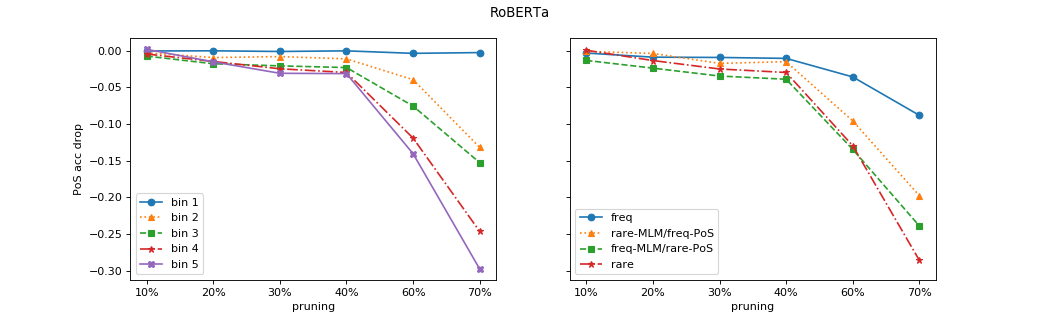}
\caption{Breakdown of POS tagging accuracy decrease by token frequency. We report the drop in accuracy compared to the baseline (unpruned) model. \textbf{Left}: all bins (1-5) have comparable cumulative counts in the pretraining data. The distribution of tokens over the bins: bin 1 --- 5 tokens, bin 2 --- 25 tokens, bin 3 --- 368 tokens, bin 4 --- 2813 tokens, bin 5 --- 600K tokens. \textbf{Right}: 4 bins of tokens which are (1) frequent both in pretraining and POS training data (153 tokens), (2) rare in pretraining but frequent in POS training data (66 tokens), (3) frequent in pretraining but rare in POS training data (235 tokens), and (4) rare everywhere (46K tokens). }
\label{fig:acc_per_tok_count}
\end{figure}

The reader might have noted that in the case of {\sc CoVe} in Fig.~\ref{fig:edge_probing} when the range of $\ell_i$'s is restricted to low values, there is a lack of correlation between $\ell_i$ and $\Delta s_i$. 
We discuss this further in Appendix~\ref{app:p_values} as we argue that this does not contradict our main finding.

\paragraph{Breakdown by token frequency.} 

The pervasiveness of Zipf's law in language has the consequence that many conclusions over aggregated scores can be attributed to effects on a small set of tokens.
Indeed, when binning tokens by their frequency in the pretraining data and analyzing the POS accuracy drop per bin it becomes obvious that the less frequent a token is the more pruning affects its POS-tag prediction.
This is shown in Fig.~\ref{fig:acc_per_tok_count} (left) where tokens are binned in 5 equally-sized groups.

A straightforward interpretation of such behavior would be: (1) pruning degrades the representations for rarer tokens in the first place. It is also possible, that (2) pruning degrades all token representations similarly; however, the probing model has the capacity to recover the POS performance for tokens that are frequent in POS tagging corpus but cannot do that for rare tokens. The first reason would support the claim that pruning removes the \textit{memorization}\footnote{To do well on the word predictions task the language model can either memorize patterns or learn certain ``language regularity'' (aka generalization). If we do an analogy with linguistic rules: there are cases when the linguistic rules ($\sim$ generalization) can be applied, and there are exceptions. The generalization has its limits, and at some point, the language model needs to memorize to boost the performance.} for rare tokens, while the second reason would mean that the pruned model redistributes its representation power across all the token groups to preserve good LM performance. The latter behavior would be more aligned with the \textit{rediscovery hypothesis}.


To better understand which of the two is a better explanation we need to distinguish between tokens which are rare in the pretraining data but are frequent in the downstream POS training data and vice-versa: the behavior on tokens which are rare in POS training data will be more informative of how pruning affects pretrained representations.

For this, we group the first three bins together in a \texttt{freq-MLM} group and the two last bins inside \texttt{rare-MLM}; and do a similar split based on the frequency of the tokens in the POS training corpus (\texttt{freq-POS} and \texttt{rare-POS}).

The impact of pruning on the four possible combinations of groups is shown in Fig.~\ref{fig:acc_per_tok_count} (right). The probing performance on the tokens which are frequent in POS tagging corpus (\texttt{freq} and \texttt{rare-MLM/freq-POS}) keep almost constant up to a pruning rate of 40\%. 
While tokens which are rare in POS tagging corpus (\texttt{freq-MLM/rare-POS} and \texttt{rare}) seem to suffer more with pruning rate growth. This lends support to the possibility that the amount of linguistics contained in the pretrained vectors decreases across all the token groups. 

Finally, we note that up to 60\% of pruning both \texttt{freq-MLM/rare-PoS} and \texttt{rare} groups behave similarly but at 70\% pruning rate the tokens that are rare everywhere (\texttt{rare}) suffer from a higher drop, compared to \texttt{freq-MLM/rare-PoS} meaning that probing model is not able to recover the correct PoS tags for rare tokens.
This suggests that \textit{memorization} gets removed from the pretrained model at a higher pruning rate, but it is exploited by the probing model at lower pruning rates.

\begin{figure}[t!]
\centering
\includegraphics[width=\textwidth]{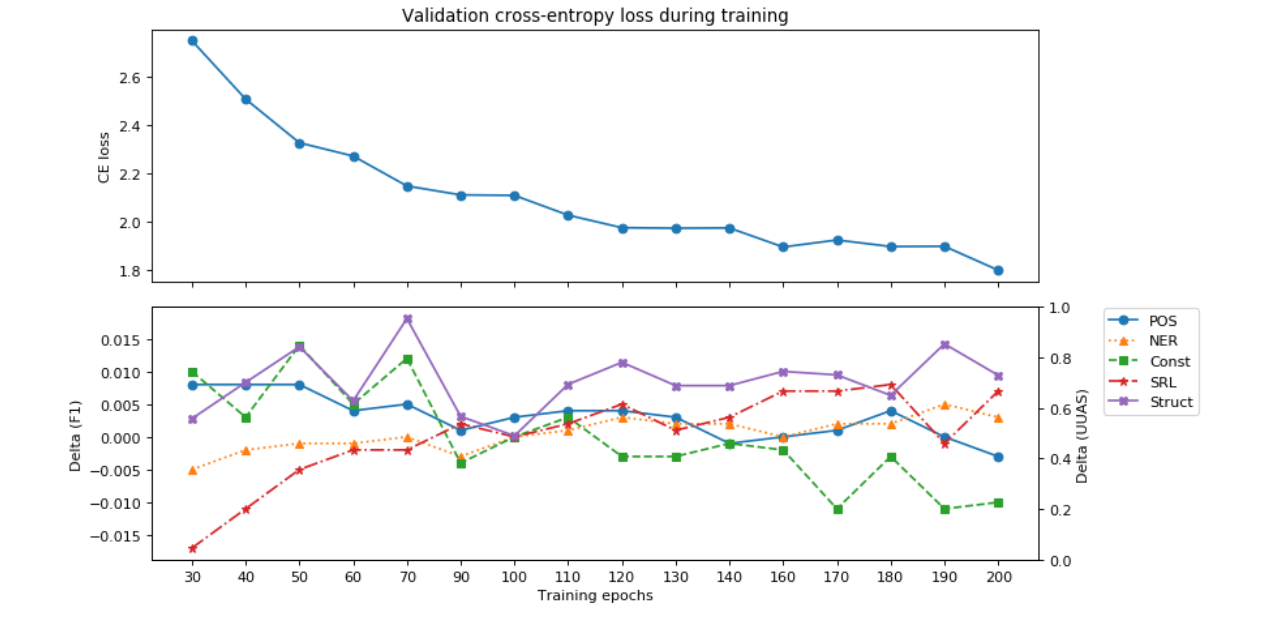}
\caption{Probing intermediate checkpoints of RoBERTa. Top: RoBERTa validation CE loss during training. Bottom: Difference between the baseline probing score\protect\footnotemark\ and probing score at each epoch (up to 200). In case of edge probing, we use the SRL, NE, POS, and constituent labeling tasks from the suite of \shortciteA{tenney2018what} and report the micro-averaged F1 score. In case of structural probing, we use the distance probe of \shortciteA{hewitt2019structural} and report the undirected unlabeled attachment score (UUAS).}
\label{fig:rob_cp_probing}
\end{figure}
\footnotetext{Baseline score corresponds to the probing with best checkpoint (around 450 epochs of training).} 

\subsection{Probing Intermediate Checkpoints} 
Until now we measured correlated probing accuracy on top of fully-trained models with the amount of pruning those models underwent.
Here we use a complementary lens, by analyzing how quickly a pretrained model discovers different types of linguistic knowledge? 
For this, we save intermediate checkpoints of RoBERTa models and probe those.
Similar to pruning, this has the advantage that all models share the same architecture.
The results of the probing accuracy over different epochs are provided in Fig.~\ref{fig:rob_cp_probing}.

Similar to \shortciteA{zhang2020billions}, we find that syntactic tasks (e.g.,\ POS tagging, dependency, and constituency parsing) seem to reach the top probing performance at the beginning of training (30-60 epochs), while semantic tasks (e.g.,\ NER, SRL) keep improving further. 
While \shortciteA{tenney2019bert} argued that ``\textit{BERT rediscovers NLP pipeline}'' by looking at intermediate layers of a fully pretrained model, the same seems to apply \textit{during} the training: linguistic knowledge to solve lexical tasks are learned first while more complex tasks are easier solved with representation obtained later in the learning.

\section*{Inter-Section Interlude}

We conclude that in general the rediscovery hypothesis seems valid: when 
language models reach good performance they indeed capture linguistic structures (at least in the case of the English language). 
So we obtained a negative result when we tried to reject the rediscovery hypothesis. Along with this, positive examples are abundant, including the extensive literature on probing (of which we give an overview in Section~\ref{sec:rel_work}).
This might indicate that the hypothesis is indeed true. The lack of correlation between probing scores and LM performance for one of the considered models only questions the probing methodology but does not reject the very fundamental connection between language modeling and learning linguistic structures. 
We address this in the next section, where we formalize the connection and show how it holds empirically.
The empirical experiments will consist in adversarially removing information that could serve probing accuracy while keeping good LM performance.\footnote{A more restricted scenario, where such adversarial training is performed on {\sc CoVe} is reported in Appendix~\ref{app:adversarial}.}

\section{An Information-Theoretic Framework}\label{sec:theory}
Recall that the rediscovery hypothesis asserts that neural language models, in the process of their pretraining, rediscover linguistic knowledge. We will prove this claim by contraposition, which states that without linguistic knowledge, the neural LMs cannot perform at their best. A recent paper of~\shortciteA{elazar2020bert} has already investigated how linearly removing\footnote{\textit{Removing linearly} means that a linear classifier cannot predict the required linguistic property with above majority class accuracy.} certain linguistic information from {\sc BERT}'s layers impacts its accuracy of predicting a masked token. They showed \textit{empirically} that dependency information, part-of-speech tags, and named entity labels are important for word prediction, while syntactic constituency boundaries (which mark the beginning and the end of a phrase) are not. One of the questions raised by the authors is how to \textit{quantify} the relative importance of different properties encoded in the representation for the word prediction task. The current section of our work attempts to answer this question---we provide a metric $\rho$ that is a reliable predictor of such importance. This metric occurs naturally when we take an information-theory lens and develop a theoretical framework that ties together linguistic properties, word representations, and language modeling performance.
We show that when a linguistic property is removed from word vectors, the decline in the quality of a language model depends on how strongly the removed property is interdependent with the underlying text, which is measured by $\rho$: \textit{a greater $\rho$ leads to a greater drop}. 

The proposed metric has an undeniable advantage: its calculation does not require word representations themselves or a pretrained language model. All that is needed is the text and its linguistic annotation. Thanks to our Theorem~\ref{prop:main}, we can express the influence of a linguistic property on the word prediction task in terms of the coefficient $\rho$.

\subsection{Notation}\label{sec:preliminaries}
We will use plain-faced lowercase letters ($x$) to denote scalars and plain-faced uppercase letters ($X$) for random variables. 
Bold-faced lowercase letters ($\mathbf{x}$) will denote vectors---both random and non-random---in the Euclidean space $\mathbb{R}^d$, while bold-faced uppercase letters ($\mathbf{X}$) will be used for matrices. 

Assuming there is a finite vocabulary $\mathcal{W}$, members of that vocabulary are called \textit{tokens}. A sentence $W_{1:n}$ is a sequence of tokens $W_i\in\mathcal{W}$, this is $W_{1:n}=[W_1,W_2,\ldots, W_n]$. A linguistic annotation $T$ of a sentence $W_{1:n}$ may take different forms. For example, it may be a sequence of per-token tags $T=[T_1,T_2,\ldots,T_n]$, or a parse-tree $T=(\mathcal{V},\mathcal{E})$ with vertices $\mathcal{V}=\{W_1,\ldots,W_n\}$ and edges $\mathcal{E}\subset \mathcal{V}\times \mathcal{V}$. We only require that $T$ is a \emph{deterministic} function of $W_{1:n}$.\footnote{Although in reality two people can give two different annotations of the same text due to inherent ambiguity of language or different linguistic theories, we will treat $T$ as the final---also called gold---annotation after disagreements are resolved between annotators and a common reference annotation is agreed upon.}

A (masked) language model is formulated as the probability distribution $q_{\boldsymbol{\theta}}(W_i\mid\boldsymbol\xi_i)\approx\Pr(W_i\mid C_i)$, where $C_i$ is the context of $W_i$ (see below for different types of context), and $\boldsymbol\xi_i$ is the vector representation of $C_i$. The cross-entropy loss of such a model is $\ell(W_i, \boldsymbol\xi_i):=\E_{(W_i,C_i)\sim \mathcal{D}}[-\log q_{\boldsymbol\theta}\left(W_i\mid\boldsymbol\xi_i\right)]$, where $\mathcal{D}$ is the true joint distribution of word-context pairs $(W, C)$.


For a random variable $X$, its entropy is denoted by $\h[X]$. For a pair of random variables $X$ and $Y$, their mutual information is denoted by $\I[X; Y]$. In Appendix~\ref{app:inf_theory} we provide the necessary background on information theory, and we refer the reader to refer to it if needed.

\paragraph{Discreteness of representations.} Depending on the LM, $C_i$ is usually either the left context $[W_1,\ldots,W_{i-1}]$ (for a causal LM) or a subsequence of the bidirectional context $[W_1,\ldots,W_{i-1},W_{i+1},\ldots,W_n]$ (for a masked LM). Although the possible set of all such contexts $\mathcal{C}$ is infinite, it is still \emph{countable}. Thus the set of all contextual representations $\{\boldsymbol\xi:\,\boldsymbol\xi\text{ is a vector representation of }C\mid C\in\mathcal{C}\}$ is also countable. Hence, we treat $\boldsymbol\xi$ as \emph{discrete} random vector.

\bigskip

\subsection{Main Result}
Our main result is the following

\begin{theorem}\label{prop:main}
Let 
\begin{enumerate}
\item $\mathbf{x}_i$ be a (contextualized) embedding of a token $W_i$ in a sentence $W_{1:n}$, and denote
\begin{equation}
\sigma_i:=\I[W_i;\mathbf{x}_i]/\h[W_{1:n}],\label{eq:sigma}
\end{equation}
\item $T$ be a linguistic annotation of $W_{1:n}$, and the dependence between $T$ and $W_{1:n}$ is measured by the coefficient
\begin{equation}
\rho:={\I[T;W_{1:n}]}/{\h[W_{1:n}]},\label{eq:rho}
\end{equation}
\item $\rho>1-\sigma_i$,
\item $\tilde{\mathbf{x}}_i$ be a (contextualized) embedding of $W_i$ that contains no information on $T$.
\end{enumerate}
Then the decline in the language modeling quality when using $\tilde{\mathbf{x}}_i$ instead of $\mathbf{x}_i$ is approximately supralinear in $\rho$:
\begin{equation}
\ell(W_i, \tilde{\mathbf{x}}_i)-\ell(W_i,\mathbf{x}_i)\gtrapprox \h[W_{1:n}]\cdot\rho+c\label{eq:delta_l}
\end{equation}
for $\rho>\rho_0$, with constants $\rho_0>0$, and $c$ depending on $\h[W_{1:n}]$ and $\I[W_i;\mathbf{x}_i]$.
\end{theorem}

\begin{proof}The proof is given in Appendix~\ref{app:proof}. Here we provide a less formal argument. Using visualization tricks as in \shortciteA{45511} we can illustrate the essence of the proof by Figure~\ref{fig:thm1}. 
\begin{figure}[htbp]
\centering
\includegraphics[width=.45\textwidth]{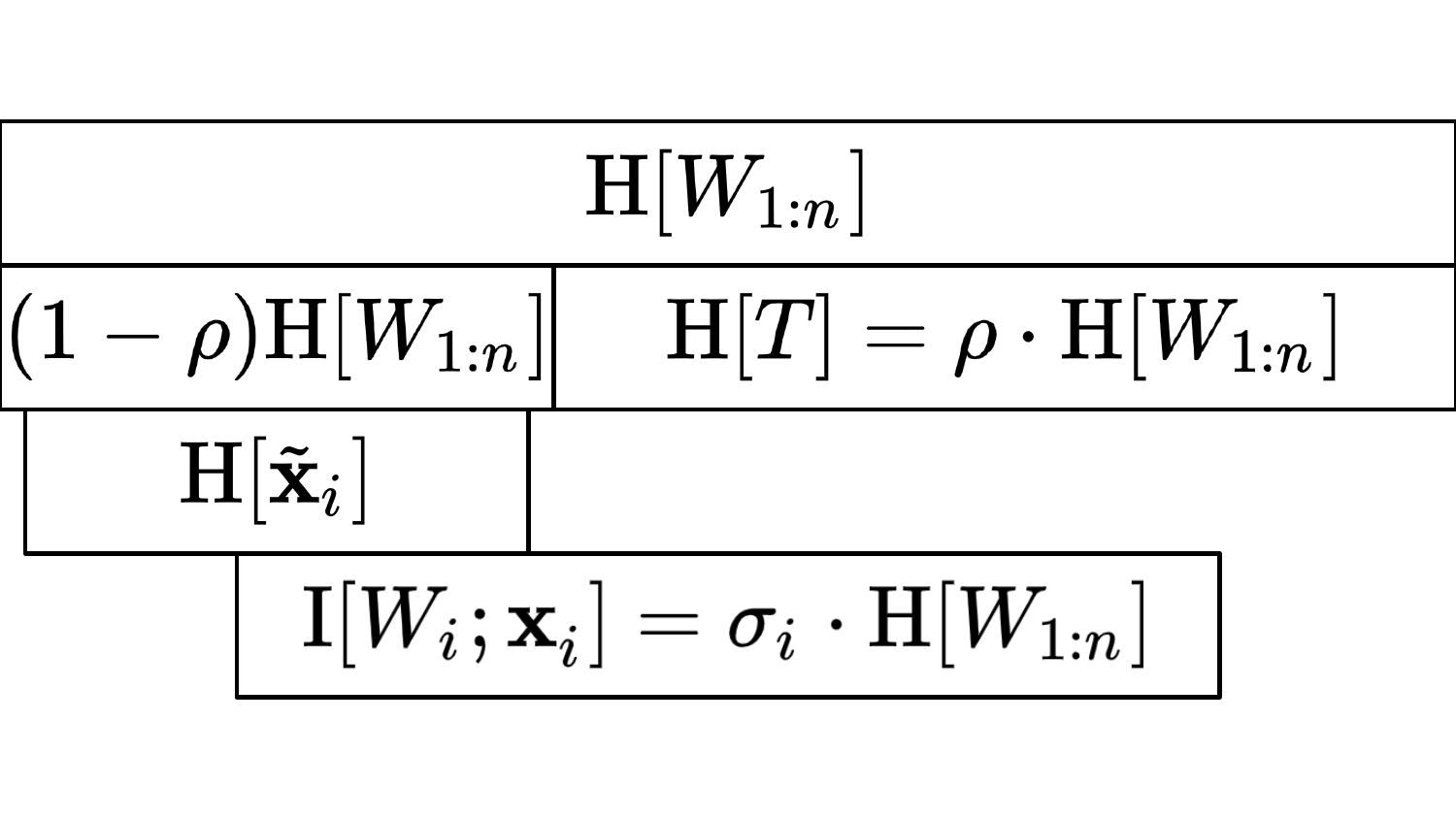}
\caption{Illustration of Theorem~\ref{prop:main}.}
\label{fig:thm1}
\end{figure}
First of all, look at the entropy $\h[X]$ as the amount of information in the variable $X$. Imagine amounts of information as bars. These bars overlap if there is shared information between the respective variables. 

The annotation $T$ and the embedding vector $\tilde{\mathbf{x}}_i$ are derived from the underlying text $W_{1:n}$, thus $W_{1:n}$ contains more information than $T$ or $\tilde{\mathbf{x}}_i$, hence $\h[W_{1:n}]$ fully covers $\h[T]$ and $\h[\tilde{\mathbf{x}}_i]$. Since the mutual information $\I[W_i;\tilde{\mathbf{x}}_i]$ cannot exceed the information in $\tilde{\mathbf{x}}_i$, we can write
\begin{equation}
\I[W_i;\tilde{\mathbf{x}}_i]\le\h[\tilde{\mathbf{x}}_i].\label{eq:x_tilde_text}
\end{equation}

Now recall the Eq.~\eqref{eq:rho}: it simply means that $\rho$ is the fraction of information that is left in $T$ after it was derived from $W_{1:n}$. This immediately implies that the information which is \emph{not} in $T$ (but was in $W_{1:n}$ initially) is equal to $(1-
\rho)\cdot\h[W_{1:n}]$.

Since the embedding $\tilde{\mathbf{x}}_i$ contains no information about the annotation $T$, $\h[\tilde{\mathbf{x}}_i]$ and $\h[T]$ do not overlap. But this means that 
\begin{equation}
\h[\tilde{\mathbf{x}}_i]\le(1-\rho)\cdot\h[W_{1:n}],\label{eq:first_ineq}
\end{equation}
because $\tilde{\mathbf{x}}_i$ is in the category of \textit{no information about} $T$. 

Combining \eqref{eq:sigma}, \eqref{eq:x_tilde_text}, \eqref{eq:first_ineq}, and the assumption $\rho>1-\sigma_i$, we have
\begin{align}
&\I[W_{i};\tilde{\mathbf{x}}_i]\le(1-\rho)\cdot\h[W_{1:n}]<\sigma_i\cdot\h[W_{1:n}]=\I[W_i;\mathbf{x}_i],\notag\\
&\Rightarrow\quad\I[W_i;\mathbf{x}_i]-\I[W_i;\tilde{\mathbf{x}}_i]>(\rho+\sigma_i-1)\cdot\h[W_{1:n}].\label{eq:almost}
\end{align}
The inequality \eqref{eq:almost} is almost the required \eqref{eq:delta_l}---it remains to show that the change in mutual information can be approximated by the change in LM loss; this is done in Lemma~\ref{lem:delta_i_l}.
\end{proof}

\paragraph{Role of $\rho$ and $\sigma_i$.} 
Equation~\eqref{eq:rho} quantifies the dependence between $T$ and $W_{1:n}$, and it simply means that the annotation $T$ carries $100\cdot\rho\%$ of information contained in the underlying text $W_{1:n}$ (in the information-theoretic sense). 
The quantity $\rho$ is well known as the \textit{entropy coefficient} in the information theory literature \shortcite{press2007conditional}. It can be thought of as an analog of the correlation coefficient for measuring not only linear or monotonic dependence between numerical variables but any kind of statistical dependence between any kind of variables (numerical and non-numerical). As we see from Equation~\eqref{eq:delta_l}, the coefficient $\rho$ plays a key role in predicting the LM degradation when the linguistic structure $T$ is removed from the embedding $\mathbf{x}_i$. In Section~\ref{sec:experiments} we give a practical way of its estimation for the case when $T$ is a per-token annotation of $W_{1:n}$.

Similarly, Equation~\eqref{eq:sigma} means that both $W_i$ and $\mathbf{x}_i$ carry at least $100\cdot\sigma_i\%$ of information contained in $W_{1:n}$. By Firth's distributional hypothesis~\cite{Firth57},\footnote{``you shall know a word by the company it keeps''} we assume that $\sigma_i$ significantly exceeds zero. 

\paragraph{Range of $\rho$.} 
In general, mutual information is non-negative. Mutual information of the annotation $T$ and the underlying text $W_{1:n}$ cannot exceed information contained in either of these variables, i.e.\ $0\le\I[T;W_{1:n}]\le\h[W_{1:n}]$, and therefore $\rho\in[0,1]$.

\paragraph{Absence of information.} When we write ``$\tilde{\mathbf{x}}_i$ contains no information on $T$'', this means that the mutual information between $\tilde{\mathbf{x}}_i$ and $T$ is zero: 
\begin{equation}
\I[T;\tilde{\mathbf{x}}_i]=0.\label{eq:mi_zero} 
\end{equation}
In the language of \shortciteA{DBLP:conf/emnlp/PimentelSWC20}, Equation~\eqref{eq:mi_zero} assumes that all probes---even the best---perform poorly in extracting $T$ from $\tilde{\mathbf{x}}_i$. This essentially means that the information on $T$ has been filtered out of $\tilde{\mathbf{x}}$. In practice, we will approximate this with the techniques of Gradient Reversal Layer \shortcite{ganin2015unsupervised} and Iterative Nullspace Projection \shortcite{ravfogel-etal-2020-null}.


\subsection{Experiments}\label{sec:experiments}

In this section, we empirically verify the prediction of our theory---the stronger the dependence of a linguistic property with the underlying text the greater the decline in the performance of a language model that does not have access to such property.

Here we will focus on only one contextualized embedding model, {\sc BERT} because it is the most mainstream model. Along with this, we will keep the {\sc SGNS} for consideration as a model of static embeddings.\footnote{Recall that {\sc SGNS} stands for skip-gram with negative sampling. {\sc Skip-gram} \shortcite{mikolov2013efficient} is a masked language model with all tokens but one in a sequence being masked. {\sc SGNS} approximates its cross-entropy loss by negative sampling procedure \shortcite{DBLP:conf/aistats/BengioS03}.}

\subsubsection{Removal Techniques}\label{sec:removals}

We consider two ways of removing linguistic information from the embeddings: Gradient Reversal Layer (GRL) and Iterative Nullspace Projection (INLP).

\paragraph{GRL \shortcite{ganin2015unsupervised}} is a method of adversarial training that, in our case, can be used to remove linguistic information $T$ from word embeddings $\boldsymbol\xi$. For this, the pretraining model is formulated as $\widehat{W}=\omega(\boldsymbol\xi)$, and an auxiliary model is $\widehat{T}=\tau(\boldsymbol\xi)$. The training procedure optimizes
$$
\min_{\omega, \tau, \boldsymbol\xi}\left[\ell(\omega(\boldsymbol\xi),W)+\ell(\tau(\gamma_\lambda(\boldsymbol\xi)),T)\right]\label{eq:train_main}
$$ 
where $\ell(\cdot,\cdot)$ is the loss function, and $\gamma_\lambda$ is a layer inserted between $\boldsymbol\xi$ and $\tau$ which acts as the identity during the forward pass, while it scales the gradients passed through it by $-\lambda$ during backpropagation. In theory, the resulting embeddings $\tilde{\mathbf{x}}$ are maximally informative for the pretraining task while at the same time minimally informative for the auxiliary task. However, in practice, GRL not always succeeds to fully remove the auxiliary information from the embeddings as was shown by \shortciteA{elazar2018adversarial} (and confirmed by our experiments in Appendix \ref{app:grl_finetune}).

\paragraph{INLP \shortcite{ravfogel-etal-2020-null}} is a method of post-hoc removal of some property $T$ from the pretrained embeddings ${\mathbf{x}}$. INLP neutralizes the ability to linearly
predict $T$ from $\mathbf{x}$ (here $T$ is a single tag, and $\mathbf{x}$ is a single vector). It does so by training a sequence of auxiliary models $\tau_1,\ldots,\tau_k$ that predict $T$ from $\mathbf{x}$, interpreting each one as conveying information on unique directions in the latent space that correspond to $T$, and iteratively removing each of these directions. In the $i^\text{th}$ iteration, $\tau_i$ is a \textit{linear} model\footnote{\shortciteA{ravfogel-etal-2020-null} use the Linear SVM \shortcite{cortes1995support}, and we follow their setup.} parameterized by a matrix $\mathbf{U}_i$ and trained to predict $T$ from $\mathbf{x}$. When the embeddings are projected onto $\mathrm{null}(\textbf{U}_{i})$ by a projection matrix $\mathbf{P}_{\mathrm{null}(\mathbf{U}_{i})}$, we have 
$$
\mathbf{U}_i\mathbf{P}_{\mathrm{null}(\mathbf{U}_{i})}\mathbf{x}=\mathbf{0},
$$
i.e.\ $\tau_i$ will be unable to predict $T$ from 
$\mathbf{P}_{\mathrm{null}(\mathbf{U}_{i})}\mathbf{x}$. Figure~\ref{fig:inlp} illustrates the method for the case when the property $T$ has only two types of tags and $\mathbf{x}\in\mathbb{R}^2$.
\begin{figure}
\centering
\includegraphics[width=.5\textwidth]{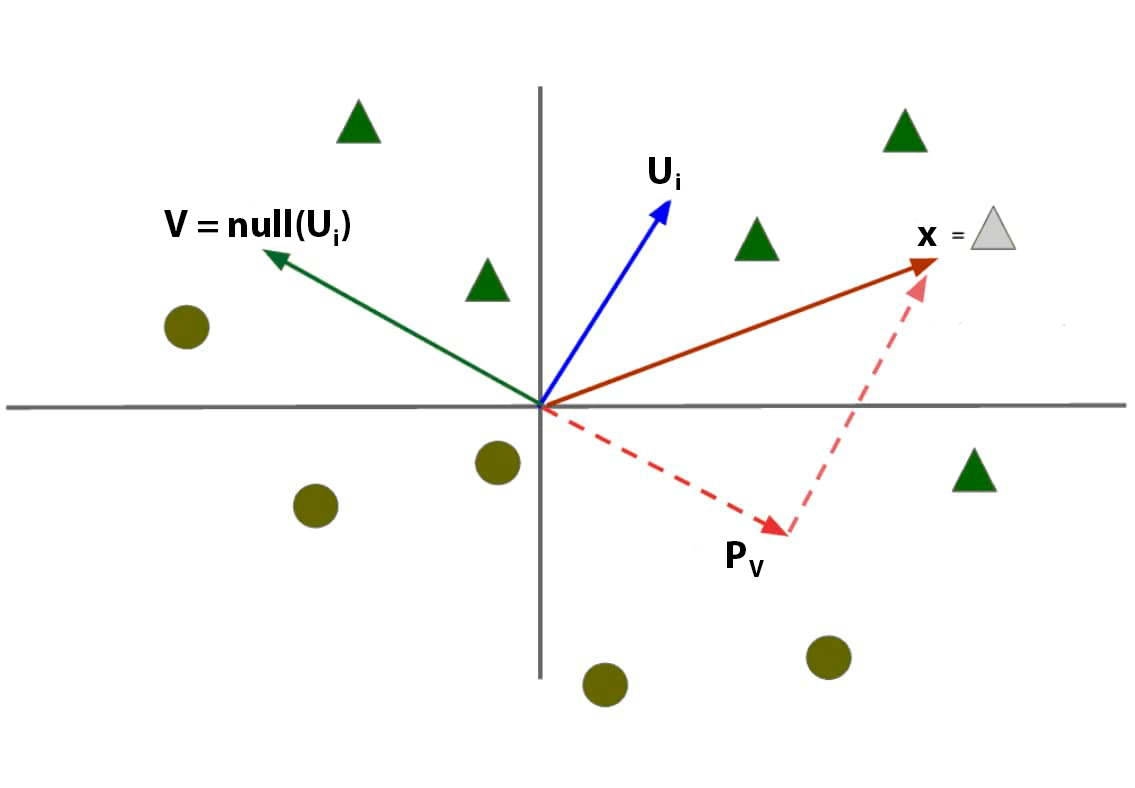}
\caption{Nullspace projection for a 2-dimensional binary classifier. The decision boundary of $\mathbf{U}_i$ is $\mathbf{U}_i$'s null-space. Source: \shortciteA{ravfogel-etal-2020-null}.}
\label{fig:inlp}
\end{figure}
The number of iterations $k$ is taken such that no linear classifier achieves above-majority accuracy when predicting $T$ from $\tilde{\mathbf{x}}=\mathbf{P}_{\mathrm{null}(\mathbf{U}_{k})}\mathbf{P}_{\mathrm{null}(\mathbf{U}_{k-1})}\ldots\mathbf{P}_{\mathrm{null}(\mathbf{U}_{1})}\mathbf{x}$. Before measuring the LM loss we allow the model to adapt to the modified embeddings $\tilde{\mathbf{x}}$ by fine-tuning its softmax layer that predicts $W$ from $\tilde{\mathbf{x}}$, while keeping the rest of the model (encoding layers) frozen. This fine-tuning does not bring back information on $T$ as the softmax layer linearly separates classes.

\subsubsection{Tasks}\label{sec:tasks}
We perform experiments on two real tasks and a series of synthetic tasks. In real tasks, we cannot arbitrarily change $\rho$---we can only measure it for some given annotations. Although we have several real annotations that give some variation in the values of $\rho$, we want more control over this metric. Therefore, we come up with synthetic annotations that allow us to smoothly change $\rho$ in a wider range and track the impact on the loss function of a language model.

To remove a property $T$ from embeddings, INLP requires a small annotated dataset, while GRL needs the \textit{whole} pretraining corpus to be annotated.\footnote{One may argue that GRL can be applied to a pretrained LM in a fine-tuning mode. Our preliminary experiments show that in this case, the linear probe's accuracy remains way above the majority class accuracy. See Appendix~\ref{app:grl_finetune} for details.} Therefore for INLP we use gold annotations, while for GRL we tag the pretraining data with the help of the {\sc Stanza} tagger \shortcite{qi2020stanza}.

\paragraph{Real tasks.} We consider two real tasks with per-token annotations: part-of-speech tagging (POS) and named entity labeling (NER). The choice of per-token annotations is guided by the suggested method of estimating $\rho$ (Section~\ref{sec:exp_setup}).

\textbf{POS} is the syntactic task of assigning tags such as noun, verb, adjective,
etc. to individual tokens. We consider POS tagged datasets that are annotated with different tagsets:
\begin{itemize}
\item Universal Dependencies English Web Treebank \shortcite<UD EWT,>{silveira14gold}, which uses two annotation schemas: UPOS corresponding to 17 universal tags across languages, and FPOS which encodes additional lexical and grammatical properties. 
\item English part of the OntoNotes corpus \shortcite{weischedel2013ontonotes} based on the Penn Treebank annotation schema \shortcite{DBLP:journals/coling/MarcusSM94}.
\end{itemize} 

\textbf{NER} is the task of predicting the category of an entity referred to by a given token,
e.g.\ does the entity refer to a person, a location, an organization, etc. This task is taken from the English part of the OntoNotes corpus.

Note that the {\sc Stanza} tagger, used to generate annotations for GRL training, relies on universal tagset from Universal Dependencies (UPOS) for POS tagging, and on OntoNotes NER annotations.
Table~\ref{tab:tagsets} reports some statistics on the size of different tagsets. 
\begin{table}[htbp]
\begin{center}
\begin{tabular}{c c c c}
\toprule
UD UPOS & UD FPOS & ON POS & ON NER \\
\midrule
17 & 50 & 48 & 66\\
\bottomrule
\end{tabular}
\caption{Size of the tagsets for Universal Dependencies and OntoNotes PoS tagging and NER datasets. }
\label{tab:tagsets} 
\end{center}
\end{table}

\paragraph{Synthetic tasks} are created as follows. Let $T^{(0)}$ be an annotation of a corpus $W_{1:n}$, which has $m$ unique tags, and let the corresponding $\rho^{(0)}:=\I[T^{(0)};W_{1:n}]/\h[W_{1:n}]$. We select the two least frequent tags from the tagset, and conflate them into one tag. This gives us an annotation $T^{(1)}$ which contains less information about $W_{1:n}$ than the annotation $T^{(0)}$, and thus has $\rho^{(1)}<\rho^{(0)}$. In Table~\ref{tab:example} we give an example of such conflation for a POS-tagged sentence from the OntoNotes corpus \shortcite{weischedel2013ontonotes}.

\begin{table}[htbp]
\centering
\begin{tabular}{l | c c c c c c c c c}
\toprule
$W_{1:9}$ & When & a & browser & starts & to & edge & near & to & consuming \\
$T^{(0)}$ & WRB & DT & NN & VBZ & TO & VB & RB & IN & VBG \\
$T^{(1)}$ & X & DT & NN & VBZ & X & VB & RB & IN & VBG\\
\midrule
$W_{10:18}$ & 500 & MB & of & RAM & on & a & regular & basis & , \\
$T^{(0)}$ & CD & NNS& IN & NN & IN& DT & JJ & NN & ,\\
$T^{(1)}$ & CD & NNS& IN & NN & IN& DT & JJ & NN & ,\\
\midrule
$W_{19:22}$ & something & is & wrong & .\\
$T^{(0)}$ & NN & VBZ & JJ & .\\
$T^{(1)}$ & NN & VBZ & JJ & .\\
\bottomrule
\end{tabular}
\caption{Example of conflating two least frequent tags (WRB and TO) into one tag (X).}
\label{tab:example}
\end{table}

Next, we select two least frequent tags from the annotation $T^{(1)}$ and conflate them. This will give an annotation $T^{(2)}$ with $\rho^{(2)}<\rho^{(1)}$. Iterating this process $m-1$ times we will end up with the annotation $T^{(m-1)}$ that tags all tokens with a single (most frequent) tag. In this last iteration, the annotation has no mutual information with $W_{1:n}$, i.e.\ $\rho^{(m-1)}=0$.



\subsubsection{Experimental Setup}\label{sec:exp_setup}

We remove (pseudo-)linguistic structures (Section~\ref{sec:tasks}) from {\sc BERT} and {\sc SGNS} embeddings using the methods from Section~\ref{sec:removals},\footnote{INLP is applied to the last layers of {\sc BERT} and {\sc SGNS}.} and measure the decline in the language modeling performance. Assuming that $\Delta\ell_{\omega,T,\mu}$ is the validation loss increase of the embedding model $\omega\in\{\textsc{BERT},\textsc{SGNS}\}$ when information on $T\in\{\text{Synthetic},\text{POS},\text{NER}\}$ is removed from $\omega$ using the removal method $\mu\in\{\text{GRL},\text{INLP}\}$, we compare $|\Delta\ell_{\omega,T,\mu}|$ against $\rho$ defined by \eqref{eq:rho} which is the strength of interdependence between the underlying text $W_{1:n}$ and its annotation $T$. By Theorem~\ref{prop:main}, $\Delta\ell_{\omega,T,\mu}$ is $\Omega(\rho)$\footnote{as a reminder, this means that there exist constants $c_0$, $c_1$, $\rho_0$ s.t. $\Delta\ell_{\omega,T,\mu}\ge c_1\cdot\rho+c_0$ for $\rho\ge\rho_0$.} for any combination of $\omega$, $T$, $\mu$.

\paragraph{Estimating $\rho$.} Recall that $\rho:=\I[T; W_{1:n}]/\h[W_{1:n}]$ (Eq.~\ref{eq:rho}) and that the annotation $T$ is a deterministic function of the underlying text $W_{1:n}$ (Sec.~\ref{sec:preliminaries}). 
In this case, we can write\footnote{In general, for two random variables $X$ and $Y$, $Y=f(X)$ if and only if $\h[Y\mid X]=0$.}
\begin{equation}
\rho=\frac{\h[T]-\overbrace{\h[T\mid W_{1:n}]}^0}{\h[W_{1:n}]}=\frac{\h[T]}{\h[W_{1:n}]}. \label{eq:rho_simple} 
\end{equation}
and when $T$ is a per-token annotation of $W_{1:n}$, i.e.\ $T=T_{1:n}$ (which is the case for the annotations that we consider), this becomes $\rho=\h[T_{1:n}]/\h[W_{1:n}]$. Thus to estimate $\rho$, we simply need to be able to estimate the latter two entropies. This can be done by training an autoregressive sequence model, such as LSTM, on $W_{1:n}$ and on $T_{1:n}$. The loss function of such a model---the cross-entropy loss---serves as an estimate of the required entropy. Notice, that we cannot use masked LMs for this estimation as they do not give a proper factorization of the probability $p(w_1,\ldots,w_n)$. 
Thus, we decided to choose the {\sc AWD-LSTM-MoS} model of \shortciteA{yang2017breaking} which is a compact and competitive LM\footnote{\url{https://paperswithcode.com/sota/language-modelling-on-wikitext-2}} that can be trained in a reasonable time and with moderate computing resources. In addition, we also estimated the entropies through a vanilla LSTM with tied input and output embeddings \shortcite{DBLP:conf/iclr/InanKS17}, and a 
Kneser-Ney 3-gram model\footnote{We used the SRILM toolkit \shortcite{DBLP:conf/interspeech/Stolcke02}. We used Witten-Bell discounting for annotations because modified Knesser-Ney discounting does not apply to them. This is a known issue with Knesser-Ney when the vocabulary size (number of unique tags) is small. See item C3 at \url{http://www.cs.cmu.edu/afs/cs/project/cmt-55/lti/Courses/731/homework/srilm/man/html/srilm-faq.7.html}.} \shortcite{DBLP:journals/csl/NeyEK94} to test how strongly our method depends on the underlying sequence model.

\paragraph{Limitations.} The suggested method of estimating $\rho$ through autoregressive sequence models is limited to per-token annotations only. However, according to formula~\eqref{eq:rho_simple}, to estimate $\rho$ for deeper annotations, it is sufficient to be able to estimate the entropy $\h[T]$ of such deeper linguistic structures $T$. For example, to estimate the entropy of a parse tree, one can use the cross-entropy produced by a probabilistic parser. The only limitation is the determinism of the annotation process.


\paragraph{Amount of information to remove.}

Each of the removal methods has a hyperparameter that controls how much of the linguistic information $T$ is removed from the word vectors $\mathbf{x}$---the number of iterations in INLP and $\lambda$ in GRL. Following \shortciteA{elazar2020bert} we keep iterating the INLP procedure or increasing $\lambda$ in GRL until the performance of a linear probe that predicts $T$ from the filtered embeddings $\tilde{\mathbf{x}}$ drops to the majority-class accuracy. When this happens we treat the resulting filtered embeddings $\tilde{\mathbf{x}}$ as containing no information on $T$. 

\paragraph{Optimization details.} For the INLP experiments we use pretrained {\sc BERT-Base} from HuggingFace \shortcite{DBLP:conf/emnlp/WolfDSCDMCRLFDS20}, and an {\sc SGNS} pretrained in-house \shortcite{sgns}. For the GRL experiments we pretrain LMs ourselves, details are in Appendix~\ref{app:optimization}.

\subsubsection{Results}\label{sec:results}

\paragraph{Real tasks.} Table~\ref{tab:pos_ner_results} reports the loss drop of pretrained LM when removing linguistic information (POS or NER) from the pretrained model. Plots that illustrate how losses and probing accuracies change with INLP iterations are provided in Appendix~\ref{app:inlp_dynamics}.

First, we compare UPOS tagsets versus FPOS tagset: intuitively FPOS should have a tighter link with underlying text, and therefore result in higher $\rho$ and as a consequence in a higher drop in loss after the removal of this information from words representations. This is confirmed by the numbers reported in Table~\ref{tab:pos_ner_results}. 

We also see a greater LM performance drop when the POS information is removed from the models compared to NE information removal. This is in line with Theorem~\ref{prop:main} as POS tags depend stronger on the underlying text than NE labels as measured by $\rho$. 

In the case of pretraining {\sc SGNS} with GRL, we could not reach the majority class accuracy.\footnote{This is because GRL does not always fully remove all the auxiliary task information from the main model \shortcite{elazar2018adversarial}.} Reported is the loss increase for the model that had the lowest accuracy 
($\text{Acc}=.466$ while the majority accuracy for this task is $.126$). If it were possible to train the SGNS model with GRL removing POS tags and reaching the majority class accuracy, its loss increase would be even higher than the one reported.

\begin{table}[htbp]
\centering
\begin{tabular}{l | c c c c | c c}
\toprule
Removal method $\mu$ & \multicolumn{4}{c |}{INLP} & \multicolumn{2}{c}{GRL}\\
Annotation $T$ & ON NER & UD UPOS & UD FPOS & ON POS & NER & UPOS \\
\midrule
$\rho$ (AWD-LSTM-MoS) & 0.18 & 0.32 & 0.36 & 0.42 & 0.18 & 0.43 \\
$\rho$ (LSTM) & 0.18 & 0.32 & 0.37 & 0.42 & 0.22 & 0.34 \\
$\rho$ (KN-3) & 0.18 & 0.36 & 0.41 & 0.42 & 0.17 & 0.37 \\
\midrule
$\Delta\ell_{\textsc{BERT},T,\mu}$ & 0.13 & 0.54 & 0.70 & 0.87 & 6.16 & 6.16 \\
$\Delta\ell_{\textsc{SGNS},T,\mu}$ & 1.33 & 1.62 & 1.79 & 2.04 & 0.32 & $>$1.1$^\dagger$\\ \bottomrule
\end{tabular}
\caption{Results on POS tagging and NER tasks. ON stands for OntoNotes, UD for Universal Dependencies. The UD EWT dataset has two types of POS annotation: coarse tags (UPOS) and fine-grained tags (FPOS). $^\dagger$GRL for SGNS on UPOS annotation did not reach the majority class accuracy (.126). KN-3 is a Kneser-Ney 3-gram model.}
\label{tab:pos_ner_results}
\end{table}


We note larger $\Delta\ell_{\text{\sc BERT},T,\text{GRL}}$ compared to $\Delta\ell_{\text{\sc BERT},T,\text{INLP}}$. We attribute it to the following: INLP works with already pretrained \textit{optimal} $\omega$ and we fine-tune its softmax layer for modified embeddings $\tilde{\mathbf{x}}$, while GRL with a large $\lambda$ (that is needed to bring the probing accuracy down to a majority class) moves $\omega$'s weights into a bad local minimum when pretraining from scratch. Moreover, the pretrained {\sc BERT} from HuggingFace and {\sc RoBERTa} pretrained by us are pretrained over different corpora and are using different vocabularies. Thus increases in the cross-entropies caused by INLP and by GRL are not directly comparable. 

We attempted to conduct experiments on GRL in such a way that their results were comparable to the results of the INLP---we tried to fine-tune a pretrained RoBERTa with GRL. However, our attempt was not successful as we did not manage to reach majority class accuracy (more details in Appendix \ref{app:grl_finetune}). This reinforces \citeauthor{elazar2018adversarial}'s \citeyear{elazar2018adversarial} argument that GRL has more difficulties removing fully the auxiliary information from the word representations. 

Finally, we see that although $\rho$ indeed depends on the underlying sequence model that is used to estimate the entropies $\h[T_{1:n}]$ and $\h[W_{1:n}]$, all models---AWD-LSTM-MoS, LSTM, and KN-3---preserve the relative order for the annotations that we consider. E.g.,\ all models indicate that OntoNotes NER annotation is the least interdependent with the underlying text, while OntoNotes POS annotation is the most interdependent one. In addition, it turns out that for a quick estimate of $\rho$, one can use the KN-3 model, which on a modern laptop calculates the entropy of texts of 100 million tokens in a few minutes, in contrast to the LSTM, which takes several hours on a modern GPU.

\bigskip

\paragraph{Synthetic tasks.} 
To obtain synthetic data, we apply the procedure described in Sect.~\ref{sec:tasks} to the OntoNotes POS annotation as it has the highest $\rho$ in Table~\ref{tab:pos_ner_results} and thus allows us to vary the metric in a wider range.

The results of evaluation on the synthetic tasks through the INLP\footnote{We did not perform GRL experiments on synthetic tasks as they are computationally expensive--- finding $\lambda$ for each annotation requires multiple pretraining runs.} are provided in Figure~\ref{fig:syn_results}. 
\begin{figure}
\centering
\includegraphics[width=.45\textwidth]{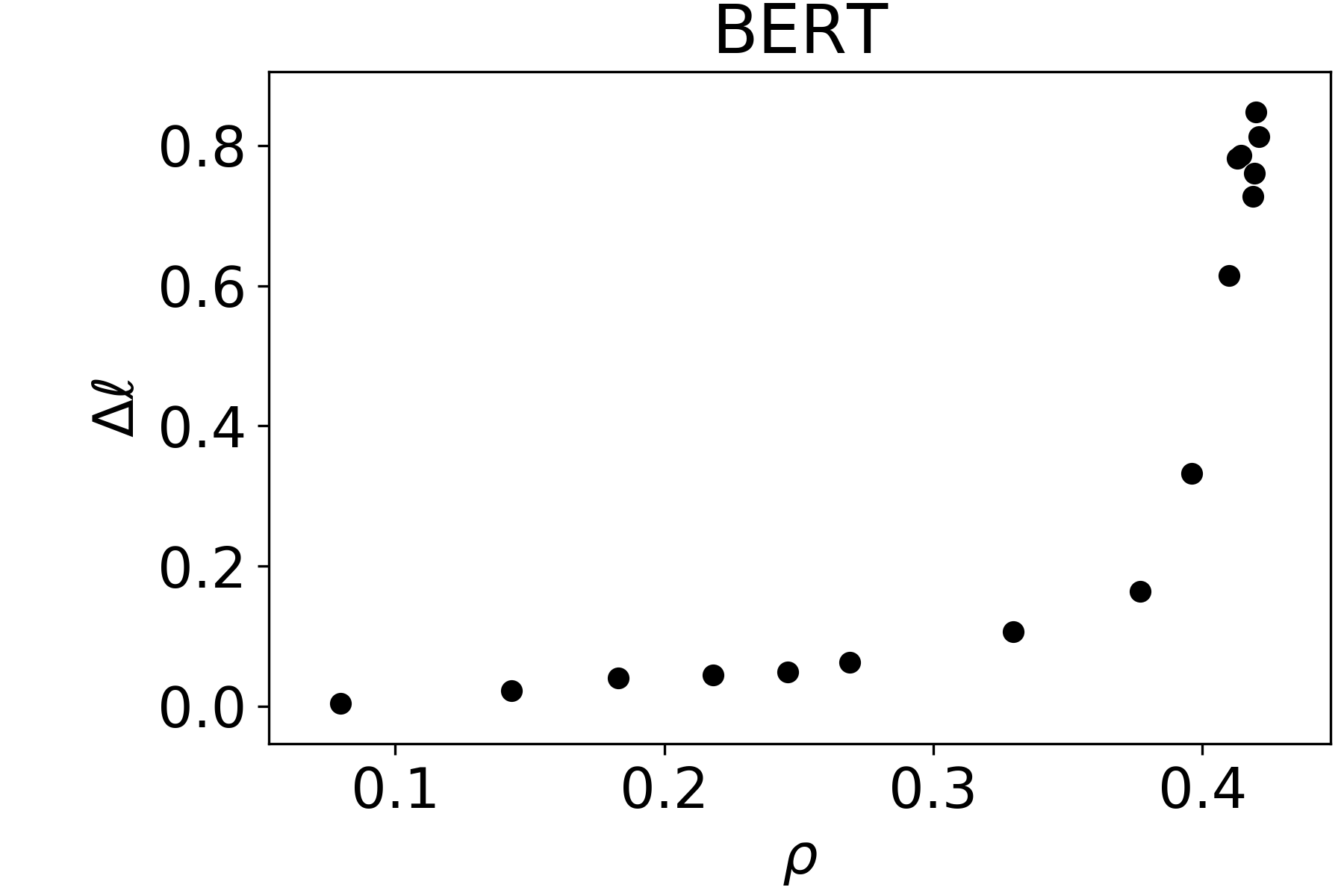}
\caption{Synthetic task results for INLP. $\Delta\ell$ is the increase in cross-entropy loss when pseudo-linguistic information is removed from the {\sc BERT}'s last layer with the INLP procedure. $\rho$ is estimated with the help of {AWD-LSTM-MoS} model \shortcite{yang2017breaking} as described in Section~\ref{sec:exp_setup}.}
\label{fig:syn_results}
\end{figure}
They validate the predictions of our theory---for the annotations with greater $\rho$ there is a bigger drop in the LM performance (i.e.\ increase in the LM loss) when the information on such annotations is removed from the embeddings. We notice that $\Delta\ell$ is piecewise-linear in $\rho$ with the slope changing at $\rho\approx0.4$. We attribute this change to the following: for $\rho<0.4$, the majority class (i.e.\ the most frequent tag) is the tag that encapsulates several conflated tags (see Subsection~\ref{sec:tasks} for details), while for $\rho>0.4$, the majority is NN tag. This switch causes a significant drop in the majority class accuracy which in turn causes a significant increase in the number of INLP iterations to reach that accuracy, and hence an increase in the amount of information being removed which implies greater degradation of the LM performance.

\section{Related Work}\label{sec:rel_work}

\paragraph{Theoretical analysis.} Since the success of early word embedding algorithms like {\sc SGNS} and {\sc GloVe}, there were attempts to analyze theoretically the connection between their pretraining objectives and performance on downstream tasks such as word similarity and word analogy tasks. An incomplete list of such attempts includes those of \shortciteA{DBLP:conf/nips/LevyG14,DBLP:journals/tacl/AroraLLMR16,DBLP:journals/tacl/HashimotoAJ16,DBLP:conf/acl/GittensAM17,DBLP:journals/ml/TianOI17,DBLP:conf/acl/EthayarajhDH19a,DBLP:conf/icml/AllenH19}. 
Most of these works represent pretraining as a low-rank approximation of some co-occurrence matrix---such as PMI---and then use an empirical fact that the set of columns (or rows) of such a matrix is already a good solution to the analogy and similarity tasks. Recently, we have seen a growing number of works devoted to the theoretical analysis of contextualized embeddings. \shortciteA{DBLP:conf/iclr/KongdYLDY20} showed that modern embedding models, as well as the old warrior {\sc SGNS}, maximize an objective function that is a lower bound on the mutual information between different parts of the text. \shortciteA{DBLP:journals/corr/abs-2008-01064} formalized how solving
certain pretraining tasks allows learning representations that provably decrease the sample complexity of downstream supervised tasks. 
Of particular interest is a recent paper by \shortciteA{saunshi2021a} that relates a pretraining performance of an autoregressive LM with a downstream performance for downstream tasks that \textit{can} be reformulated as next word prediction tasks. The authors showed that for such tasks, if the pretraining objective is $\epsilon$-optimal,\footnote{\shortciteA{saunshi2021a} say that the pre-training loss $\ell$ is \emph{$\epsilon$-optimal} if $\ell-\ell^\ast\le\epsilon$, where $\ell^\ast$ is the minimum achievable loss.} then the downstream objective of a linear classifier is $\mathcal{O}(\sqrt{\epsilon})$-optimal. In the second part of our work (Section~\ref{sec:theory}) we prove a similar statement, but the difference is that we study how the removal of linguistic information affects the pretraining objective and our approach is not limited to downstream tasks that can be reformulated as next word prediction.

\paragraph{Probing.} Early work on probing tried to analyze LSTM language models \shortcite{linzen2016assessing,shi-etal-2016-string,adi2017fine,conneau2018you,gulordava2018colorless,kuncoro2018lstms,tenney2018what}. Moreover, word similarity \shortcite{finkelstein2002placing} and word analogy \shortcite{mikolov2013linguistic} tasks can be regarded as non-parametric probes of static embeddings such as {\sc SGNS} \shortcite{mikolov2013distributed} and {\sc GloVe} \shortcite{pennington2014glove}. Recently the probing approach has been used mainly for the analysis of contextualized word embeddings. \shortciteA{hewitt2019structural} for example
showed that entire parse trees can be linearly extracted from {\sc ELMo}'s \shortcite{peters2018deep} and {\sc BERT}'s \shortcite{devlin2019bert} hidden layers. \shortciteA{tenney2018what} probed contextualized embeddings for various linguistic phenomena and showed that, in general, contextualized embeddings improve over their non-contextualized counterparts largely on syntactic tasks (e.g.,\ constituent labeling) in comparison to semantic tasks (e.g.,\ coreference). The probing methodology has also shown that {\sc BERT} learns some reflections of semantics \shortcite{reif2019visualizing} and factual knowledge \shortcite{petroni2019language} into the linguistic form which are useful in applications such as word sense disambiguation and question answering respectively. \shortciteA{zhang2020billions} analyzed how the quality of representations in a pretrained model evolves with the amount of pretraining data. They performed extensive probing experiments on various NLU tasks and found that pretraining with 10M sentences was already able to solve most of the syntactic tasks, while it required 1B training sentences to be able to solve tasks requiring semantic knowledge (such as Named Entity Labeling, Semantic Role Labeling, and some others as defined by \shortciteA{tenney2018what}). 

Closest to our second part (Section~\ref{sec:theory}), \shortciteA{elazar2020bert} propose to look at probing from a different angle, proposing \textit{amnesic probing} which is defined as the drop in performance of a pretrained LM after the relevant linguistic information is removed from one of its layers. The notion of amnesic probing fully relies on the assumption that the amount of the linguistic information contained in the pretrained vectors should correlate with the drop in LM performance after this information is removed. In this work (Section \ref{sec:theory}) we theoretically prove this assumption. While \shortciteA{elazar2020bert} measured LM performance as word prediction accuracy, we focus on the native LM cross-entropy loss. In addition, we answer one of the questions raised by the authors on how to measure the influence of different linguistic properties on the word prediction task---we provide an easy-to-estimate metric that does exactly this.

\paragraph{Criticism of probing.} The probing approach has been criticized from different angles. Our attempt to systematize this line of work is given in Figure~\ref{fig:probe_critique}. 
\begin{figure}[htbp]
\centering
\includegraphics[width=\textwidth]{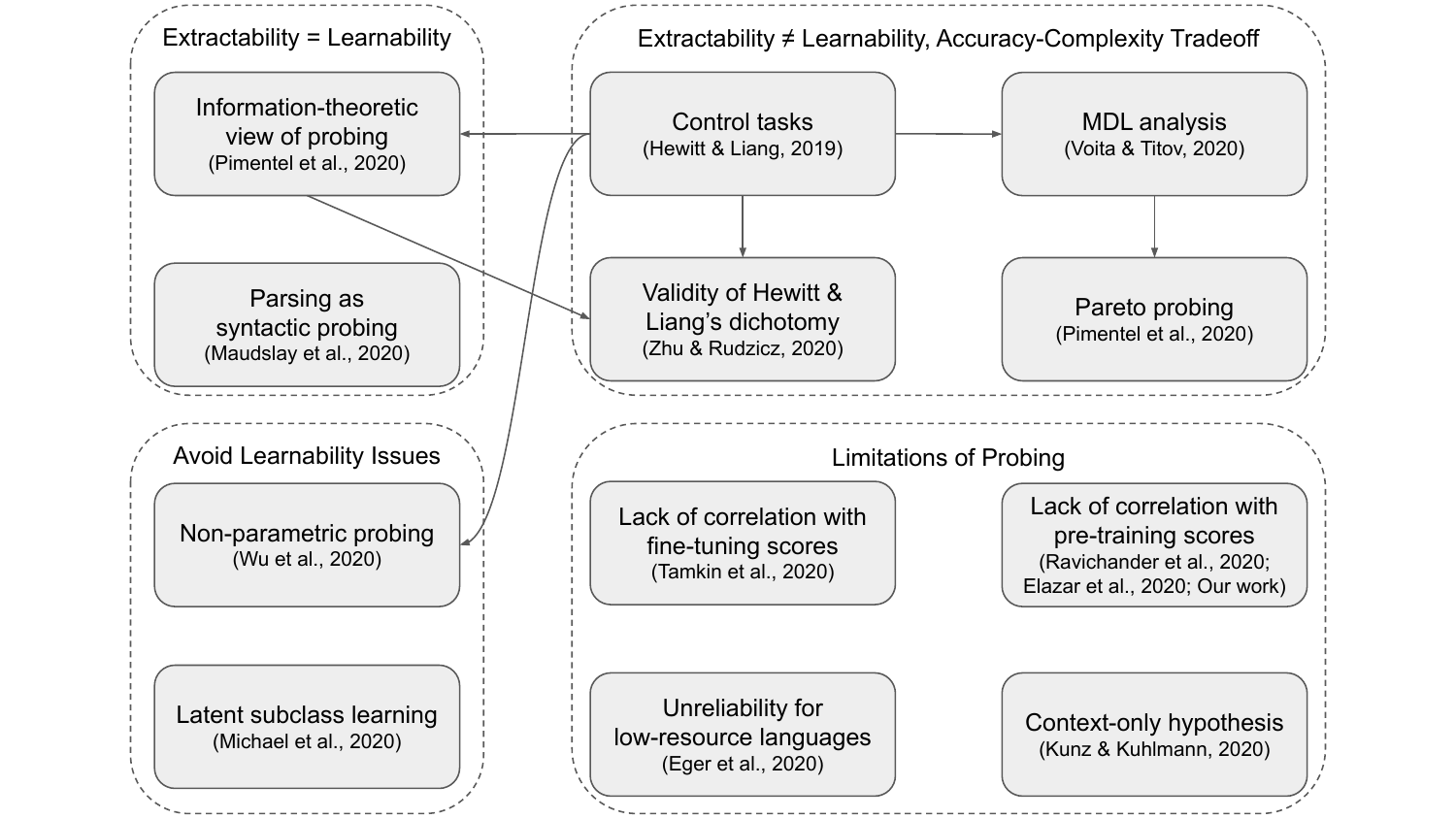}
\caption{Criticism and improvement of the probing methodology. An arrow $A\to B$ means that $B$ criticizes and/or improves $A$.}
\label{fig:probe_critique}
\end{figure}
The seminal paper of \shortciteA{hewitt2019designing} raises the issue of separation between \textit{extracting} linguistic structures from contextualized embeddings and \textit{learning} such structures by the probes themselves. This dichotomy was challenged by \shortciteA{pimentel2020information,maudslay2020}, but later validated by \shortciteA{DBLP:conf/emnlp/ZhuR20} using an information-theoretic view on probing. Meanwhile, methods were proposed that take into account not only the probing performance but also the ease of extracting linguistic information \shortcite{voita2020information} or the complexity of the probing model \shortcite{DBLP:conf/emnlp/PimentelSWC20}. At the same time, \shortciteA{DBLP:conf/acl/WuCKL20} and \shortciteA{DBLP:conf/emnlp/MichaelBT20} suggested avoiding learnability issues by non-parametric probing\footnote{Parametric probes transform embeddings to linguistic structures using \emph{parameterized} operations on vectors (such as feedforward layers). Non-parametric probes transform embeddings using \emph{non-parameterized} operations on vectors (such as vector addition/subtraction, inner product, Euclidean distance, etc.). The approach of \shortciteA{DBLP:conf/acl/WuCKL20} builds a so-called impact matrix and then feeds it into a graph-based algorithm to induce a dependency tree, all done without learning any parameters.} and weak supervision respectively. The remainder of the criticism is directed at the limitations of probing such as insufficient reliability for low-resourced languages \shortcite{eger-etal-2020-probe}, lack of evidence that probes indeed extract linguistic structures but do not learn from the linear context only \shortcite{DBLP:conf/coling/KunzK20}, lack of correlation with fine-tuning scores \shortcite{DBLP:conf/emnlp/TamkinSGG20} and with pretraining scores \shortcite{ravichander2020probing,elazar2020bert}. The first part of our work (Section~\ref{sec:lth_probing}) partly falls into this latter group, as we did not find any evidence for a correlation between probing scores and pretraining objectives for better performing {\sc CoVe} \shortcite{mccann2017learned}.

\paragraph{Pruning language models.} A recent work by \shortciteA{gordon-etal-2020-compressing} compressed {\sc BERT} using conventional pruning and showed a linear correlation between pretraining loss and downstream task accuracy. \shortciteA{DBLP:conf/nips/ChenFC0ZWC20} pruned pretrained {\sc BERT} with LTH and fine-tuned it to downstream tasks, while \shortciteA{DBLP:conf/emnlp/PrasannaRR20} pruned fine-tuned {\sc BERT} with LTH and then re-fine-tuned it. 
\shortciteA{sanh2020movement} showed that the weights needed for specific tasks are a small subset of the weights needed for masked language modeling, but they prune \textit{during fine-tuning} which is beyond the scope of our work.
\shortciteA{zhao-etal-2020-masking} propose to learn the masks of the pretrained LM as an alternative to finetuning on downstream tasks and shows that it is possible to find a subnetwork of large pretrained model which can reach the performance on the downstream tasks comparable to finetuning on this task. 
Generally speaking, the findings of the above-mentioned papers are aligned with our findings that the performance of pruned models on downstream tasks is correlated with the pretraining loss. 
The one difference from Section~\ref{sec:lth_probing} of our work is that most of the previous work looks at the performance of \textit{fine-tuned} pruned models. In our work, we \textit{probe} pruned models, i.e.\ the remaining weights of language models are \textit{not} adjusted to the downstream probing task. It is not obvious whether the conclusions from the former should carry over to the latter.

\section{Conclusion}
In this work, we tried to better understand the phenomenon of the \textit{rediscovery hypothesis} in pretrained language models. Our main contribution is two-fold: we demonstrate that the \textit{rediscovery hypothesis} \begin{itemize}
\item holds across various pretrained LMs (SGNS, CoVE, RoBERTa) even when a significant amount of weights get pruned (in English);
\item can be formally defined within an information-theoretic framework and proved (assuming that the linguistic annotation is a deterministic function of the underlying text and that the annotation is sufficiently interdependent with the text).
\end{itemize}

First (Section~\ref{sec:lth_probing}), we performed probing of different pruned\footnote{Pruning was done through iterative magnitude pruning (aka Lottery Ticket Hypothesis) or 1-shot magnitude pruning} instances of the original models. If models are overparametrized, then it could be that the pruned model only keeps the connections that are important for the pretraining task, but not for auxiliary tasks like probing. Our experiments show that there is a correlation between the pretrained model's cross-entropy loss and probing performance on various linguistic tasks. 
We believe that such correlation can be interpreted as strong evidence in favor of the \textit{rediscovery hypothesis}. 

While performing the above-mentioned analysis, several additional interesting observations have been made. 
\begin{itemize}
\item Break-down frequency analysis (Section \ref{sec:results1}) reveals that at moderate pruning rates the POS tag performance drop is quite similar across all the token frequency groups. It indicates that to preserve LM performance, the pruned models were forced to distribute their ``linguistic knowledge'' across all the token groups, and thus reinforces the \textit{rediscovery hypothesis}. 

We noted, however, that part of the probing performance could be due to exploitation of \textit{spurious memorization} in pretraining LM. A question on whether such \textit{spurious memorization} should be considered as \textit{linguistic knowledge} learnt by pretrained LM is an ongoing debate in the community \shortcite{Bender2020,sahlgren2021}. We believe that further investigation in this direction could bring more insights into the pretrained model's behavior. 
\item We have analyzed how the amount of \textit{linguistic knowledge} evolves during the pretraining process. Our analysis of intermediate pretraining checkpoints suggests that models follow \textit{NLP pipeline discovery} \textit{during} the pretraining process: it reaches good performance on syntactic tasks at the beginning of the training and keeps improving on semantic tasks further in training. 
\end{itemize}

Section \ref{sec:theory} proposes an information-theoretic formulation of the rediscovery hypothesis and a measure of the dependence between the text and linguistic annotations. 
In particular, Theorem~\ref{prop:main} states that neural language models\footnote{This is, language models which rely on intermediate representations of tokens.} \textit{do need to discover linguistic information} to perform at their best: when the language models are denied access to linguistic information, their performance is not optimal. 
The \textit{metric $\rho$}, which naturally appears in the development of our information-theoretic framework proposes a measure of the relationship between a given linguistic property and the underlying text. We show that it is also a reliable predictor of how suboptimal a language model will be if it is denied access to a given linguistic property. Our proposed methods of estimating $\rho$ can be used in further studies to \textit{quantify} the dependence between text and its per-token annotation. 

A number of previous studies \shortcite{tenney2018what,conneau2018you,hewitt2019structural,zhang2020billions} have shown some evidence in favour of \textit{rediscovery hypothesis}, while many others \shortcite{voita2020information,pimentel2020information,hewitt2019designing,elazar2020bert} questioned the methodology of those studies. However, none of the above-mentioned works questioned the rediscovery hypothesis \textit{per se} nor demonstrated any evidence for rejecting it. We believe that our work brings an additional understanding of this phenomenon. 
In addition, it is the first attempt to formalize the rediscovery hypothesis and provide a theoretical frame for it.

\section*{Acknowledgements}
The authors would like to thank the anonymous JAIR reviewers of this paper, the reviewers of previous versions, as well as Marc Dymetman, Arianna Bisazza, Germ\'an Kruszewski, Laurent Besacier, Shauli Ravfogel, and Yanai Elazar, for their detailed and constructive feedback. The work of Maxat Tezekbayev, Nuradil Kozhakhmet, Madina Babazhanova, and Zhenisbek Assylbekov is supported by the Nazarbayev University Collaborative Research Program 091019CRP2109.

\appendix


\section{Optimization}
\label{app:optimization}
\subsection*{CoVe} Two-layer BiLSTM encoder with size 600. For embedding matrix, we use {\sc GloVe} embeddings with 840B tokens and fix embedding matrix while training as was done in the original implementation \shortcite{mccann2017learned}. Thus we do not prune the embedding matrix in LTH experiments. For training, we used IWSLT 16 dataset. Scripts for preprocessing the data and training the model were adapted to a newer version of {\sc OpenNMT-py} \shortcite{klein2017opennmt}. We use the default hyperparameters of {\sc OpenNMT-py}: batch size 64, maximum training steps 100,000. We use the following early stopping criterion: we stop training the model if there is no improvement in 3 consecutive validation perplexity scores (validation is performed every 2,000 steps).

\subsection*{RoBERTa} Optimization is performed on the full Wikitext-103 data. We used the default {\tt roberta.base} tokenizer (distributed along with {\tt roberta.base} model) \shortcite{liu2019roberta}. The following training settings were used for training on 2 GPUs for 100 epochs:
\begin{itemize}
\item architecture: embedding size 512, 6 layers, 4 heads, hidden size 1024
\item batch size: 4096 tokens (with gradient accumulation for 4 steps)
\item dropout 0.3
\end{itemize}
Pruning rates used for LTH experiments are chosen as follows:
\begin{itemize}
\item step $0.025$ for pruning up to 20\% of model parameters
\item step $0.05$ between 20\% until 70\% of model parameters
\end{itemize}

\subsection*{SGNS} Embedding size 200, 15 epochs to train, 5 negative samples per training example, batch size 1024, window size 5. For training we use the \texttt{text8} dataset \shortcite{text8}, which is the first 100MB of the English Wikipedia dump on Mar. 3, 2006. For validation we use the next 10MB of the same dump. We were ignoring words that appeared less than 5 times in \texttt{text8}, and subsampling the rest as follows: each word $w$ in training set is discarded with probability $p=\sqrt{t/f(w)}+t/f(w)$, where $t=10^{-4}$, and $f(w)$ is the frequency of word $w$. We use {\sc Adam} optimizer with its default settings.

\section{Lack of Correlation for Better Performing Models}
\label{app:p_values}
In Section~\ref{sec:lth_probing} we noticed that when restricting the loss values of {\sc CoVe} to the better performing end (zoomed in regions in Figure~\ref{fig:cove_zoom}), drop in probing scores (compared to the scores of the baseline unpruned model) become indistinguishable. Here we treat this claim more rigorously. Assuming that $\ell^{\textsc{CoVe}}_{i}$ is the validation loss of \textsc{CoVe} at iteration $i$, and $\Delta s^{\textsc{CoVe},T}_{i}$ is the corresponding drop in the score on the probing task $T\in\{$NE, POS, Constituents, Structural$\}$, we obtain pairs $(\ell^{{\textsc{CoVe}}}_{i}, \Delta s^{{\textsc{CoVe}},{T}}_{i})$ and setup a simple linear regression
for the zoomed in regions in Figure~\ref{fig:cove_zoom}:
\begin{figure}
\centering
\includegraphics[width=.49\textwidth]{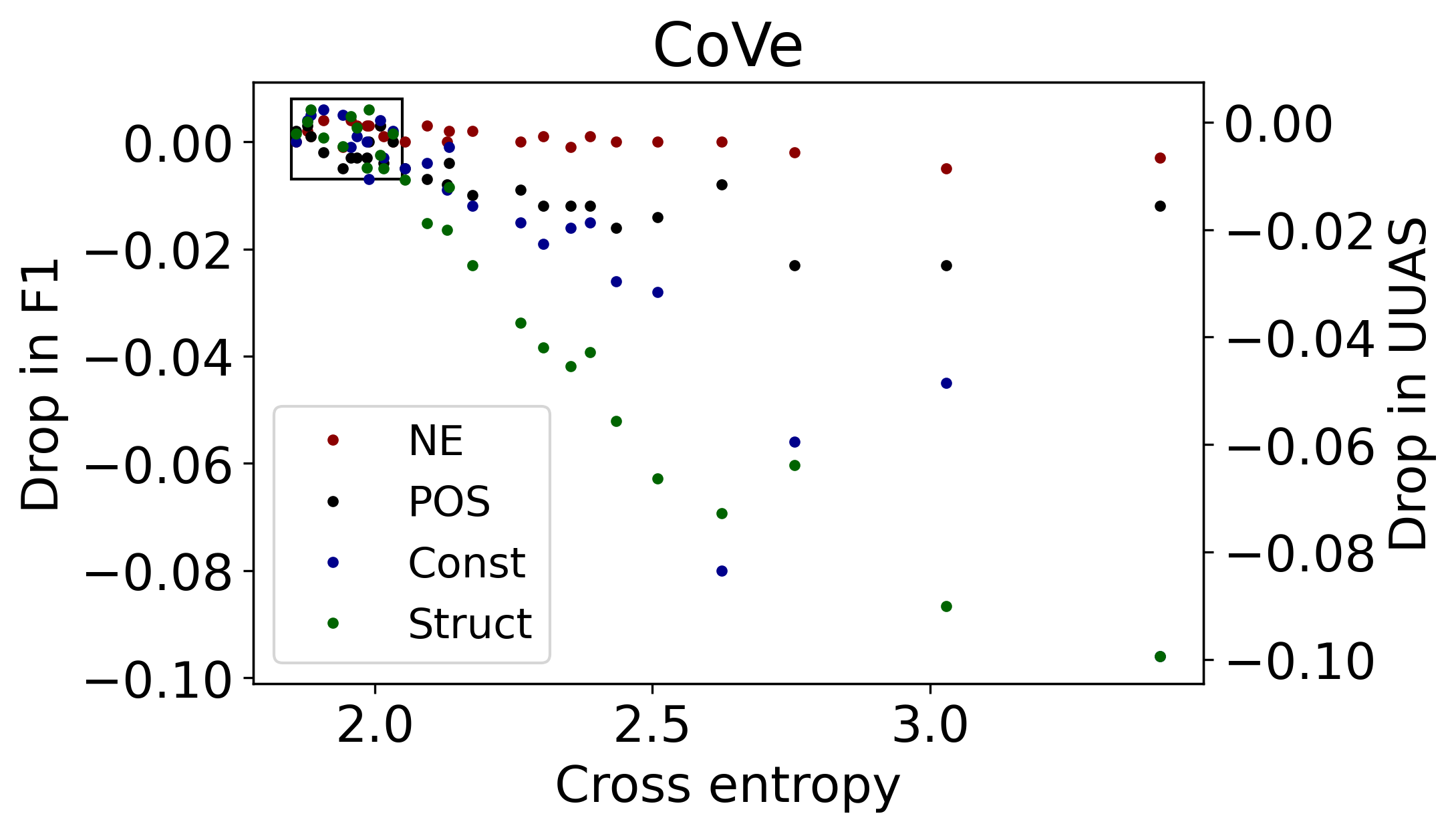}\hfill\includegraphics[width=.49\textwidth]{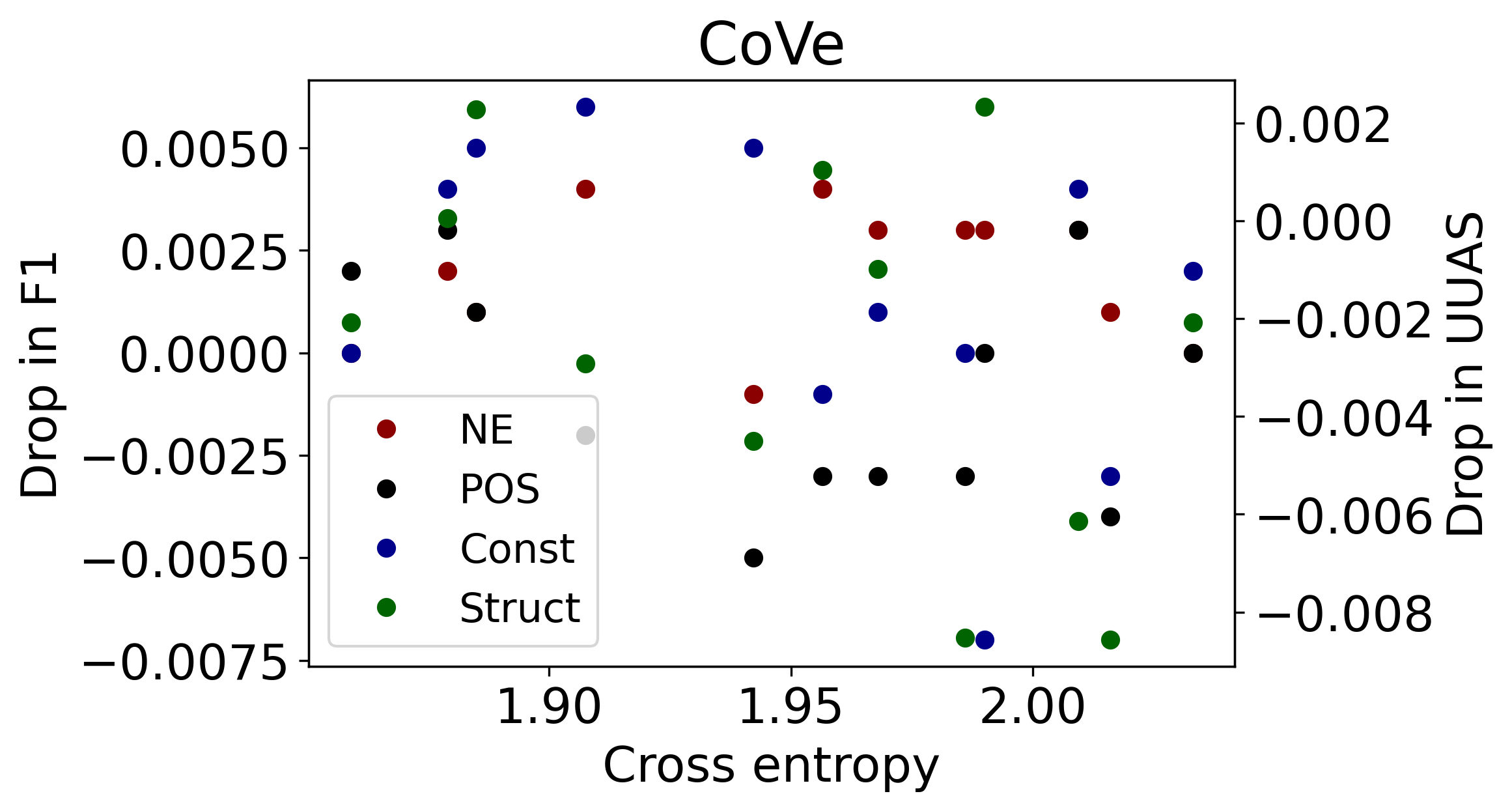}
\caption{Zooming in on better performing {\sc CoVe} models. Indication of the region to be zoomed in (left) and the zoomed in region (right).}
\label{fig:cove_zoom}
\end{figure}
$$
\Delta s^{\textsc{CoVe},T}_i=\alpha+\beta\cdot \ell^{\textsc{CoVe}}_i+\epsilon_i.
$$
Below are $p$-values when testing $H_0:\,\beta=0$ with a two-sided Student $t$-test.
\begin{table}[htbp]
\centering
\begin{tabular}{l | c c c c}
\toprule
Task & NE & POS & Constituents & Structural \\
\midrule
p-value & 0.725 & 0.306 & 0.184 & 0.160 \\
\bottomrule
\end{tabular}
\caption{Testing $\beta=0$, reported are $p$-values.}
\label{tab:my_label}
\end{table}
In all cases $\beta$ is \textit{not} significantly different from $0$. Hence, the correlation between model performance 
and its linguistic knowledge 
may become very weak for better solutions to the pretraining objectives. We argue that this phenomenon is related to the amount of information remaining in an LM after pruning, and how much of this information is enough for the probe. It turns out that in the best {\sc CoVe} (with the lowest pretraining loss) such information is redundant for the probe, and in the not-so-good (but still decent) {\sc CoVe} there is just enough of information for the probe to show its best performance.

\subsection*{MDL Analysis}\label{app:mdl}

Recently, \shortciteA{voita2020information} have addressed a related issue by proposing a minimum description length (MDL) analysis of probes, that takes into account the \textit{difficulty} of extracting the information.
Here, we apply their analysis to {\sc CoVe} (that is faster to experiment with than {\sc RoBERTA}), on top of the edge probe for constituents.\footnote{MDL requires the probe's response to be categorical and thus cannot be applied directly to the structural probe that predicts tree-distance between the vertices of a parse tree.} The resulting online codelengths vs cross-entropy values are provided in Fig.~\ref{fig:cove_mdl}. 
\begin{figure}[htbp]
\begin{minipage}[t]{.47\textwidth}
\includegraphics[width=\textwidth]{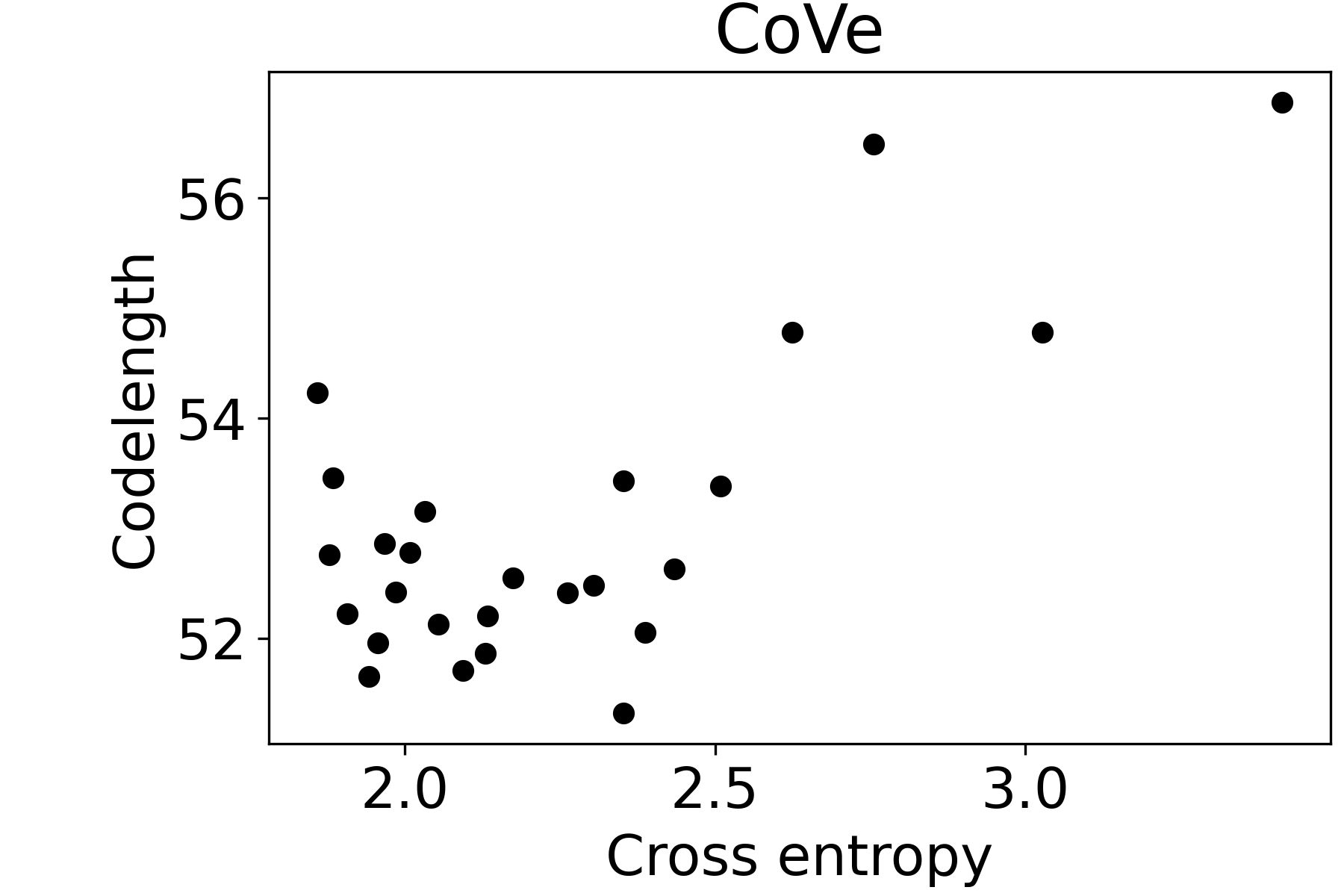}
\caption{MDL probing of {\sc CoVe} for constituents (online code). Unlike F1 score, the codelength takes into account not only the performance but also the amount of efforts to achieve it. 
}
\label{fig:cove_mdl} 
\end{minipage}\hfill
\begin{minipage}[t]{.47\textwidth}
\includegraphics[width=\textwidth]{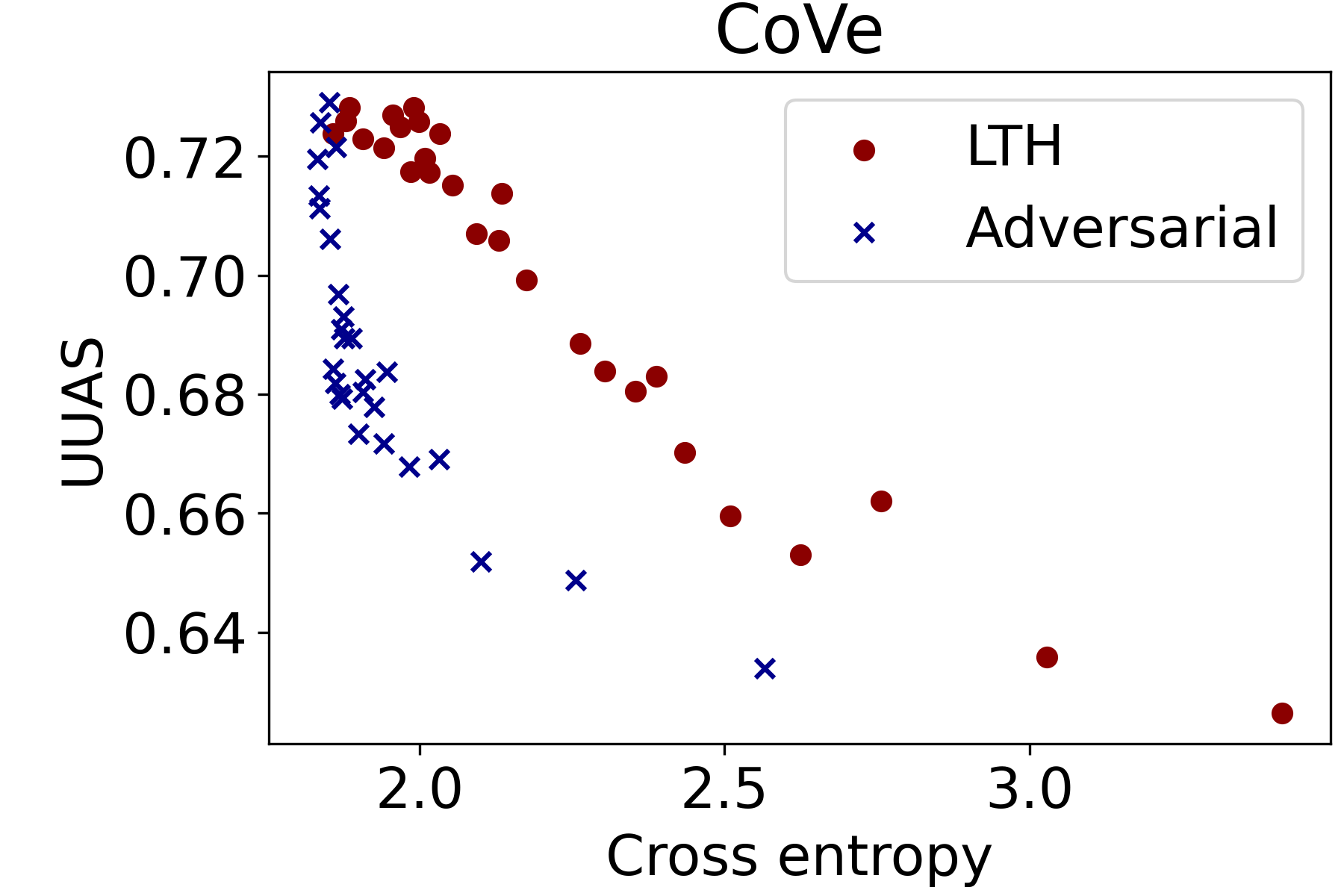}

\caption{Structural probing of adversarially trained {\sc CoVe}.}
\label{fig:cove_adv} 
\end{minipage}
\end{figure}
We observe an even \textit{wider} spectrum of models (with cross-entropies 1.9--2.5) that have indistinguishable codelengths. Nevertheless, the overall pattern does not contradict the rediscovery hypothesis---better models have better probing scores (lower codelengths).


\section{Adversarial Training} \label{app:adversarial}
In Sect.~\ref{sec:probes} we have analyzed linguistic knowledge discovered through standard language modeling pretraining. 
Here we propose a more radical approach, consisting in explicitly adding adversarial noise to the pretraining objective of {\sc CoVe} that enforces \textit{random} angles between contextualized representations of nearby words in the source sentence of the translation task.
Different to what is done in Sect.~\ref{sec:theory}, this is done on {\sc CoVE}.
For this, we exploit an empirical fact that the cosine similarity between word vectors tends to be a good predictor of their semantic similarity when trained non-adversarially.
We vary the amount of noise (through a hyperparameter), train the models, and probe them with the structural probe. 

\subsection{Training Details}\label{sec:cove_adv_train}
Let $F=[f_1,f_2,\ldots,f_n]$ be a sentence in the source language, and $E=[e_1,e_2,\ldots,e_m]$ be a corresponding sentence in the target language. Parallel corpus is the set of sentence pairs $\mathcal{D}=\{(F,E)\}$. Neural machine translation (NMT) model parameterized by $\boldsymbol\theta$ solves
\begin{equation}
\min_{\boldsymbol\theta}\frac{1}{|\mathcal{D}|}\sum_{(F,E)\in\mathcal{D}}\underbrace{-\log p_{\boldsymbol\theta}(E\mid F)}_{\ell^{\text{NMT}}(F,E)} 
\end{equation}
In case of {\sc CoVE}, the NMT model is an LSTM-based encoder-decoder architecture with attention \shortcite{bahdanau2014neural}, and let $\mathbf{h}_{1:n}=[\mathbf{h}_1,\mathbf{h}_2,\ldots,\mathbf{h}_n]$ be the top-level activations of the encoder that correspond to the sentence $F$.

Recall that for a pair of words $f_i$ and $f_j$, the cosine similarity between their vector representations $\mathbf{h}_i$ and $\mathbf{h}_j$ strongly correlates with the human similarity rating $r(f_i,f_j)$ \shortcite{mikolov2013distributed}. 
Strong correlation means that there is an approximately linear relationship between $r(f_i,f_j)$ and $\cos(\mathbf{h}_i, \mathbf{h}_j)$. 
In order to ``break'' this relationship we fit $\cos(\mathbf{h}_i,\mathbf{h}_j)$ to a \textit{random} $r\sim\mathcal{U}(0, 1)$ for nearby source words while training the NMT model (which {\sc CoVe} model is based on).
For a pair $(F,E)$ an individual objective becomes
$$
\ell(F,E)=\ell^{\text{NMT}}(F,E)+\alpha\cdot\ell^\text{Adv}(F),
$$
where 
$$
\ell^\text{Adv}(F)=\frac{1}{n}\sum_{i=1+k}^{n-k}\sum_{\substack{i-k\le j\le i+k\\j\ne0}}\mathbb{E}_{r_{ij}\sim\mathcal{U}(0,1)}[w\cdot\cos(\mathbf{h}_i,\mathbf{h}_j)+b-r_{ij}]^2,$$
$w$ and $b$ are regression coefficients, $k$ is the size of a sliding window, in our experiments we use $k=5$. Hyperparameter $\alpha$ regulates the impact of the adversarial loss.

\subsection{Results}

The probing results are provided in Fig.~\ref{fig:cove_adv}, where cross entropies are {\sc CoVe}'s validation loss values without the adversarial additive. 
Although the probing score decreases with an increase of cross-entropy, it does so significantly faster than in the case of LTH---we can have a near-to-best LM with a mediocre probing score. 
This would suggest that it is possible to find the state of the network that would encode less linguistic information given the same pretraining loss. This is however not the path naturally chosen by the pretraining models, therefore probably more difficult to reach.
Note, that several artifacts of {\sc CoVE} training could also explain these results. First, it can be because the pretrained embedding layer is not updated in these experiments, therefore keeping a good (static) representation of tokens. Moreover, recall that {\sc CoVE} is an encoder-decoder architecture, and reported cross-entropy loss is computed based on the decoder representations. The adversarial loss however has been defined on the encoder representations. Thus, even if the encoder representations got corrupted, the decoder still has a chance to recover and obtain reasonable cross-entropy loss. More extensive adversarial experiments might be necessary to be able to make any conclusion from this experiment. 

\section{Information Theory Background}\label{app:inf_theory}
For a random variable $X$ with the distribution function $p(x)$, its \textit{entropy} is defined as
$$
\h[X]:=\E_X[-\log p(X)].
$$
The quantity inside the expectation, $-\log p(x)$, is usually referred to as \textit{information content} or \textit{surprisal}, and can be interpreted as quantifying the level of ``surprise'' of a particular outcome $x$. As $p(x)\to0$, the surprise of observing $x$ approaches $+\infty$, and conversely as $p(x)\to1$, the surprise of observing $x$ approaches $0$. Then the entropy $\h[X]$ can be interpreted as the average level of ``information'' or ``surprise'' inherent in the $X$'s possible outcomes. We will use the following
\begin{property}\label{prop:ent}
If $Y=f(X)$, then
\begin{equation}
\h[Y]\le\h[X],
\end{equation}
i.e.\ the entropy of a variable can only decrease when the latter is passed through a (deterministic) function.
\end{property}

Now let $X$ and $Y$ be a pair of random variables with the joint distribution function $p(x,y)$. The \textit{conditional entropy} of $Y$ given $X$ is defined as
$$
\h[Y\mid X] := \E_{X,Y}[-\log p(Y\mid X)],
$$
and can be interepreted as the amount of information needed to describe the outcome of $Y$
given that the outcome of $X$ is known. The following property is important for us.
\begin{property}\label{prop:cond}
$\h[Y\mid X]=0\,\,\Leftrightarrow\,\,Y=f(X)$.
\end{property}

The \textit{mutual information} of $X$ and $Y$ is defined as
$$
\I[X;Y]:=\mathbb{E}_{X,Y}\left[\log\frac{p(X,Y)}{p(X)\cdot p(Y)}\right]
$$
and is a measure of mutual dependence between $X$ and $Y$. We will need the following properties of the mutual information.

\begin{property}\label{prop:mi}

\begin{enumerate}
\item \label{prop:diff}$\I[X;Y]=\h[X]-\h[X\mid Y]=\h[Y]-\h[Y\mid X]$.
\item \label{prop:indep}$\I[X;Y]=0\,\,\Leftrightarrow\,$ $X$ and $Y$ are independent.
\item \label{prop:dpi}For a function $f$, $\I[X;Y]\ge \I[X;f(Y)]$. 
\item \label{prop:joint} $\h[X,Y]=\h[X]+\h[Y]-\I[X;Y]$, 
\end{enumerate}
\end{property}
\noindent Property~\ref{prop:mi}.\ref{prop:dpi} is known as data processing inequality, and it means that post-processing cannot increase information. By \emph{post-processing} we mean a transformation $f(Y)$ of a random variable $Y$, independent of other random variables. In Property~\ref{prop:mi}.\ref{prop:joint}, $\h[X,Y]$ is the \textit{joint entropy} of $X$ and $Y$ which is defined as 
$$\h[X,Y]:=\E_{X,Y}[-\log p(X,Y)]$$ and is a measure of ``information'' associated with the outcomes of the tuple $(X,Y)$.

\section{Proof of Theorem~\ref{prop:main}}\label{app:proof}
The proof is split into two steps: first, in Lemma~\ref{prop:main_rigor} we show that the decrease in mutual information $\Delta\I:=\I[W_i;\mathbf{x}_i]-\I[W_i;\tilde{\mathbf{x}}_i]$ is $\Omega(\rho)$, and then in Lemma~\ref{lem:delta_i_l} we approximate $\Delta\I$ by the increase in cross-entropy loss $\ell(W_i,\tilde{\mathbf{x}}_i)-\ell(W_i,\mathbf{x}_i)$.

\begin{lemma} \label{prop:main_rigor}Let $T$ be an annotation of a sentence $W_{1:n}=[W_1,\ldots,W_i,\ldots,W_n]$ and let $\mathbf{x}_i$, $\tilde{\mathbf{x}}_i$ be (contextualized) word vectors corresponding to a token $W_i$ such that
\begin{equation}
\I[T;\tilde{\mathbf{x}}_i]=0.\label{eq:x_min} 
\end{equation}
Denote
\begin{align}
\rho&:=\I[T;W_{1:n}]/\h[W_{1:n}],\quad \rho\in[0,1],\label{eq:mi_p}\\
\sigma_i&:=\I[W_i;\mathbf{x}_i]/\h[W_{1:n}],\quad \sigma_i\in[0,1].\label{eq:x_max}
\end{align}
If $\rho>1-\sigma_i$, then
\begin{equation}
\I[W_i;\tilde{\mathbf{x}}_i]<\I[W_i;\mathbf{x}_i].\label{eq:main}
\end{equation}
and the difference $\Delta\I:=\I[W_i;\mathbf{x}_i]-\I[W_i;\tilde{\mathbf{x}}_i]$ is lowerbounded as 
\begin{equation}
\Delta\I\ge(\rho+\sigma_i-1)\cdot\h[W_{1:n}].\label{eq:delta_i}
\end{equation}
\end{lemma}
\begin{proof} From \eqref{eq:x_min} and by Property~\ref{prop:mi}.\ref{prop:joint}, we have
\begin{align}
\h[T,\tilde{\mathbf{x}}_i]&=\h[T]+\h[\tilde{\mathbf{x}}_i]-\underbrace{\I[T;\tilde{\mathbf{x}}_i]}_{0}\notag\\&=\h[T]+\h[\tilde{\mathbf{x}}_i].\label{eq:joint_sum}
\end{align}
On the other hand, both the annotation $T$ and the word vectors $\tilde{\mathbf{x}}_i$ are obtained from the underlying sentence $W_{1:n}$, and therefore $T=T(W_{1:n})$ and $\tilde{\mathbf{x}}_i=\tilde{\mathbf{x}}_i(W_{1:n})$, i.e.\ the tuple $(T,\tilde{\mathbf{x}}_i)$ is a function of $W_{1:n}$, and by Property~\ref{prop:ent},
\begin{equation}
\h[T,\tilde{\mathbf{x}}_i]\le\h[W_{1:n}].\label{eq:joint_ub}
\end{equation}
From \eqref{eq:joint_sum} and \eqref{eq:joint_ub}, we get
\begin{equation}
\h[T]+\h[\tilde{\mathbf{x}}_i]\le\h[W_{1:n}].\label{eq:sum_ub}
\end{equation}
Since $T=T(W_{1:n})$ and $\tilde{\mathbf{x}}_i=\tilde{\mathbf{x}}_i(W_{1:n})$, by Property~\ref{prop:cond} we have $\h[T\mid W_{1:n}]=0$ and $\h[\tilde{\mathbf{x}}_i\mid W_{1:n}]=0$, and therefore
$$
\I[T;W_{1:n}]=\h[T]-\h[T\mid W_{1:n}]=\h[T]
$$
Plugging this into \eqref{eq:sum_ub}, rearranging the terms, and taking into account \eqref{eq:mi_p}, we get
\begin{equation}
\h[\tilde{\mathbf{x}}_i]\le\h[W_{1:n}]-\underbrace{\I[T;W_{1:n}]}_{\rho\cdot\h[W_{1:n}]}=(1-\rho)\cdot\h[W_{1:n}].\label{eq:ub}
\end{equation}
Also, by Property~\ref{prop:mi}.\ref{prop:diff},
\begin{equation}
\I[W_i;\tilde{\mathbf{x}}_i]=\h[\tilde{\mathbf{x}}_i]-\underbrace{\h[\tilde{\mathbf{x}}_i\mid W_i]}_{\ge0}\le\h[\tilde{\mathbf{x}}_i]\label{eq:mi_h}
\end{equation}
From \eqref{eq:x_max}, \eqref{eq:ub}, \eqref{eq:mi_h}, and the assumption $\rho>1-\sigma_i$, we have
$$
\I[W_i;\tilde{\mathbf{x}}_i]\le\h[\tilde{\mathbf{x}}_i]\le(1-\rho)\cdot \h[W_{1:n}]<\sigma_i\cdot \h[W_{1:n}]=\I[W_i;\mathbf{x}_i],
$$
which implies \eqref{eq:main} and \eqref{eq:delta_i}.\end{proof}

\begin{remark}
Equation~\eqref{eq:mi_p} quantifies the dependence between $T$ and $W_{1:n}$, and it means that $T$ carries $100\cdot\rho\%$ of information contained in $W_{1:n}$. Similarly, Equation~\eqref{eq:x_max} means that both $W_i$ and $\mathbf{x}_i$ carry at least $100\cdot\sigma_i\%$ of information contained in $W_{1:n}$. By Firth's distributional hypothesis~\cite{Firth57},\footnote{``you shall know a word by the company it keeps''} $\sigma_i$ significantly exceeds zero. 
\end{remark}

\begin{remark}
The mutual information $\I[W_i;\boldsymbol\xi_i]$ can be interpreted as the performance of the best language model that tries to predict the token $W_{i}$ given its contextual representation $\boldsymbol\xi_i$.\footnote{Just as $\I[T_i;\boldsymbol\xi_{i}]$ is interpreted as the performance of the best probe that tries to predict the linguistic property $T_i$ given the embedding $\boldsymbol\xi_i$ \shortcite{pimentel2020information}} Therefore, inequalities \eqref{eq:main} and \eqref{eq:delta_i} mean that removing a linguistic structure from word vectors \textit{does} harm the performance of a language model based on such word vectors, and the more the linguistic structure is interdependent with the underlying text (that is being predicted by the language model) the bigger is the harm.
\end{remark}

We treat $\I[W;\boldsymbol\xi]$ as the performance of a language model. However, in practice, this performance is measured by the validation loss $\ell$ of such a model, which is usually the cross-entropy loss. Nevertheless, the change in mutual information \textit{can} be estimated by the change in the LM objective, as we show below.

\begin{lemma}\label{lem:delta_i_l}
Let $\mathbf{x}_i$ and $\tilde{\mathbf{x}}_i$ be (contextualized) embeddings of a token $W_i$ in a sentence $W_{1:n}=[W_1,\ldots,W_i,\ldots,W_n]$. Let $\ell(W_i,\boldsymbol\xi_i)$ be the cross-entropy loss of a neural (masked) language model $q_{\boldsymbol{\theta}}(w_i\mid\boldsymbol\xi_i)$, parameterized by $\boldsymbol\theta$, that provides distribution over the vocabulary $\mathcal{W}$ given vector representation $\boldsymbol\xi_i$ of the $W_i$'s context. Then
\begin{equation}
\I[W_i;\mathbf{x}_i]-\I[W_i;\tilde{\mathbf{x}}_i]\approx\ell(W_i,\tilde{\mathbf{x}}_i)-\ell(W_i, \mathbf{x}_i).\label{eq:delta_i_l}
\end{equation}
\end{lemma}
\begin{proof}
By Property~\ref{prop:mi}.\ref{prop:diff}, we have
\begin{align}
\Delta\I:&=\I[W_i;\mathbf{x}_i]-\I[W_i;\tilde{\mathbf{x}}_i]=\h[W_i]-\h[W_i\mid\mathbf{x}_i]-\h[W_i]+\h[W_i\mid\tilde{\mathbf{x}}_i]\notag\\
&=\h[W_i\mid\tilde{\mathbf{x}}_i]-\h[W_i\mid\mathbf{x}_i]\approx\h_{q}[W_i\mid\tilde{\mathbf{x}}_i]-\h_{q}[W_i\mid\mathbf{x}_i],\label{eq:delta_i_approx}
\end{align}
where $\h_q[W_i\mid\boldsymbol\xi_i]$---a cross entropy---is an estimate of $\h[W_i\mid\boldsymbol\xi_i]$ when the true distribution $p(w_i\mid\boldsymbol{\xi}_i)$ is replaced by a parametric language model $q_{\boldsymbol\theta}(w_i\mid\boldsymbol{\xi}_i)$, which is exactly the cross-entropy loss function of a LM $q$:
\begin{equation}
\h_q[W_i\mid\boldsymbol\xi_i]=\E_{(W_i,C_i)\sim \mathcal{D}}[-\log q_{\boldsymbol\theta}\left(W_i\mid\boldsymbol\xi_i\right)]=:\ell(W_i,\boldsymbol\xi_i).\label{eq:xent}
\end{equation}
From \eqref{eq:delta_i_approx} and \eqref{eq:xent}, we have \eqref{eq:delta_i_l}.
\end{proof}

\begin{remark}
When we approximate $\h[W_i\mid\mathbf{x}_i]\approx\mathrm{H}_q[W_i\mid\mathbf{x}_i]$, here the parameters $\boldsymbol{\theta}$ of $q=q_{\boldsymbol\theta}(w_i\mid\mathbf{x}_i)$ are only the LM head's parameters (e.g., the weights and biases of the output softmax layer in {\sc BERT}), and we assume that $\mathbf{x}_i$ can be \emph{any} representation (e.g., from the last encoding layer of BERT). We assume that $\boldsymbol{\theta}$ are chosen to minimize the KL-divergence between the true distribution $p(w_i\mid\mathbf{x}_i)$ and the model $q_{\boldsymbol{\theta}}(w_i\mid\mathbf{x}_i)$, thus making the approximation reasonable.

Similarly, when we approximate $\h[W_i\mid\tilde{\mathbf{x}}_i]\approx\h_q[W_i\mid\tilde{\mathbf{x}}_i]$ we again have the LM head $q_{\boldsymbol{\eta}}(w_i\mid\tilde{\mathbf{x}}_i)$, whose parameters $\boldsymbol\eta$ should be chosen to minimize the KL-divergence between true $p(w_i\mid\tilde{\mathbf{x}}_i)$ and the model $q_{\boldsymbol{\eta}}(w_i\mid\tilde{\mathbf{x}}_i)$, for the approximation to be reasonable. In the experimental part (Section~\ref{sec:experiments}), we do exactly this: after applying INLP (or GRL) to a representation, we fine-tune the softmax layer (but keep the encoding layers frozen) so that it can adapt to the modified representation.
\end{remark}

\section{Using GRL in Finetuning Mode}\label{app:grl_finetune}
In order to make GRL experiments more comparable to INLP we pretrain {\sc BERT} (as {\sc RoBERTa}) on WikiText-103 and finetune it with GRL on the OntoNotes POS annotation The results for different values of the GRL's $\lambda$ are given in Table~\ref{tab:grl_finetune_results}.
\begin{table}[htbp]
\centering
\begin{tabular}{c c c | c c c}
\toprule
$\lambda$ & $\ell_\text{\sc BERT}$ & Accuracy & $\lambda$ & $\ell_\text{\sc BERT}$ & Accuracy \\
\midrule
0 & 2.062 & 0.942 & 10 & 5.808 & 0.880 \\
0.1 & 2.070 & 0.920 & 50 & 6.362 & 0.844 \\
1 & 2.647 & 0.917 & 100 & 6.473 & 0.828 \\
5 & 3.585 & 0.913 & 1000 & 6.267 & 0.865 \\
\bottomrule
\end{tabular}
\caption{Finetuning {\sc RoBERTa} with GRL on OntoNotes POS annotation.}
\label{tab:grl_finetune_results}
\end{table}
The majority class accuracy on this dataset is 0.126, and we see that no matter how large the $\lambda$ we take, the GRL still cannot lower the accuracy to this level. This confirms observations of \shortciteA{elazar2018adversarial} that GRL may not always remove all the information from a model it is applied to.

\section{INLP Dynamics}\label{app:inlp_dynamics}
\begin{figure}[h]
\centering
\includegraphics[width=.4\textwidth]{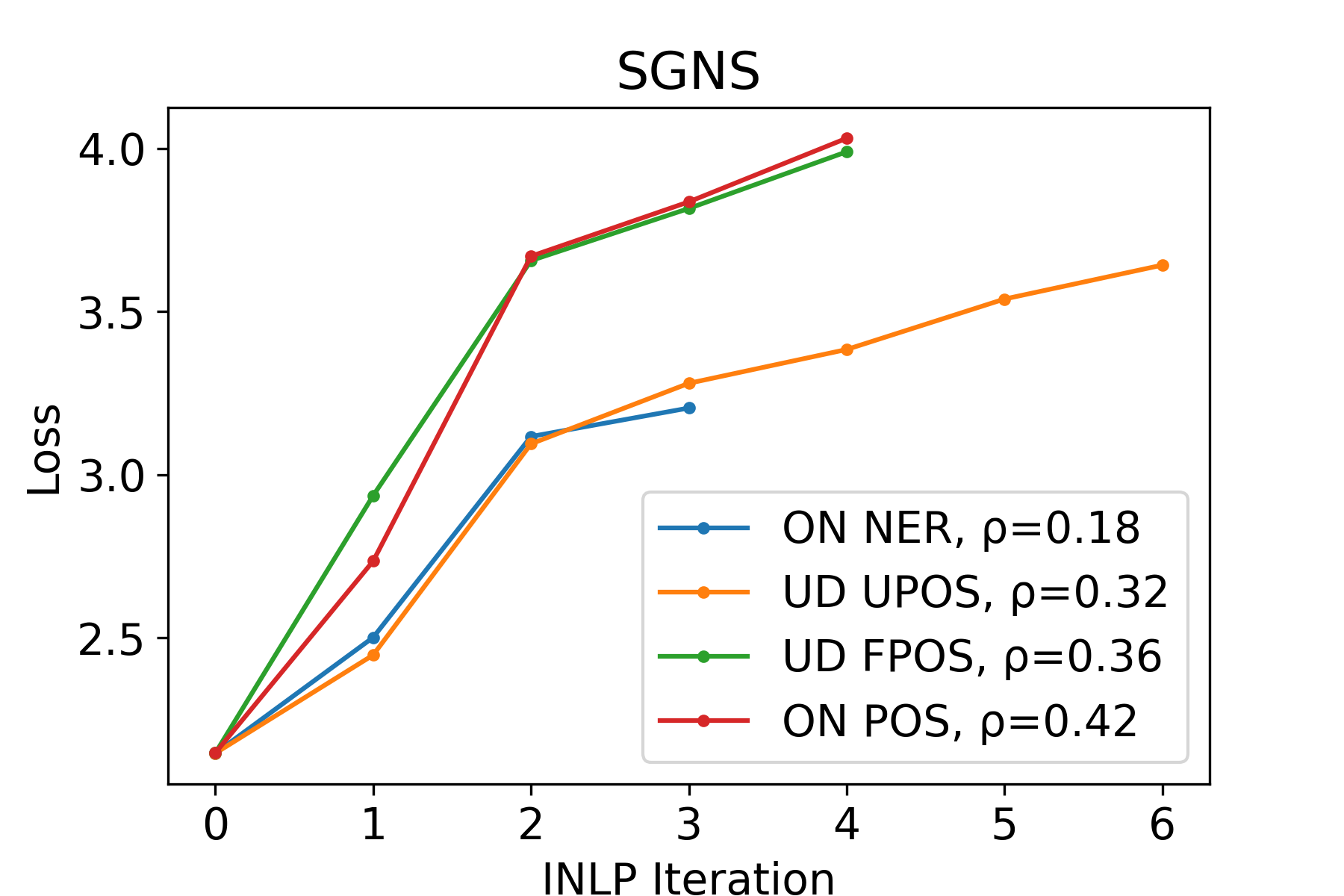}\hspace{15pt}\includegraphics[width=.4\textwidth]{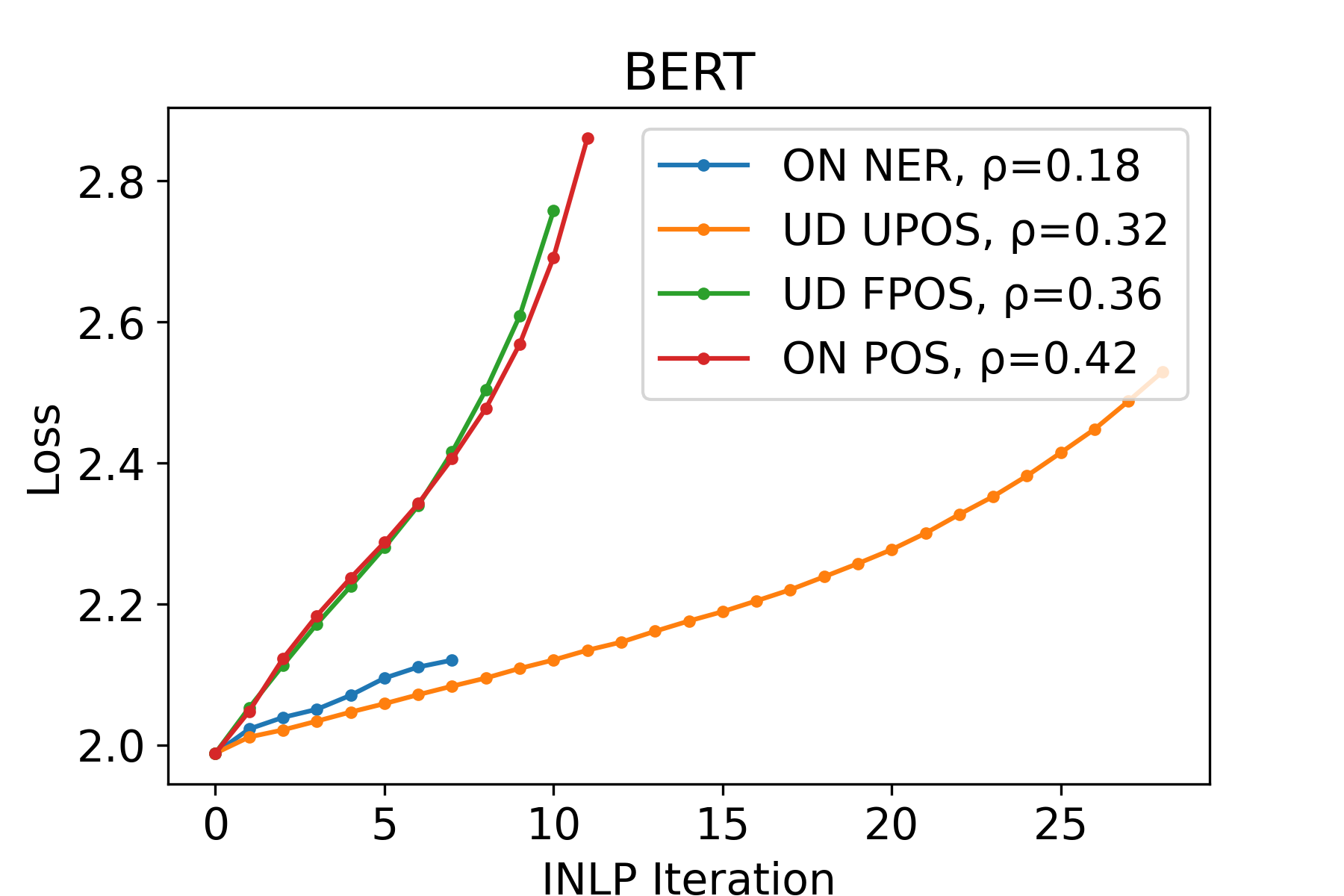}
\caption{Loss increase w.r.t. INLP iteration for different tasks. ON stands for OntoNotes, UD for Universal Dependencies. The UD EWT dataset has two types of POS annotation: coarse tags (UPOS) and fine-grained tags (FPOS)}
\label{fig:inlp_loss}
\end{figure}

\begin{figure}[h]
\begin{center}
\includegraphics[width=.4\textwidth]{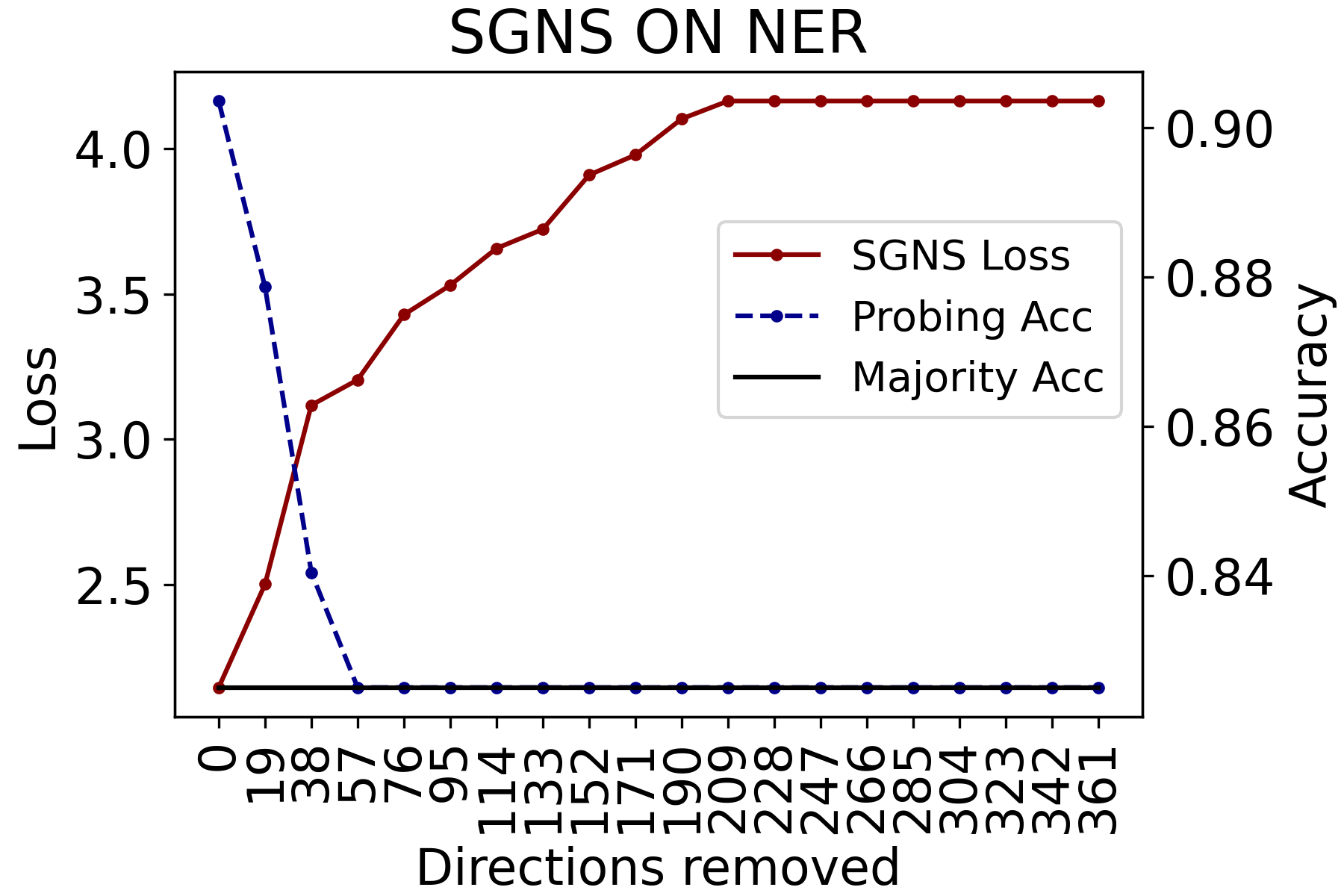}\hspace{15pt}
\includegraphics[width=.4\textwidth]{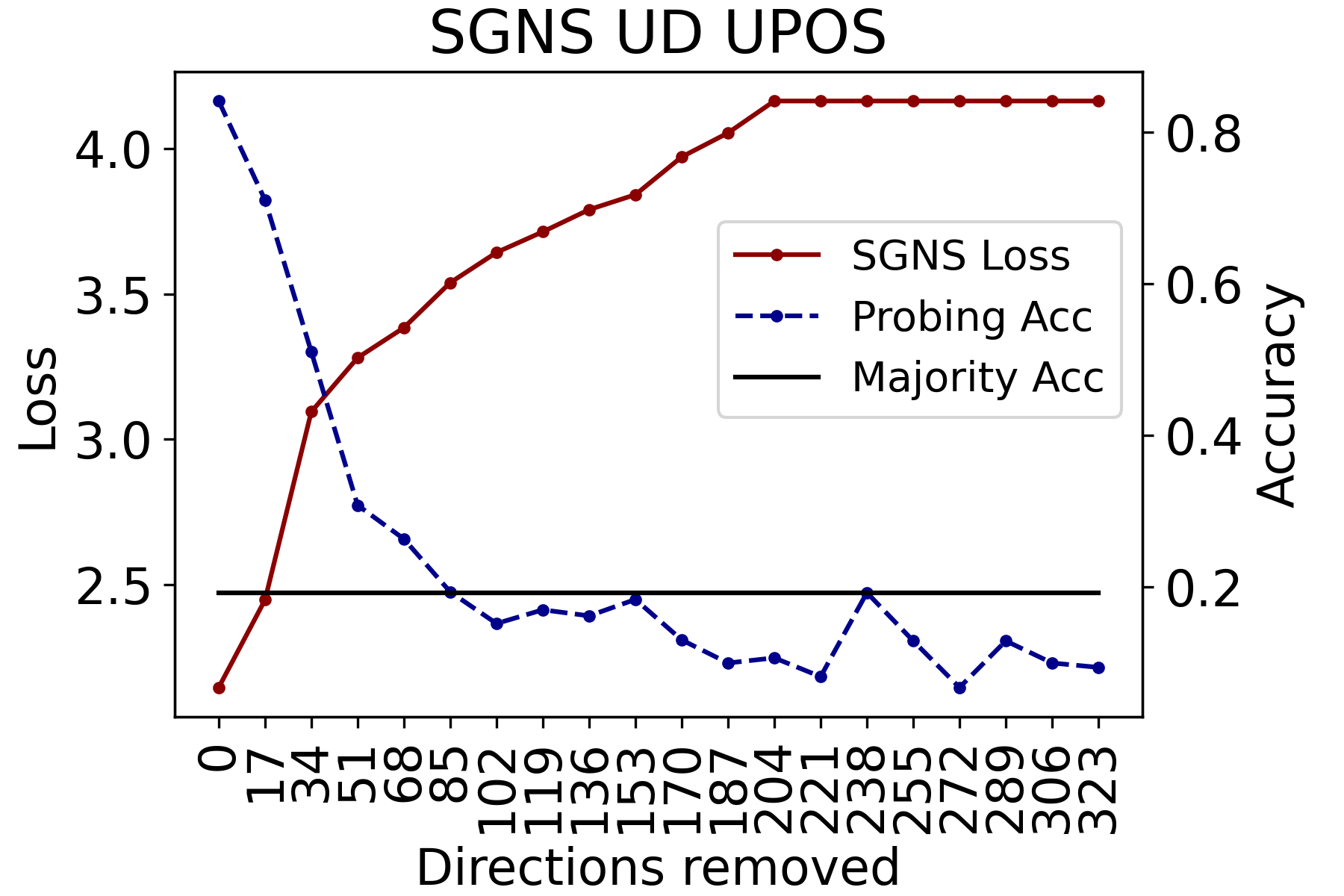}\\~\\
\includegraphics[width=.4\textwidth]{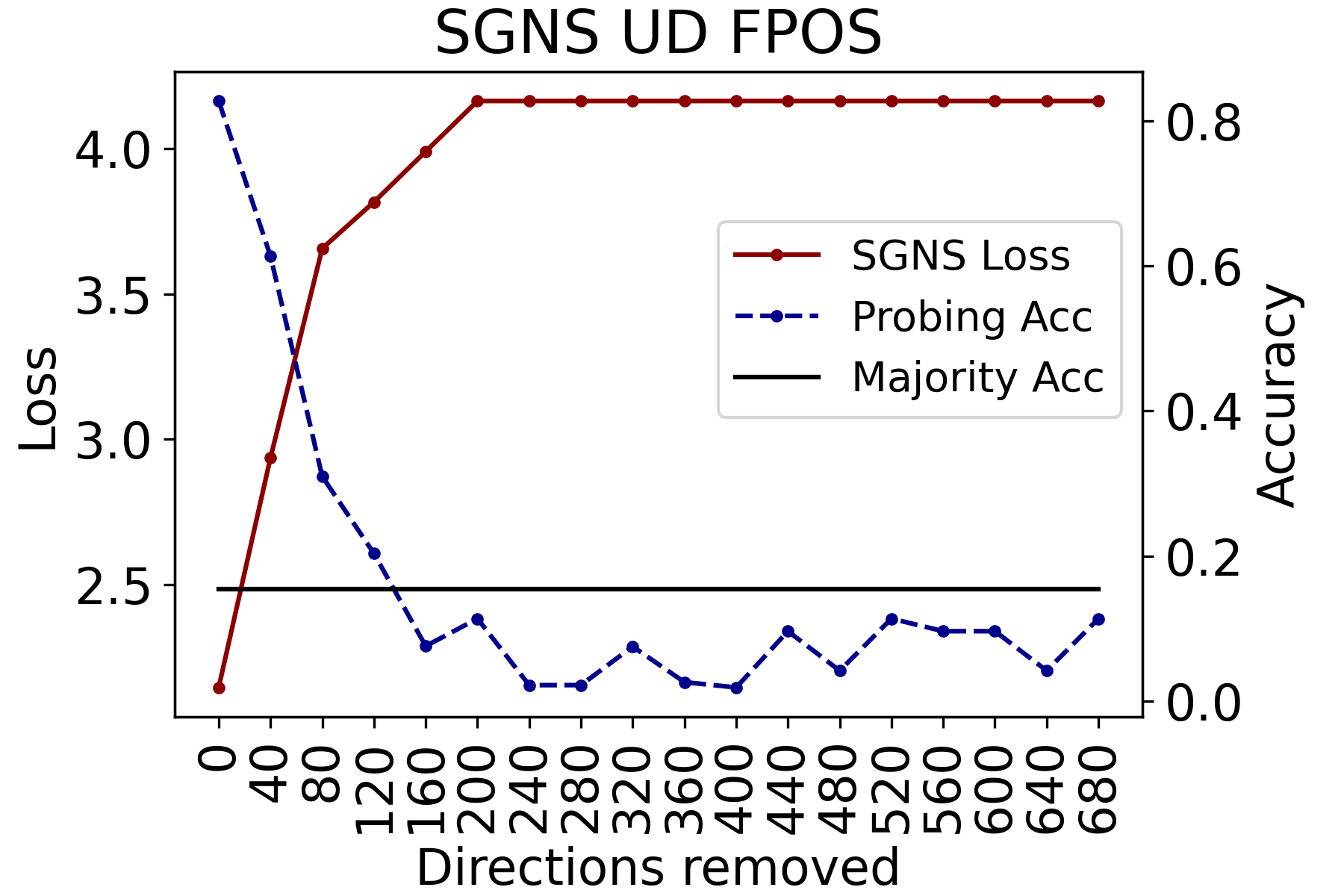}\hspace{15pt}
\includegraphics[width=.4\textwidth]{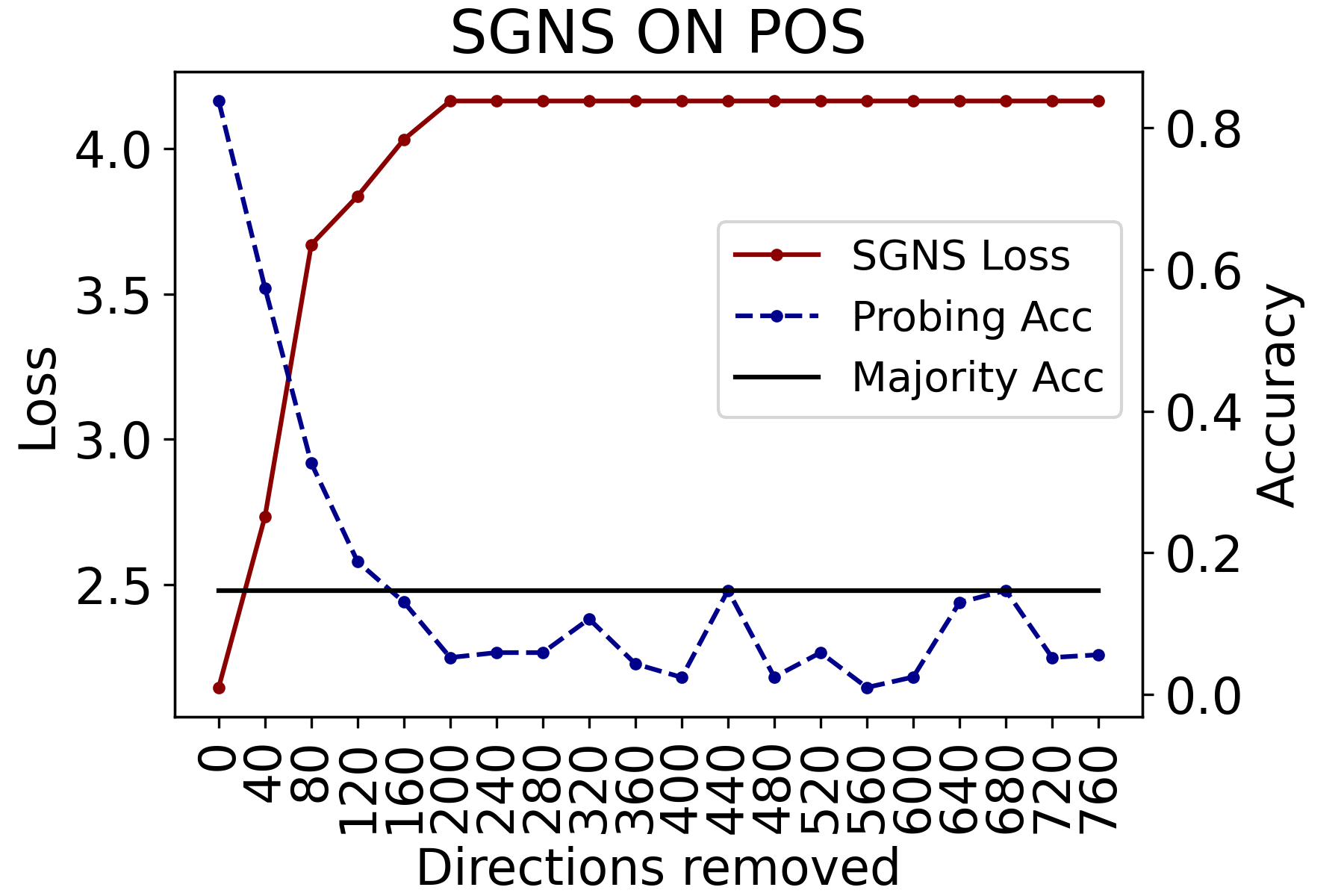}
\end{center}
\caption{INLP Dynamics for {\sc SGNS}. ON stands for OntoNotes, UD for Universal Dependencies. The UD EWT dataset has two types of POS annotation: coarse tags (UPOS) and fine-grained tags (FPOS).}
\label{fig:inlp_loss_acc_sgns}
\end{figure}

\begin{figure}[h]
\begin{center}
\includegraphics[width=.4\textwidth]{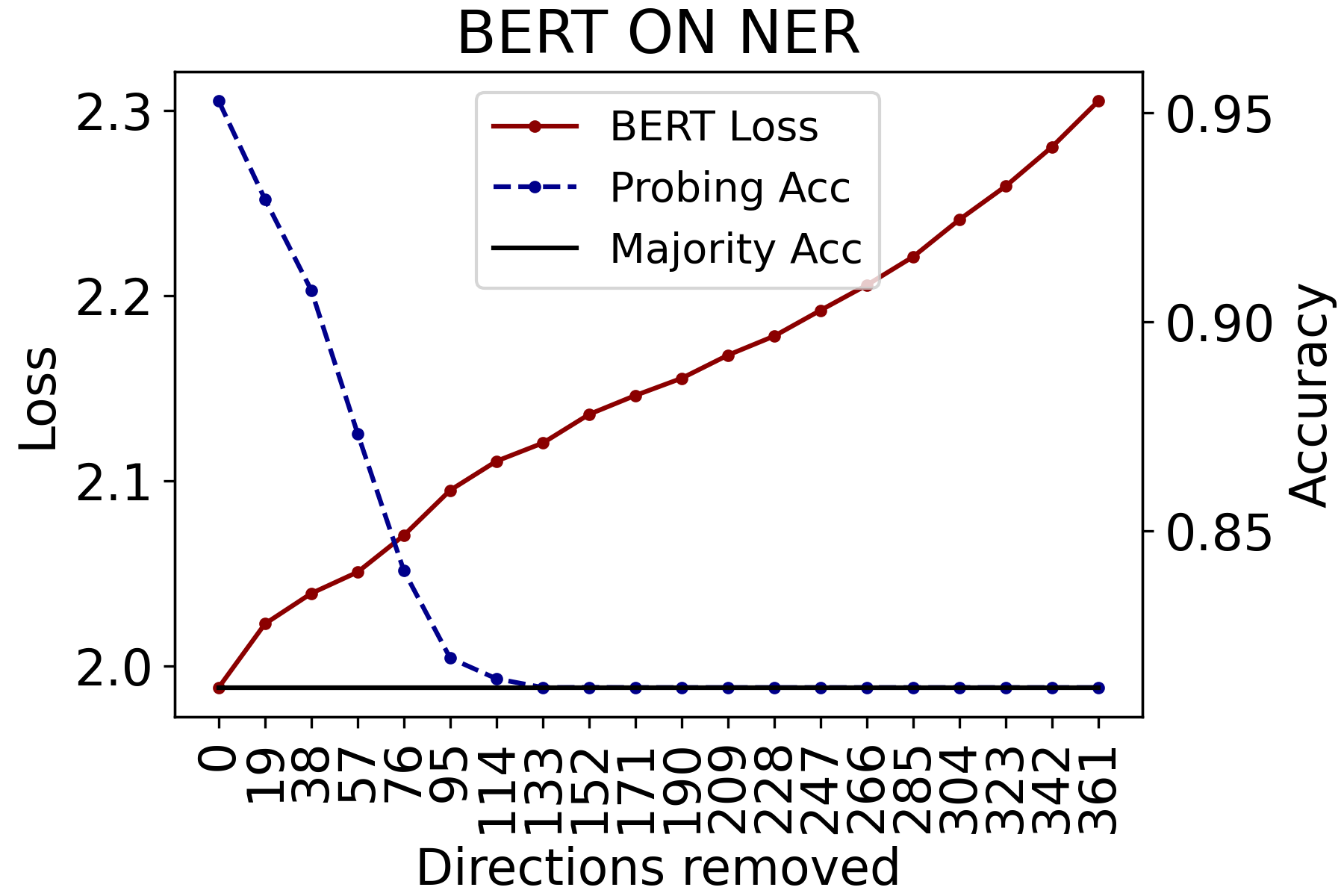}\hspace{15pt}
\includegraphics[width=.4\textwidth]{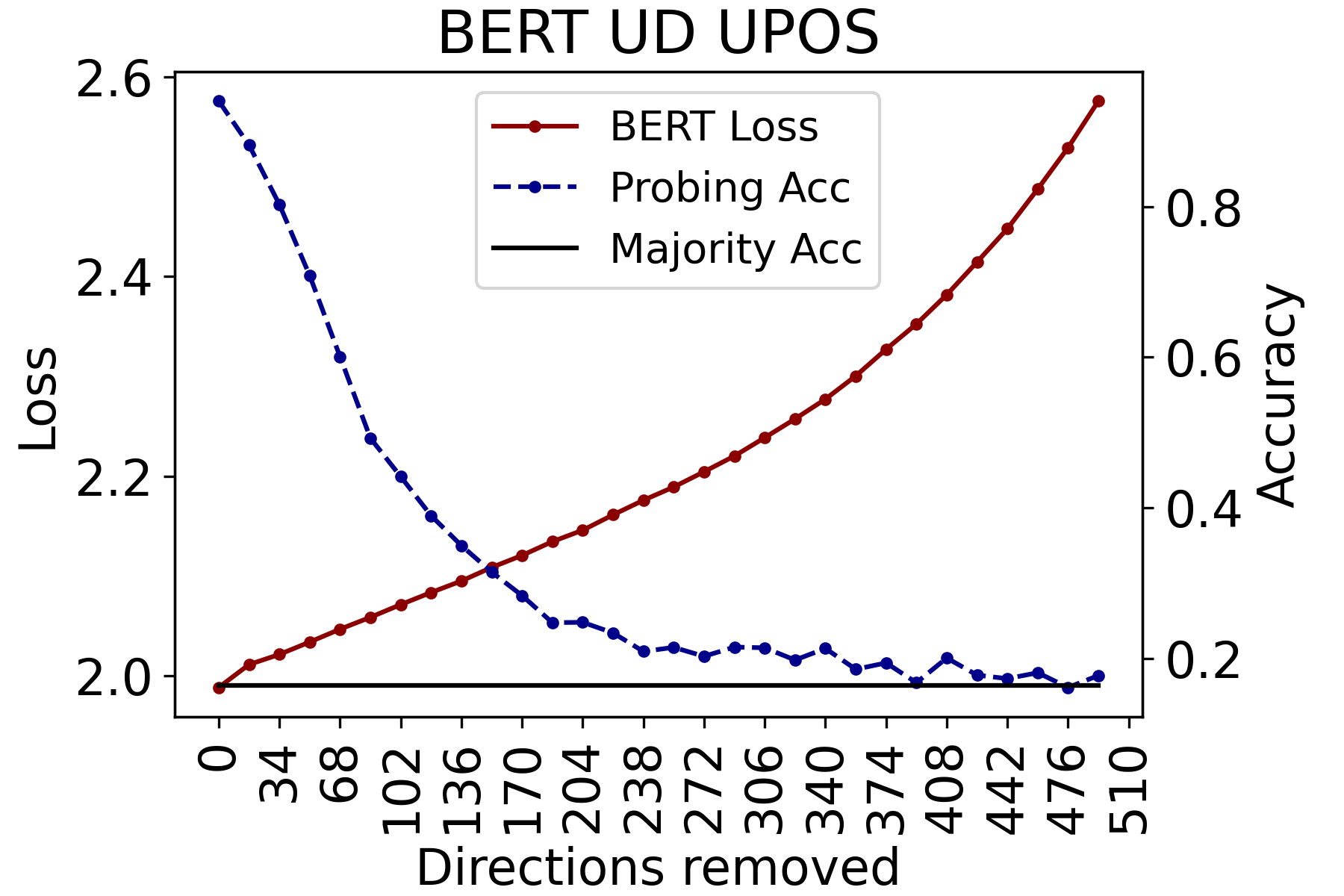}\\~\\
\includegraphics[width=.4\textwidth]{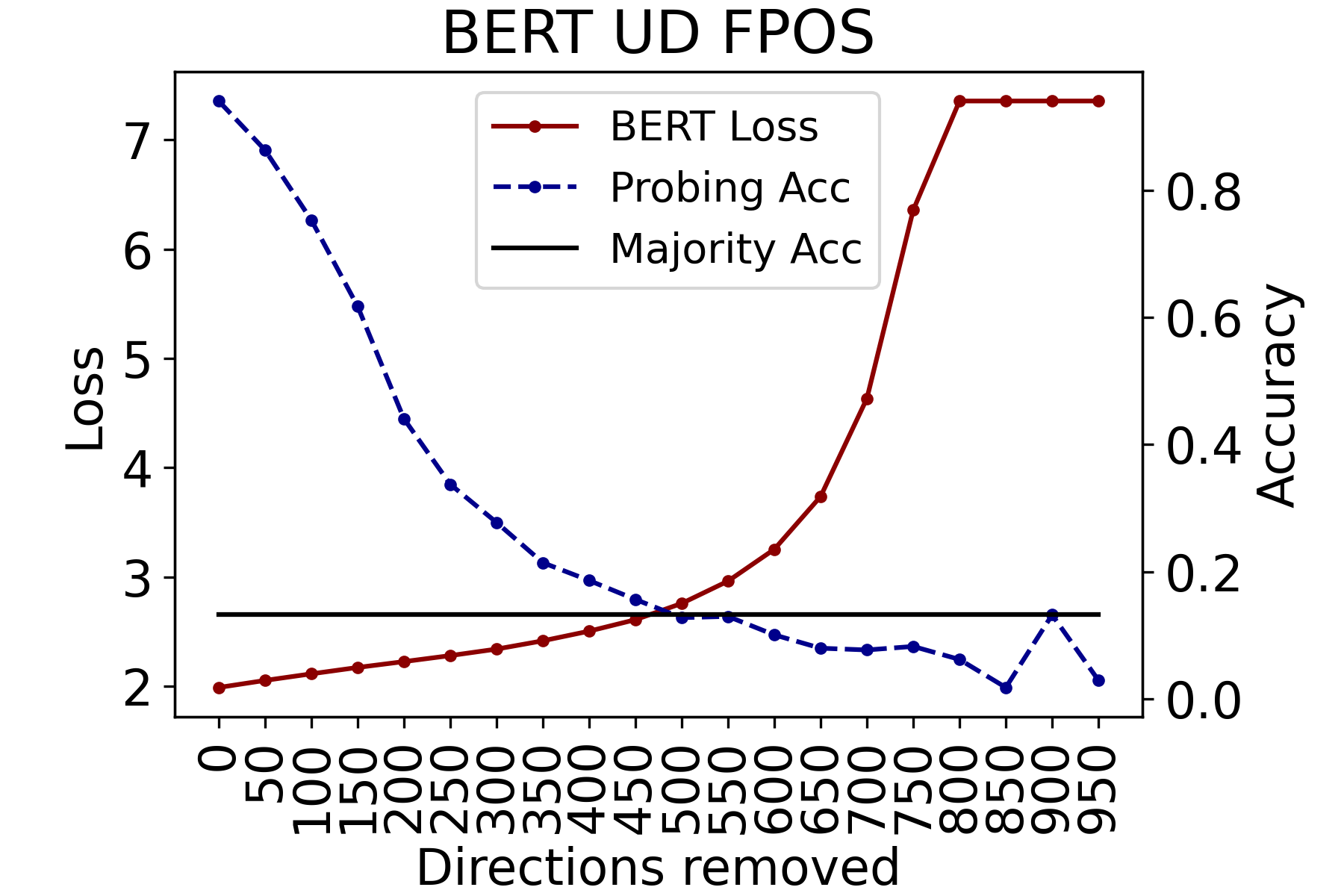}\hspace{15pt}
\includegraphics[width=.4\textwidth]{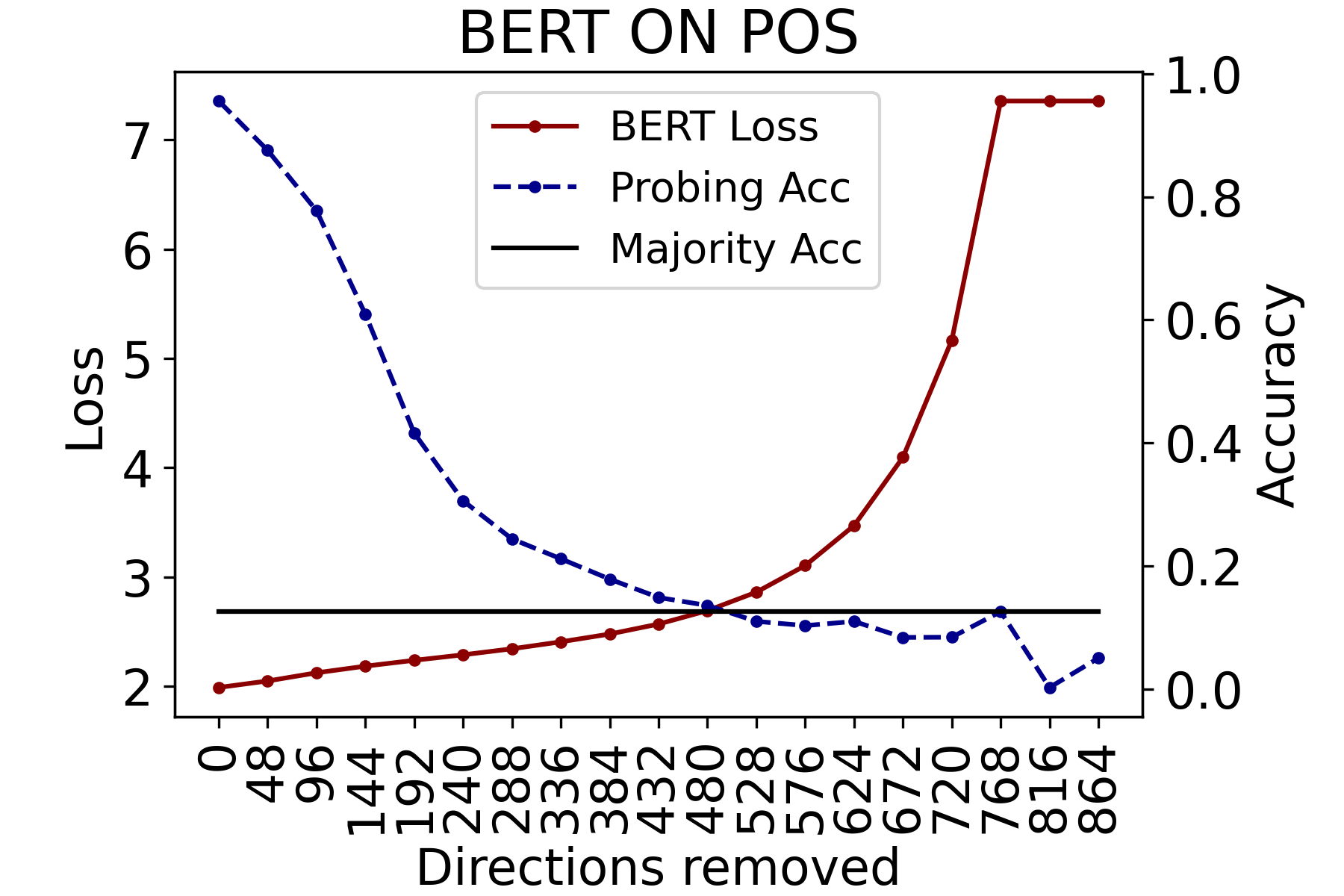}
\end{center}
\caption{INLP Dynamics for {\sc BERT}.}
\label{fig:inlp_loss_acc_bert}
\end{figure}

Figure~\ref{fig:inlp_loss} shows how quickly and how much the pre-training loss grows when applying the INLP procedure for various linguistic tasks. For each task, we show the minimum number of iterations at which the probing accuracy drops to (or below) the level of the majority accuracy.

We illustrate how pre-training loss and probing accuracy change with respect to INLP iteration in Figures~\ref{fig:inlp_loss_acc_sgns}~and~\ref{fig:inlp_loss_acc_bert}. The number of directions being removed at each INLP iteration is equal to the number of tags in the respective task.

\bibliography{ref}

\begin{thebibliography}{}

\bibitem[\protect\BCAY{Adi, Kermany, Belinkov, Lavi,\ \BBA\ Goldberg}{Adi
  et~al.}{2017}]{adi2017fine}
Adi, Y., Kermany, E., Belinkov, Y., Lavi, O., \BBA\ Goldberg, Y.
  \BBOP2017\BBCP.
\newblock \BBOQ Fine-grained analysis of sentence embeddings using auxiliary
  prediction tasks\BBCQ\
\newblock In {\Bem 5th International Conference on Learning Representations,
  {ICLR} 2017, Toulon, France, April 24-26, 2017, Conference Track
  Proceedings}. OpenReview.net.

\bibitem[\protect\BCAY{Allen\ \BBA\ Hospedales}{Allen\ \BBA\
  Hospedales}{2019}]{DBLP:conf/icml/AllenH19}
Allen, C.\BBACOMMA\  \BBA\ Hospedales, T.~M. \BBOP2019\BBCP.
\newblock \BBOQ Analogies explained: Towards understanding word
  embeddings\BBCQ\
\newblock In Chaudhuri, K.\BBACOMMA\  \BBA\ Salakhutdinov, R.\BEDS, {\Bem
  Proceedings of the 36th International Conference on Machine Learning, {ICML}
  2019, 9-15 June 2019, Long Beach, California, {USA}}, \lowercase{\BVOL}~97 of
  {\Bem Proceedings of Machine Learning Research}, \BPGS\ 223--231. {PMLR}.

\bibitem[\protect\BCAY{Antoniak\ \BBA\ Mimno}{Antoniak\ \BBA\
  Mimno}{2018}]{DBLP:journals/tacl/AntoniakM18}
Antoniak, M.\BBACOMMA\  \BBA\ Mimno, D. \BBOP2018\BBCP.
\newblock \BBOQ Evaluating the stability of embedding-based word
  similarities\BBCQ\
\newblock {\Bem Trans. Assoc. Comput. Linguistics}, {\Bem 6}, 107--119.

\bibitem[\protect\BCAY{Arora, Li, Liang, Ma,\ \BBA\ Risteski}{Arora
  et~al.}{2016}]{DBLP:journals/tacl/AroraLLMR16}
Arora, S., Li, Y., Liang, Y., Ma, T., \BBA\ Risteski, A. \BBOP2016\BBCP.
\newblock \BBOQ A latent variable model approach to pmi-based word
  embeddings\BBCQ\
\newblock {\Bem Trans. Assoc. Comput. Linguistics}, {\Bem 4}, 385--399.

\bibitem[\protect\BCAY{Assylbekov}{Assylbekov}{2020}]{sgns}
Assylbekov, Z. \BBOP2020\BBCP.
\newblock \BBOQ Sgns implementation in pytorch\BBCQ\
\newblock https://github.com/zh3nis/SGNS.
\newblock Accessed: 2021-11-07.

\bibitem[\protect\BCAY{Bahdanau, Cho,\ \BBA\ Bengio}{Bahdanau
  et~al.}{2015}]{bahdanau2014neural}
Bahdanau, D., Cho, K., \BBA\ Bengio, Y. \BBOP2015\BBCP.
\newblock \BBOQ Neural machine translation by jointly learning to align and
  translate\BBCQ\
\newblock In Bengio, Y.\BBACOMMA\  \BBA\ LeCun, Y.\BEDS, {\Bem 3rd
  International Conference on Learning Representations, {ICLR} 2015, San Diego,
  CA, USA, May 7-9, 2015, Conference Track Proceedings}.

\bibitem[\protect\BCAY{Bender\ \BBA\ Koller}{Bender\ \BBA\
  Koller}{2020}]{Bender2020}
Bender, E.~M.\BBACOMMA\  \BBA\ Koller, A. \BBOP2020\BBCP.
\newblock \BBOQ Climbing towards {NLU}: {On} meaning, form, and understanding
  in the age of data\BBCQ\
\newblock In {\Bem Proceedings of the 58th Annual Meeting of the Association
  for Computational Linguistics}, \BPGS\ 5185--5198, Online. Association for
  Computational Linguistics.

\bibitem[\protect\BCAY{Bengio\ \BBA\ Senecal}{Bengio\ \BBA\
  Senecal}{2003}]{DBLP:conf/aistats/BengioS03}
Bengio, Y.\BBACOMMA\  \BBA\ Senecal, J. \BBOP2003\BBCP.
\newblock \BBOQ Quick training of probabilistic neural nets by importance
  sampling\BBCQ\
\newblock In Bishop, C.~M.\BBACOMMA\  \BBA\ Frey, B.~J.\BEDS, {\Bem Proceedings
  of the Ninth International Workshop on Artificial Intelligence and
  Statistics, {AISTATS} 2003, Key West, Florida, USA, January 3-6, 2003}.
  Society for Artificial Intelligence and Statistics.

\bibitem[\protect\BCAY{Cettolo, Jan, Sebastian, Bentivogli, Cattoni,\ \BBA\
  Federico}{Cettolo et~al.}{2016}]{cettolo2016iwslt}
Cettolo, M., Jan, N., Sebastian, S., Bentivogli, L., Cattoni, R., \BBA\
  Federico, M. \BBOP2016\BBCP.
\newblock \BBOQ The {IWSLT} 2016 evaluation campaign\BBCQ\
\newblock In {\Bem Proceedings of IWSLT}.

\bibitem[\protect\BCAY{Chen, Frankle, Chang, Liu, Zhang, Wang,\ \BBA\
  Carbin}{Chen et~al.}{2020}]{DBLP:conf/nips/ChenFC0ZWC20}
Chen, T., Frankle, J., Chang, S., Liu, S., Zhang, Y., Wang, Z., \BBA\ Carbin,
  M. \BBOP2020\BBCP.
\newblock \BBOQ The lottery ticket hypothesis for pre-trained {BERT}
  networks\BBCQ\
\newblock In Larochelle, H., Ranzato, M., Hadsell, R., Balcan, M., \BBA\ Lin,
  H.\BEDS, {\Bem Advances in Neural Information Processing Systems 33: Annual
  Conference on Neural Information Processing Systems 2020, NeurIPS 2020,
  December 6-12, 2020, virtual}.

\bibitem[\protect\BCAY{Conneau, Kruszewski, Lample, Barrault,\ \BBA\
  Baroni}{Conneau et~al.}{2018}]{conneau2018you}
Conneau, A., Kruszewski, G., Lample, G., Barrault, L., \BBA\ Baroni, M.
  \BBOP2018\BBCP.
\newblock \BBOQ What you can cram into a single
  {\textbackslash}{\textdollar}{\&}!{\#}* vector: Probing sentence embeddings
  for linguistic properties\BBCQ\
\newblock In Gurevych, I.\BBACOMMA\  \BBA\ Miyao, Y.\BEDS, {\Bem Proceedings of
  the 56th Annual Meeting of the Association for Computational Linguistics,
  {ACL} 2018, Melbourne, Australia, July 15-20, 2018, Volume 1: Long Papers},
  \BPGS\ 2126--2136. Association for Computational Linguistics.

\bibitem[\protect\BCAY{Cortes\ \BBA\ Vapnik}{Cortes\ \BBA\
  Vapnik}{1995}]{cortes1995support}
Cortes, C.\BBACOMMA\  \BBA\ Vapnik, V. \BBOP1995\BBCP.
\newblock \BBOQ Support-vector networks\BBCQ\
\newblock {\Bem Mach. Learn.}, {\Bem 20\/}(3), 273--297.

\bibitem[\protect\BCAY{Devlin, Chang, Lee,\ \BBA\ Toutanova}{Devlin
  et~al.}{2019}]{devlin2019bert}
Devlin, J., Chang, M., Lee, K., \BBA\ Toutanova, K. \BBOP2019\BBCP.
\newblock \BBOQ {BERT:} pre-training of deep bidirectional transformers for
  language understanding\BBCQ\
\newblock In Burstein, J., Doran, C., \BBA\ Solorio, T.\BEDS, {\Bem Proceedings
  of the 2019 Conference of the North American Chapter of the Association for
  Computational Linguistics: Human Language Technologies, {NAACL-HLT} 2019,
  Minneapolis, MN, USA, June 2-7, 2019, Volume 1 (Long and Short Papers)},
  \BPGS\ 4171--4186. Association for Computational Linguistics.

\bibitem[\protect\BCAY{Eger, Daxenberger,\ \BBA\ Gurevych}{Eger
  et~al.}{2020}]{eger-etal-2020-probe}
Eger, S., Daxenberger, J., \BBA\ Gurevych, I. \BBOP2020\BBCP.
\newblock \BBOQ How to probe sentence embeddings in low-resource languages: On
  structural design choices for probing task evaluation\BBCQ\
\newblock In {\Bem Proceedings of the 24th Conference on Computational Natural
  Language Learning}, \BPGS\ 108--118, Online. Association for Computational
  Linguistics.

\bibitem[\protect\BCAY{Elazar\ \BBA\ Goldberg}{Elazar\ \BBA\
  Goldberg}{2018}]{elazar2018adversarial}
Elazar, Y.\BBACOMMA\  \BBA\ Goldberg, Y. \BBOP2018\BBCP.
\newblock \BBOQ Adversarial removal of demographic attributes from text
  data\BBCQ\
\newblock In Riloff, E., Chiang, D., Hockenmaier, J., \BBA\ Tsujii, J.\BEDS,
  {\Bem Proceedings of the 2018 Conference on Empirical Methods in Natural
  Language Processing, Brussels, Belgium, October 31 - November 4, 2018},
  \BPGS\ 11--21. Association for Computational Linguistics.

\bibitem[\protect\BCAY{Elazar, Ravfogel, Jacovi,\ \BBA\ Goldberg}{Elazar
  et~al.}{2021}]{elazar2020bert}
Elazar, Y., Ravfogel, S., Jacovi, A., \BBA\ Goldberg, Y. \BBOP2021\BBCP.
\newblock \BBOQ Amnesic probing: Behavioral explanation with amnesic
  counterfactuals\BBCQ\
\newblock {\Bem Trans. Assoc. Comput. Linguistics}, {\Bem 9}, 160--175.

\bibitem[\protect\BCAY{Ethayarajh, Duvenaud,\ \BBA\ Hirst}{Ethayarajh
  et~al.}{2019}]{DBLP:conf/acl/EthayarajhDH19a}
Ethayarajh, K., Duvenaud, D., \BBA\ Hirst, G. \BBOP2019\BBCP.
\newblock \BBOQ Towards understanding linear word analogies\BBCQ\
\newblock In Korhonen, A., Traum, D.~R., \BBA\ M{\`{a}}rquez, L.\BEDS, {\Bem
  Proceedings of the 57th Conference of the Association for Computational
  Linguistics, {ACL} 2019, Florence, Italy, July 28- August 2, 2019, Volume 1:
  Long Papers}, \BPGS\ 3253--3262. Association for Computational Linguistics.

\bibitem[\protect\BCAY{Ettinger, Elgohary,\ \BBA\ Resnik}{Ettinger
  et~al.}{2016}]{ettinger2016}
Ettinger, A., Elgohary, A., \BBA\ Resnik, P. \BBOP2016\BBCP.
\newblock \BBOQ Probing for semantic evidence of composition by means of simple
  classification tasks\BBCQ\
\newblock In {\Bem Proceedings of the 1st Workshop on Evaluating Vector-Space
  Representations for {NLP}}, \BPGS\ 134--139, Berlin, Germany. Association for
  Computational Linguistics.

\bibitem[\protect\BCAY{Finkelstein, Gabrilovich, Matias, Rivlin, Solan,
  Wolfman,\ \BBA\ Ruppin}{Finkelstein et~al.}{2002}]{finkelstein2002placing}
Finkelstein, L., Gabrilovich, E., Matias, Y., Rivlin, E., Solan, Z., Wolfman,
  G., \BBA\ Ruppin, E. \BBOP2002\BBCP.
\newblock \BBOQ Placing search in context: the concept revisited\BBCQ\
\newblock {\Bem {ACM} Trans. Inf. Syst.}, {\Bem 20\/}(1), 116--131.

\bibitem[\protect\BCAY{Firth}{Firth}{1957}]{Firth57}
Firth, J.~R. \BBOP1957\BBCP.
\newblock \BBOQ A synopsis of linguistic theory, 1930-1955\BBCQ\
\newblock {\Bem Studies in linguistic analysis. Special volume of the
  Philological Society}, {\Bem 22\/}(1).

\bibitem[\protect\BCAY{Frankle\ \BBA\ Carbin}{Frankle\ \BBA\
  Carbin}{2019}]{frankle2018the}
Frankle, J.\BBACOMMA\  \BBA\ Carbin, M. \BBOP2019\BBCP.
\newblock \BBOQ The lottery ticket hypothesis: Finding sparse, trainable neural
  networks\BBCQ\
\newblock In {\Bem 7th International Conference on Learning Representations,
  {ICLR} 2019, New Orleans, LA, USA, May 6-9, 2019}. OpenReview.net.

\bibitem[\protect\BCAY{Ganin\ \BBA\ Lempitsky}{Ganin\ \BBA\
  Lempitsky}{2015}]{ganin2015unsupervised}
Ganin, Y.\BBACOMMA\  \BBA\ Lempitsky, V.~S. \BBOP2015\BBCP.
\newblock \BBOQ Unsupervised domain adaptation by backpropagation\BBCQ\
\newblock In Bach, F.~R.\BBACOMMA\  \BBA\ Blei, D.~M.\BEDS, {\Bem Proceedings
  of the 32nd International Conference on Machine Learning, {ICML} 2015, Lille,
  France, 6-11 July 2015}, \lowercase{\BVOL}~37 of {\Bem {JMLR} Workshop and
  Conference Proceedings}, \BPGS\ 1180--1189. JMLR.org.

\bibitem[\protect\BCAY{Gittens, Achlioptas,\ \BBA\ Mahoney}{Gittens
  et~al.}{2017}]{DBLP:conf/acl/GittensAM17}
Gittens, A., Achlioptas, D., \BBA\ Mahoney, M.~W. \BBOP2017\BBCP.
\newblock \BBOQ Skip-gram - zipf + uniform = vector additivity\BBCQ\
\newblock In Barzilay, R.\BBACOMMA\  \BBA\ Kan, M.\BEDS, {\Bem Proceedings of
  the 55th Annual Meeting of the Association for Computational Linguistics,
  {ACL} 2017, Vancouver, Canada, July 30 - August 4, Volume 1: Long Papers},
  \BPGS\ 69--76. Association for Computational Linguistics.

\bibitem[\protect\BCAY{Gordon, Duh,\ \BBA\ Andrews}{Gordon
  et~al.}{2020}]{gordon-etal-2020-compressing}
Gordon, M.~A., Duh, K., \BBA\ Andrews, N. \BBOP2020\BBCP.
\newblock \BBOQ Compressing {BERT:} studying the effects of weight pruning on
  transfer learning\BBCQ\
\newblock In Gella, S., Welbl, J., Rei, M., Petroni, F., Lewis, P. S.~H.,
  Strubell, E., Seo, M.~J., \BBA\ Hajishirzi, H.\BEDS, {\Bem Proceedings of the
  5th Workshop on Representation Learning for NLP, RepL4NLP@ACL 2020, Online,
  July 9, 2020}, \BPGS\ 143--155. Association for Computational Linguistics.

\bibitem[\protect\BCAY{Gulordava, Bojanowski, Grave, Linzen,\ \BBA\
  Baroni}{Gulordava et~al.}{2018}]{gulordava2018colorless}
Gulordava, K., Bojanowski, P., Grave, E., Linzen, T., \BBA\ Baroni, M.
  \BBOP2018\BBCP.
\newblock \BBOQ Colorless green recurrent networks dream hierarchically\BBCQ\
\newblock In Walker, M.~A., Ji, H., \BBA\ Stent, A.\BEDS, {\Bem Proceedings of
  the 2018 Conference of the North American Chapter of the Association for
  Computational Linguistics: Human Language Technologies, {NAACL-HLT} 2018, New
  Orleans, Louisiana, USA, June 1-6, 2018, Volume 1 (Long Papers)}, \BPGS\
  1195--1205. Association for Computational Linguistics.

\bibitem[\protect\BCAY{Hashimoto, Alvarez{-}Melis,\ \BBA\ Jaakkola}{Hashimoto
  et~al.}{2016}]{DBLP:journals/tacl/HashimotoAJ16}
Hashimoto, T.~B., Alvarez{-}Melis, D., \BBA\ Jaakkola, T.~S. \BBOP2016\BBCP.
\newblock \BBOQ Word embeddings as metric recovery in semantic spaces\BBCQ\
\newblock {\Bem Trans. Assoc. Comput. Linguistics}, {\Bem 4}, 273--286.

\bibitem[\protect\BCAY{Hewitt}{Hewitt}{2019}]{structprobe}
Hewitt, J. \BBOP2019\BBCP.
\newblock \BBOQ structural-probes\BBCQ\
\newblock \url{https://github.com/john-hewitt/structural-probes}.
\newblock Accessed: 2021-11-07.

\bibitem[\protect\BCAY{Hewitt\ \BBA\ Liang}{Hewitt\ \BBA\
  Liang}{2019}]{hewitt2019designing}
Hewitt, J.\BBACOMMA\  \BBA\ Liang, P. \BBOP2019\BBCP.
\newblock \BBOQ Designing and interpreting probes with control tasks\BBCQ\
\newblock In Inui, K., Jiang, J., Ng, V., \BBA\ Wan, X.\BEDS, {\Bem Proceedings
  of the 2019 Conference on Empirical Methods in Natural Language Processing
  and the 9th International Joint Conference on Natural Language Processing,
  {EMNLP-IJCNLP} 2019, Hong Kong, China, November 3-7, 2019}, \BPGS\
  2733--2743. Association for Computational Linguistics.

\bibitem[\protect\BCAY{Hewitt\ \BBA\ Manning}{Hewitt\ \BBA\
  Manning}{2019}]{hewitt2019structural}
Hewitt, J.\BBACOMMA\  \BBA\ Manning, C.~D. \BBOP2019\BBCP.
\newblock \BBOQ A structural probe for finding syntax in word
  representations\BBCQ\
\newblock In Burstein, J., Doran, C., \BBA\ Solorio, T.\BEDS, {\Bem Proceedings
  of the 2019 Conference of the North American Chapter of the Association for
  Computational Linguistics: Human Language Technologies, {NAACL-HLT} 2019,
  Minneapolis, MN, USA, June 2-7, 2019, Volume 1 (Long and Short Papers)},
  \BPGS\ 4129--4138. Association for Computational Linguistics.

\bibitem[\protect\BCAY{Inan, Khosravi,\ \BBA\ Socher}{Inan
  et~al.}{2017}]{DBLP:conf/iclr/InanKS17}
Inan, H., Khosravi, K., \BBA\ Socher, R. \BBOP2017\BBCP.
\newblock \BBOQ Tying word vectors and word classifiers: {A} loss framework for
  language modeling\BBCQ\
\newblock In {\Bem 5th International Conference on Learning Representations,
  {ICLR} 2017, Toulon, France, April 24-26, 2017, Conference Track
  Proceedings}. OpenReview.net.

\bibitem[\protect\BCAY{Klein, Kim, Deng, Senellart,\ \BBA\ Rush}{Klein
  et~al.}{2017}]{klein2017opennmt}
Klein, G., Kim, Y., Deng, Y., Senellart, J., \BBA\ Rush, A.~M. \BBOP2017\BBCP.
\newblock \BBOQ {OpenNMT}: Open-source toolkit for neural machine
  translation\BBCQ\
\newblock In Bansal, M.\BBACOMMA\  \BBA\ Ji, H.\BEDS, {\Bem Proceedings of the
  55th Annual Meeting of the Association for Computational Linguistics, {ACL}
  2017, Vancouver, Canada, July 30 - August 4, System Demonstrations}, \BPGS\
  67--72. Association for Computational Linguistics.

\bibitem[\protect\BCAY{Kong, de~Masson~d'Autume, Yu, Ling, Dai,\ \BBA\
  Yogatama}{Kong et~al.}{2020}]{DBLP:conf/iclr/KongdYLDY20}
Kong, L., de~Masson~d'Autume, C., Yu, L., Ling, W., Dai, Z., \BBA\ Yogatama, D.
  \BBOP2020\BBCP.
\newblock \BBOQ A mutual information maximization perspective of language
  representation learning\BBCQ\
\newblock In {\Bem 8th International Conference on Learning Representations,
  {ICLR} 2020, Addis Ababa, Ethiopia, April 26-30, 2020}. OpenReview.net.

\bibitem[\protect\BCAY{Kuncoro, Dyer, Hale, Yogatama, Clark,\ \BBA\
  Blunsom}{Kuncoro et~al.}{2018}]{kuncoro2018lstms}
Kuncoro, A., Dyer, C., Hale, J., Yogatama, D., Clark, S., \BBA\ Blunsom, P.
  \BBOP2018\BBCP.
\newblock \BBOQ {LSTMs} can learn syntax-sensitive dependencies well, but
  modeling structure makes them better\BBCQ\
\newblock In Gurevych, I.\BBACOMMA\  \BBA\ Miyao, Y.\BEDS, {\Bem Proceedings of
  the 56th Annual Meeting of the Association for Computational Linguistics,
  {ACL} 2018, Melbourne, Australia, July 15-20, 2018, Volume 1: Long Papers},
  \BPGS\ 1426--1436. Association for Computational Linguistics.

\bibitem[\protect\BCAY{Kunz\ \BBA\ Kuhlmann}{Kunz\ \BBA\
  Kuhlmann}{2020}]{DBLP:conf/coling/KunzK20}
Kunz, J.\BBACOMMA\  \BBA\ Kuhlmann, M. \BBOP2020\BBCP.
\newblock \BBOQ Classifier probes may just learn from linear context
  features\BBCQ\
\newblock In Scott, D., Bel, N., \BBA\ Zong, C.\BEDS, {\Bem Proceedings of the
  28th International Conference on Computational Linguistics, {COLING} 2020,
  Barcelona, Spain (Online), December 8-13, 2020}, \BPGS\ 5136--5146.
  International Committee on Computational Linguistics.

\bibitem[\protect\BCAY{Lee, Lei, Saunshi,\ \BBA\ Zhuo}{Lee
  et~al.}{2020}]{DBLP:journals/corr/abs-2008-01064}
Lee, J.~D., Lei, Q., Saunshi, N., \BBA\ Zhuo, J. \BBOP2020\BBCP.
\newblock \BBOQ Predicting what you already know helps: Provable
  self-supervised learning\BBCQ\
\newblock {\Bem CoRR}, {\Bem abs/2008.01064}.

\bibitem[\protect\BCAY{Levy\ \BBA\ Goldberg}{Levy\ \BBA\
  Goldberg}{2014}]{DBLP:conf/nips/LevyG14}
Levy, O.\BBACOMMA\  \BBA\ Goldberg, Y. \BBOP2014\BBCP.
\newblock \BBOQ Neural word embedding as implicit matrix factorization\BBCQ\
\newblock In Ghahramani, Z., Welling, M., Cortes, C., Lawrence, N.~D., \BBA\
  Weinberger, K.~Q.\BEDS, {\Bem Advances in Neural Information Processing
  Systems 27: Annual Conference on Neural Information Processing Systems 2014,
  December 8-13 2014, Montreal, Quebec, Canada}, \BPGS\ 2177--2185.

\bibitem[\protect\BCAY{Linzen, Dupoux,\ \BBA\ Goldberg}{Linzen
  et~al.}{2016}]{linzen2016assessing}
Linzen, T., Dupoux, E., \BBA\ Goldberg, Y. \BBOP2016\BBCP.
\newblock \BBOQ Assessing the ability of {LSTMs} to learn syntax-sensitive
  dependencies\BBCQ\
\newblock {\Bem Trans. Assoc. Comput. Linguistics}, {\Bem 4}, 521--535.

\bibitem[\protect\BCAY{Liu, Ott, Goyal, Du, Joshi, Chen, Levy, Lewis,
  Zettlemoyer,\ \BBA\ Stoyanov}{Liu et~al.}{2019}]{liu2019roberta}
Liu, Y., Ott, M., Goyal, N., Du, J., Joshi, M., Chen, D., Levy, O., Lewis, M.,
  Zettlemoyer, L., \BBA\ Stoyanov, V. \BBOP2019\BBCP.
\newblock \BBOQ {RoBERTa}: {A} robustly optimized {BERT} pretraining
  approach\BBCQ\
\newblock {\Bem CoRR}, {\Bem abs/1907.11692}.

\bibitem[\protect\BCAY{Mahoney}{Mahoney}{2011}]{text8}
Mahoney, M. \BBOP2011\BBCP.
\newblock \BBOQ About the test data\BBCQ\
\newblock \url{http://mattmahoney.net/dc/textdata.html}.
\newblock Accessed: 2021-11-07.

\bibitem[\protect\BCAY{Marcus, Santorini,\ \BBA\ Marcinkiewicz}{Marcus
  et~al.}{1993}]{DBLP:journals/coling/MarcusSM94}
Marcus, M.~P., Santorini, B., \BBA\ Marcinkiewicz, M.~A. \BBOP1993\BBCP.
\newblock \BBOQ Building a large annotated corpus of english: The penn
  treebank\BBCQ\
\newblock {\Bem Comput. Linguistics}, {\Bem 19\/}(2), 313--330.

\bibitem[\protect\BCAY{Maudslay, Valvoda, Pimentel, Williams,\ \BBA\
  Cotterell}{Maudslay et~al.}{2020}]{maudslay2020}
Maudslay, R.~H., Valvoda, J., Pimentel, T., Williams, A., \BBA\ Cotterell, R.
  \BBOP2020\BBCP.
\newblock \BBOQ A tale of a probe and a parser\BBCQ\
\newblock In Jurafsky, D., Chai, J., Schluter, N., \BBA\ Tetreault, J.~R.\BEDS,
  {\Bem Proceedings of the 58th Annual Meeting of the Association for
  Computational Linguistics, {ACL} 2020, Online, July 5-10, 2020}, \BPGS\
  7389--7395. Association for Computational Linguistics.

\bibitem[\protect\BCAY{McCann, Bradbury, Xiong,\ \BBA\ Socher}{McCann
  et~al.}{2017}]{mccann2017learned}
McCann, B., Bradbury, J., Xiong, C., \BBA\ Socher, R. \BBOP2017\BBCP.
\newblock \BBOQ Learned in translation: Contextualized word vectors\BBCQ\
\newblock In Guyon, I., von Luxburg, U., Bengio, S., Wallach, H.~M., Fergus,
  R., Vishwanathan, S. V.~N., \BBA\ Garnett, R.\BEDS, {\Bem Advances in Neural
  Information Processing Systems 30: Annual Conference on Neural Information
  Processing Systems 2017, December 4-9, 2017, Long Beach, CA, {USA}}, \BPGS\
  6294--6305.

\bibitem[\protect\BCAY{Merity, Xiong, Bradbury,\ \BBA\ Socher}{Merity
  et~al.}{2017}]{merity2016pointer}
Merity, S., Xiong, C., Bradbury, J., \BBA\ Socher, R. \BBOP2017\BBCP.
\newblock \BBOQ Pointer sentinel mixture models\BBCQ\
\newblock In {\Bem 5th International Conference on Learning Representations,
  {ICLR} 2017, Toulon, France, April 24-26, 2017, Conference Track
  Proceedings}. OpenReview.net.

\bibitem[\protect\BCAY{Michael, Botha,\ \BBA\ Tenney}{Michael
  et~al.}{2020}]{DBLP:conf/emnlp/MichaelBT20}
Michael, J., Botha, J.~A., \BBA\ Tenney, I. \BBOP2020\BBCP.
\newblock \BBOQ Asking without telling: Exploring latent ontologies in
  contextual representations\BBCQ\
\newblock In Webber, B., Cohn, T., He, Y., \BBA\ Liu, Y.\BEDS, {\Bem
  Proceedings of the 2020 Conference on Empirical Methods in Natural Language
  Processing, {EMNLP} 2020, Online, November 16-20, 2020}, \BPGS\ 6792--6812.
  Association for Computational Linguistics.

\bibitem[\protect\BCAY{Mikolov, Chen, Corrado,\ \BBA\ Dean}{Mikolov
  et~al.}{2013a}]{mikolov2013efficient}
Mikolov, T., Chen, K., Corrado, G., \BBA\ Dean, J. \BBOP2013a\BBCP.
\newblock \BBOQ Efficient estimation of word representations in vector
  space\BBCQ\
\newblock In Bengio, Y.\BBACOMMA\  \BBA\ LeCun, Y.\BEDS, {\Bem 1st
  International Conference on Learning Representations, {ICLR} 2013,
  Scottsdale, Arizona, USA, May 2-4, 2013, Workshop Track Proceedings}.

\bibitem[\protect\BCAY{Mikolov, Sutskever, Chen, Corrado,\ \BBA\ Dean}{Mikolov
  et~al.}{2013b}]{mikolov2013distributed}
Mikolov, T., Sutskever, I., Chen, K., Corrado, G.~S., \BBA\ Dean, J.
  \BBOP2013b\BBCP.
\newblock \BBOQ Distributed representations of words and phrases and their
  compositionality\BBCQ\
\newblock In Burges, C. J.~C., Bottou, L., Ghahramani, Z., \BBA\ Weinberger,
  K.~Q.\BEDS, {\Bem Advances in Neural Information Processing Systems 26: 27th
  Annual Conference on Neural Information Processing Systems 2013. Proceedings
  of a meeting held December 5-8, 2013, Lake Tahoe, Nevada, United States},
  \BPGS\ 3111--3119.

\bibitem[\protect\BCAY{Mikolov, Yih,\ \BBA\ Zweig}{Mikolov
  et~al.}{2013c}]{DBLP:conf/naacl/MikolovYZ13}
Mikolov, T., Yih, W., \BBA\ Zweig, G. \BBOP2013c\BBCP.
\newblock \BBOQ Linguistic regularities in continuous space word
  representations\BBCQ\
\newblock In Vanderwende, L., III, H.~D., \BBA\ Kirchhoff, K.\BEDS, {\Bem Human
  Language Technologies: Conference of the North American Chapter of the
  Association of Computational Linguistics, Proceedings, June 9-14, 2013,
  Westin Peachtree Plaza Hotel, Atlanta, Georgia, {USA}}, \BPGS\ 746--751. The
  Association for Computational Linguistics.

\bibitem[\protect\BCAY{Mikolov, Yih,\ \BBA\ Zweig}{Mikolov
  et~al.}{2013d}]{mikolov2013linguistic}
Mikolov, T., Yih, W., \BBA\ Zweig, G. \BBOP2013d\BBCP.
\newblock \BBOQ Linguistic regularities in continuous space word
  representations\BBCQ\
\newblock In Vanderwende, L., III, H.~D., \BBA\ Kirchhoff, K.\BEDS, {\Bem Human
  Language Technologies: Conference of the North American Chapter of the
  Association of Computational Linguistics, Proceedings, June 9-14, 2013,
  Westin Peachtree Plaza Hotel, Atlanta, Georgia, {USA}}, \BPGS\ 746--751. The
  Association for Computational Linguistics.

\bibitem[\protect\BCAY{Ney, Essen,\ \BBA\ Kneser}{Ney
  et~al.}{1994}]{DBLP:journals/csl/NeyEK94}
Ney, H., Essen, U., \BBA\ Kneser, R. \BBOP1994\BBCP.
\newblock \BBOQ On structuring probabilistic dependences in stochastic language
  modelling\BBCQ\
\newblock {\Bem Comput. Speech Lang.}, {\Bem 8\/}(1), 1--38.

\bibitem[\protect\BCAY{Olah}{Olah}{2015}]{45511}
Olah, C. \BBOP2015\BBCP.
\newblock \BBOQ Visual information theory\BBCQ\
\newblock \url{http://colah.github.io/posts/2015-09-Visual-Information/}.

\bibitem[\protect\BCAY{Ott, Baevski, Martin,\ \BBA\ Morton}{Ott
  et~al.}{2019a}]{RoBERTa}
Ott, M., Baevski, A., Martin, L., \BBA\ Morton, J. \BBOP2019a\BBCP.
\newblock \BBOQ Pretraining {RoBERTa} using your own data\BBCQ\
\newblock
  \url{https://github.com/pytorch/fairseq/blob/main/examples/roberta/README.pretraining.md}.
\newblock Accessed: 2021-11-07.

\bibitem[\protect\BCAY{Ott, Edunov, Baevski, Fan, Gross, Ng, Grangier,\ \BBA\
  Auli}{Ott et~al.}{2019b}]{ott2019fairseq}
Ott, M., Edunov, S., Baevski, A., Fan, A., Gross, S., Ng, N., Grangier, D.,
  \BBA\ Auli, M. \BBOP2019b\BBCP.
\newblock \BBOQ fairseq: {A} fast, extensible toolkit for sequence
  modeling\BBCQ\
\newblock In Ammar, W., Louis, A., \BBA\ Mostafazadeh, N.\BEDS, {\Bem
  Proceedings of the 2019 Conference of the North American Chapter of the
  Association for Computational Linguistics: Human Language Technologies,
  {NAACL-HLT} 2019, Minneapolis, MN, USA, June 2-7, 2019, Demonstrations},
  \BPGS\ 48--53. Association for Computational Linguistics.

\bibitem[\protect\BCAY{Pennington, Socher,\ \BBA\ Manning}{Pennington
  et~al.}{2014}]{pennington2014glove}
Pennington, J., Socher, R., \BBA\ Manning, C.~D. \BBOP2014\BBCP.
\newblock \BBOQ Glove: Global vectors for word representation\BBCQ\
\newblock In Moschitti, A., Pang, B., \BBA\ Daelemans, W.\BEDS, {\Bem
  Proceedings of the 2014 Conference on Empirical Methods in Natural Language
  Processing, {EMNLP} 2014, October 25-29, 2014, Doha, Qatar, {A} meeting of
  SIGDAT, a Special Interest Group of the {ACL}}, \BPGS\ 1532--1543. {ACL}.

\bibitem[\protect\BCAY{Peters, Neumann, Iyyer, Gardner, Clark, Lee,\ \BBA\
  Zettlemoyer}{Peters et~al.}{2018}]{peters2018deep}
Peters, M.~E., Neumann, M., Iyyer, M., Gardner, M., Clark, C., Lee, K., \BBA\
  Zettlemoyer, L. \BBOP2018\BBCP.
\newblock \BBOQ Deep contextualized word representations\BBCQ\
\newblock In Walker, M.~A., Ji, H., \BBA\ Stent, A.\BEDS, {\Bem Proceedings of
  the 2018 Conference of the North American Chapter of the Association for
  Computational Linguistics: Human Language Technologies, {NAACL-HLT} 2018, New
  Orleans, Louisiana, USA, June 1-6, 2018, Volume 1 (Long Papers)}, \BPGS\
  2227--2237. Association for Computational Linguistics.

\bibitem[\protect\BCAY{Petroni, Rockt{\"a}schel, Riedel, Lewis, Bakhtin, Wu,\
  \BBA\ Miller}{Petroni et~al.}{2019a}]{petroni2019}
Petroni, F., Rockt{\"a}schel, T., Riedel, S., Lewis, P., Bakhtin, A., Wu, Y.,
  \BBA\ Miller, A. \BBOP2019a\BBCP.
\newblock \BBOQ Language models as knowledge bases?\BBCQ\
\newblock In {\Bem Proceedings of the 2019 Conference on Empirical Methods in
  Natural Language Processing and the 9th International Joint Conference on
  Natural Language Processing (EMNLP-IJCNLP)}, \BPGS\ 2463--2473, Hong Kong,
  China. Association for Computational Linguistics.

\bibitem[\protect\BCAY{Petroni, Rockt{\"{a}}schel, Riedel, Lewis, Bakhtin, Wu,\
  \BBA\ Miller}{Petroni et~al.}{2019b}]{petroni2019language}
Petroni, F., Rockt{\"{a}}schel, T., Riedel, S., Lewis, P. S.~H., Bakhtin, A.,
  Wu, Y., \BBA\ Miller, A.~H. \BBOP2019b\BBCP.
\newblock \BBOQ Language models as knowledge bases?\BBCQ\
\newblock In Inui, K., Jiang, J., Ng, V., \BBA\ Wan, X.\BEDS, {\Bem Proceedings
  of the 2019 Conference on Empirical Methods in Natural Language Processing
  and the 9th International Joint Conference on Natural Language Processing,
  {EMNLP-IJCNLP} 2019, Hong Kong, China, November 3-7, 2019}, \BPGS\
  2463--2473. Association for Computational Linguistics.

\bibitem[\protect\BCAY{Pimentel, Saphra, Williams,\ \BBA\ Cotterell}{Pimentel
  et~al.}{2020a}]{DBLP:conf/emnlp/PimentelSWC20}
Pimentel, T., Saphra, N., Williams, A., \BBA\ Cotterell, R. \BBOP2020a\BBCP.
\newblock \BBOQ Pareto probing: Trading off accuracy for complexity\BBCQ\
\newblock In Webber, B., Cohn, T., He, Y., \BBA\ Liu, Y.\BEDS, {\Bem
  Proceedings of the 2020 Conference on Empirical Methods in Natural Language
  Processing, {EMNLP} 2020, Online, November 16-20, 2020}, \BPGS\ 3138--3153.
  Association for Computational Linguistics.

\bibitem[\protect\BCAY{Pimentel, Valvoda, Maudslay, Zmigrod, Williams,\ \BBA\
  Cotterell}{Pimentel et~al.}{2020b}]{pimentel2020information}
Pimentel, T., Valvoda, J., Maudslay, R.~H., Zmigrod, R., Williams, A., \BBA\
  Cotterell, R. \BBOP2020b\BBCP.
\newblock \BBOQ Information-theoretic probing for linguistic structure\BBCQ\
\newblock In Jurafsky, D., Chai, J., Schluter, N., \BBA\ Tetreault, J.~R.\BEDS,
  {\Bem Proceedings of the 58th Annual Meeting of the Association for
  Computational Linguistics, {ACL} 2020, Online, July 5-10, 2020}, \BPGS\
  4609--4622. Association for Computational Linguistics.

\bibitem[\protect\BCAY{Prasanna, Rogers,\ \BBA\ Rumshisky}{Prasanna
  et~al.}{2020}]{DBLP:conf/emnlp/PrasannaRR20}
Prasanna, S., Rogers, A., \BBA\ Rumshisky, A. \BBOP2020\BBCP.
\newblock \BBOQ When {BERT} plays the lottery, all tickets are winning\BBCQ\
\newblock In Webber, B., Cohn, T., He, Y., \BBA\ Liu, Y.\BEDS, {\Bem
  Proceedings of the 2020 Conference on Empirical Methods in Natural Language
  Processing, {EMNLP} 2020, Online, November 16-20, 2020}, \BPGS\ 3208--3229.
  Association for Computational Linguistics.

\bibitem[\protect\BCAY{Press, Teukolsky, Vetterling,\ \BBA\ Flannery}{Press
  et~al.}{2007}]{press2007conditional}
Press, W., Teukolsky, S., Vetterling, W., \BBA\ Flannery, B. \BBOP2007\BBCP.
\newblock \BBOQ Conditional entropy and mutual information\BBCQ\
\newblock In {\Bem Numerical Recipes: The Art of Scientific Computing}.
  Cambridge University Press.

\bibitem[\protect\BCAY{Pruksachatkun, Yeres, Liu, Phang, Htut, Wang, Tenney,\
  \BBA\ Bowman}{Pruksachatkun et~al.}{2020}]{DBLP:conf/acl/PruksachatkunYL20}
Pruksachatkun, Y., Yeres, P., Liu, H., Phang, J., Htut, P.~M., Wang, A.,
  Tenney, I., \BBA\ Bowman, S.~R. \BBOP2020\BBCP.
\newblock \BBOQ jiant: {A} software toolkit for research on general-purpose
  text understanding models\BBCQ\
\newblock In {\c{C}}elikyilmaz, A.\BBACOMMA\  \BBA\ Wen, T.\BEDS, {\Bem
  Proceedings of the 58th Annual Meeting of the Association for Computational
  Linguistics: System Demonstrations, {ACL} 2020, Online, July 5-10, 2020},
  \BPGS\ 109--117. Association for Computational Linguistics.

\bibitem[\protect\BCAY{Qi, Zhang, Zhang, Bolton,\ \BBA\ Manning}{Qi
  et~al.}{2020}]{qi2020stanza}
Qi, P., Zhang, Y., Zhang, Y., Bolton, J., \BBA\ Manning, C.~D. \BBOP2020\BBCP.
\newblock \BBOQ Stanza: {A} python natural language processing toolkit for many
  human languages\BBCQ\
\newblock In {\c{C}}elikyilmaz, A.\BBACOMMA\  \BBA\ Wen, T.\BEDS, {\Bem
  Proceedings of the 58th Annual Meeting of the Association for Computational
  Linguistics: System Demonstrations, {ACL} 2020, Online, July 5-10, 2020},
  \BPGS\ 101--108. Association for Computational Linguistics.

\bibitem[\protect\BCAY{Ravfogel, Elazar, Gonen, Twiton,\ \BBA\
  Goldberg}{Ravfogel et~al.}{2020}]{ravfogel-etal-2020-null}
Ravfogel, S., Elazar, Y., Gonen, H., Twiton, M., \BBA\ Goldberg, Y.
  \BBOP2020\BBCP.
\newblock \BBOQ Null it out: Guarding protected attributes by iterative
  nullspace projection\BBCQ\
\newblock In Jurafsky, D., Chai, J., Schluter, N., \BBA\ Tetreault, J.~R.\BEDS,
  {\Bem Proceedings of the 58th Annual Meeting of the Association for
  Computational Linguistics, {ACL} 2020, Online, July 5-10, 2020}, \BPGS\
  7237--7256. Association for Computational Linguistics.

\bibitem[\protect\BCAY{Ravichander, Belinkov,\ \BBA\ Hovy}{Ravichander
  et~al.}{2020}]{ravichander2020probing}
Ravichander, A., Belinkov, Y., \BBA\ Hovy, E.~H. \BBOP2020\BBCP.
\newblock \BBOQ Probing the probing paradigm: Does probing accuracy entail task
  relevance?\BBCQ\
\newblock {\Bem CoRR}, {\Bem abs/2005.00719}.

\bibitem[\protect\BCAY{Reif, Yuan, Wattenberg, Vi{\'{e}}gas, Coenen, Pearce,\
  \BBA\ Kim}{Reif et~al.}{2019}]{reif2019visualizing}
Reif, E., Yuan, A., Wattenberg, M., Vi{\'{e}}gas, F.~B., Coenen, A., Pearce,
  A., \BBA\ Kim, B. \BBOP2019\BBCP.
\newblock \BBOQ Visualizing and measuring the geometry of {BERT}\BBCQ\
\newblock In Wallach, H.~M., Larochelle, H., Beygelzimer, A.,
  d'Alch{\'{e}}{-}Buc, F., Fox, E.~B., \BBA\ Garnett, R.\BEDS, {\Bem Advances
  in Neural Information Processing Systems 32: Annual Conference on Neural
  Information Processing Systems 2019, NeurIPS 2019, December 8-14, 2019,
  Vancouver, BC, Canada}, \BPGS\ 8592--8600.

\bibitem[\protect\BCAY{Rogers, Kovaleva,\ \BBA\ Rumshisky}{Rogers
  et~al.}{2020}]{rogers2020primer}
Rogers, A., Kovaleva, O., \BBA\ Rumshisky, A. \BBOP2020\BBCP.
\newblock \BBOQ A primer in bertology: What we know about how {BERT}
  works\BBCQ\
\newblock {\Bem CoRR}, {\Bem abs/2002.12327}.

\bibitem[\protect\BCAY{Sahlgren\ \BBA\ Carlsson}{Sahlgren\ \BBA\
  Carlsson}{2021}]{sahlgren2021}
Sahlgren, M.\BBACOMMA\  \BBA\ Carlsson, F. \BBOP2021\BBCP.
\newblock \BBOQ The singleton fallacy: Why current critiques of language models
  miss the point\BBCQ\
\newblock {\Bem CoRR}, {\Bem abs/2102.04310}.

\bibitem[\protect\BCAY{Sanh, Wolf,\ \BBA\ Rush}{Sanh
  et~al.}{2020}]{sanh2020movement}
Sanh, V., Wolf, T., \BBA\ Rush, A.~M. \BBOP2020\BBCP.
\newblock \BBOQ Movement pruning: Adaptive sparsity by fine-tuning\BBCQ\
\newblock {\Bem CoRR}, {\Bem abs/2005.07683}.

\bibitem[\protect\BCAY{Saunshi, Malladi,\ \BBA\ Arora}{Saunshi
  et~al.}{2021}]{saunshi2021a}
Saunshi, N., Malladi, S., \BBA\ Arora, S. \BBOP2021\BBCP.
\newblock \BBOQ A mathematical exploration of why language models help solve
  downstream tasks\BBCQ\
\newblock In {\Bem International Conference on Learning Representations}.

\bibitem[\protect\BCAY{Schluter}{Schluter}{2018}]{DBLP:conf/naacl/Schluter18}
Schluter, N. \BBOP2018\BBCP.
\newblock \BBOQ The word analogy testing caveat\BBCQ\
\newblock In Walker, M.~A., Ji, H., \BBA\ Stent, A.\BEDS, {\Bem Proceedings of
  the 2018 Conference of the North American Chapter of the Association for
  Computational Linguistics: Human Language Technologies, NAACL-HLT, New
  Orleans, Louisiana, USA, June 1-6, 2018, Volume 2 (Short Papers)}, \BPGS\
  242--246. Association for Computational Linguistics.

\bibitem[\protect\BCAY{Shi, Padhi,\ \BBA\ Knight}{Shi
  et~al.}{2016}]{shi-etal-2016-string}
Shi, X., Padhi, I., \BBA\ Knight, K. \BBOP2016\BBCP.
\newblock \BBOQ Does string-based neural {MT} learn source syntax?\BBCQ\
\newblock In Su, J., Carreras, X., \BBA\ Duh, K.\BEDS, {\Bem Proceedings of the
  2016 Conference on Empirical Methods in Natural Language Processing, {EMNLP}
  2016, Austin, Texas, USA, November 1-4, 2016}, \BPGS\ 1526--1534. The
  Association for Computational Linguistics.

\bibitem[\protect\BCAY{Silveira, Dozat, de~Marneffe, Bowman, Connor, Bauer,\
  \BBA\ Manning}{Silveira et~al.}{2014}]{silveira14gold}
Silveira, N., Dozat, T., de~Marneffe, M., Bowman, S.~R., Connor, M., Bauer, J.,
  \BBA\ Manning, C.~D. \BBOP2014\BBCP.
\newblock \BBOQ A gold standard dependency corpus for english\BBCQ\
\newblock In Calzolari, N., Choukri, K., Declerck, T., Loftsson, H., Maegaard,
  B., Mariani, J., Moreno, A., Odijk, J., \BBA\ Piperidis, S.\BEDS, {\Bem
  Proceedings of the Ninth International Conference on Language Resources and
  Evaluation, {LREC} 2014, Reykjavik, Iceland, May 26-31, 2014}, \BPGS\
  2897--2904. European Language Resources Association {(ELRA)}.

\bibitem[\protect\BCAY{Stolcke}{Stolcke}{2002}]{DBLP:conf/interspeech/Stolcke02}
Stolcke, A. \BBOP2002\BBCP.
\newblock \BBOQ {SRILM} - an extensible language modeling toolkit\BBCQ\
\newblock In Hansen, J. H.~L.\BBACOMMA\  \BBA\ Pellom, B.~L.\BEDS, {\Bem 7th
  International Conference on Spoken Language Processing, {ICSLP2002} -
  {INTERSPEECH} 2002, Denver, Colorado, USA, September 16-20, 2002}. {ISCA}.

\bibitem[\protect\BCAY{Tamkin, Singh, Giovanardi,\ \BBA\ Goodman}{Tamkin
  et~al.}{2020}]{DBLP:conf/emnlp/TamkinSGG20}
Tamkin, A., Singh, T., Giovanardi, D., \BBA\ Goodman, N.~D. \BBOP2020\BBCP.
\newblock \BBOQ Investigating transferability in pretrained language
  models\BBCQ\
\newblock In Cohn, T., He, Y., \BBA\ Liu, Y.\BEDS, {\Bem Proceedings of the
  2020 Conference on Empirical Methods in Natural Language Processing:
  Findings, {EMNLP} 2020, Online Event, 16-20 November 2020}, \BPGS\
  1393--1401. Association for Computational Linguistics.

\bibitem[\protect\BCAY{Tenney, Das,\ \BBA\ Pavlick}{Tenney
  et~al.}{2019a}]{tenney2019bert}
Tenney, I., Das, D., \BBA\ Pavlick, E. \BBOP2019a\BBCP.
\newblock \BBOQ {BERT} rediscovers the classical {NLP} pipeline\BBCQ\
\newblock In Korhonen, A., Traum, D.~R., \BBA\ M{\`{a}}rquez, L.\BEDS, {\Bem
  Proceedings of the 57th Conference of the Association for Computational
  Linguistics, {ACL} 2019, Florence, Italy, July 28- August 2, 2019, Volume 1:
  Long Papers}, \BPGS\ 4593--4601. Association for Computational Linguistics.

\bibitem[\protect\BCAY{Tenney, Xia, Chen, Wang, Poliak, McCoy, Kim, Durme,
  Bowman, Das,\ \BBA\ Pavlick}{Tenney et~al.}{2019b}]{tenney2018what}
Tenney, I., Xia, P., Chen, B., Wang, A., Poliak, A., McCoy, R.~T., Kim, N.,
  Durme, B.~V., Bowman, S.~R., Das, D., \BBA\ Pavlick, E. \BBOP2019b\BBCP.
\newblock \BBOQ What do you learn from context? probing for sentence structure
  in contextualized word representations\BBCQ\
\newblock In {\Bem 7th International Conference on Learning Representations,
  {ICLR} 2019, New Orleans, LA, USA, May 6-9, 2019}. OpenReview.net.

\bibitem[\protect\BCAY{Tian, Okazaki,\ \BBA\ Inui}{Tian
  et~al.}{2017}]{DBLP:journals/ml/TianOI17}
Tian, R., Okazaki, N., \BBA\ Inui, K. \BBOP2017\BBCP.
\newblock \BBOQ The mechanism of additive composition\BBCQ\
\newblock {\Bem Mach. Learn.}, {\Bem 106\/}(7), 1083--1130.

\bibitem[\protect\BCAY{Tissier, Gravier,\ \BBA\ Habrard}{Tissier
  et~al.}{2019}]{DBLP:conf/aaai/TissierGH19}
Tissier, J., Gravier, C., \BBA\ Habrard, A. \BBOP2019\BBCP.
\newblock \BBOQ Near-lossless binarization of word embeddings\BBCQ\
\newblock In {\Bem The Thirty-Third {AAAI} Conference on Artificial
  Intelligence, {AAAI} 2019, The Thirty-First Innovative Applications of
  Artificial Intelligence Conference, {IAAI} 2019, The Ninth {AAAI} Symposium
  on Educational Advances in Artificial Intelligence, {EAAI} 2019, Honolulu,
  Hawaii, USA, January 27 - February 1, 2019}, \BPGS\ 7104--7111. {AAAI} Press.

\bibitem[\protect\BCAY{Vaswani, Shazeer, Parmar, Uszkoreit, Jones, Gomez,
  Kaiser,\ \BBA\ Polosukhin}{Vaswani et~al.}{2017}]{vaswani2017attention}
Vaswani, A., Shazeer, N., Parmar, N., Uszkoreit, J., Jones, L., Gomez, A.~N.,
  Kaiser, L., \BBA\ Polosukhin, I. \BBOP2017\BBCP.
\newblock \BBOQ Attention is all you need\BBCQ\
\newblock In Guyon, I., von Luxburg, U., Bengio, S., Wallach, H.~M., Fergus,
  R., Vishwanathan, S. V.~N., \BBA\ Garnett, R.\BEDS, {\Bem Advances in Neural
  Information Processing Systems 30: Annual Conference on Neural Information
  Processing Systems 2017, December 4-9, 2017, Long Beach, CA, {USA}}, \BPGS\
  5998--6008.

\bibitem[\protect\BCAY{Voita\ \BBA\ Titov}{Voita\ \BBA\
  Titov}{2020}]{voita2020information}
Voita, E.\BBACOMMA\  \BBA\ Titov, I. \BBOP2020\BBCP.
\newblock \BBOQ Information-theoretic probing with minimum description
  length\BBCQ\
\newblock In Webber, B., Cohn, T., He, Y., \BBA\ Liu, Y.\BEDS, {\Bem
  Proceedings of the 2020 Conference on Empirical Methods in Natural Language
  Processing, {EMNLP} 2020, Online, November 16-20, 2020}, \BPGS\ 183--196.
  Association for Computational Linguistics.

\bibitem[\protect\BCAY{Weischedel, Pradhan, Ramshaw, Palmer, Xue, Marcus,
  Taylor, Greenberg, Hovy, Belvin, et~al.}{Weischedel
  et~al.}{2013}]{weischedel2013ontonotes}
Weischedel, R., Pradhan, S., Ramshaw, L., Palmer, M., Xue, N., Marcus, M.,
  Taylor, A., Greenberg, C., Hovy, E., Belvin, R., et~al. \BBOP2013\BBCP.
\newblock \BBOQ Ontonotes release 5.0\BBCQ\
\newblock LDC2013T19, Philadelphia, Penn.: Linguistic Data Consortium.

\bibitem[\protect\BCAY{Wolf, Debut, Sanh, Chaumond, Delangue, Moi, Cistac,
  Rault, Louf, Funtowicz, Davison, Shleifer, von Platen, Ma, Jernite, Plu, Xu,
  Scao, Gugger, Drame, Lhoest,\ \BBA\ Rush}{Wolf
  et~al.}{2020}]{DBLP:conf/emnlp/WolfDSCDMCRLFDS20}
Wolf, T., Debut, L., Sanh, V., Chaumond, J., Delangue, C., Moi, A., Cistac, P.,
  Rault, T., Louf, R., Funtowicz, M., Davison, J., Shleifer, S., von Platen,
  P., Ma, C., Jernite, Y., Plu, J., Xu, C., Scao, T.~L., Gugger, S., Drame, M.,
  Lhoest, Q., \BBA\ Rush, A.~M. \BBOP2020\BBCP.
\newblock \BBOQ Transformers: State-of-the-art natural language
  processing\BBCQ\
\newblock In Liu, Q.\BBACOMMA\  \BBA\ Schlangen, D.\BEDS, {\Bem Proceedings of
  the 2020 Conference on Empirical Methods in Natural Language Processing:
  System Demonstrations, {EMNLP} 2020 - Demos, Online, November 16-20, 2020},
  \BPGS\ 38--45. Association for Computational Linguistics.

\bibitem[\protect\BCAY{Wu, Chen, Kao,\ \BBA\ Liu}{Wu
  et~al.}{2020}]{DBLP:conf/acl/WuCKL20}
Wu, Z., Chen, Y., Kao, B., \BBA\ Liu, Q. \BBOP2020\BBCP.
\newblock \BBOQ Perturbed masking: Parameter-free probing for analyzing and
  interpreting {BERT}\BBCQ\
\newblock In Jurafsky, D., Chai, J., Schluter, N., \BBA\ Tetreault, J.~R.\BEDS,
  {\Bem Proceedings of the 58th Annual Meeting of the Association for
  Computational Linguistics, {ACL} 2020, Online, July 5-10, 2020}, \BPGS\
  4166--4176. Association for Computational Linguistics.

\bibitem[\protect\BCAY{Yang, Dai, Salakhutdinov,\ \BBA\ Cohen}{Yang
  et~al.}{2018}]{yang2017breaking}
Yang, Z., Dai, Z., Salakhutdinov, R., \BBA\ Cohen, W.~W. \BBOP2018\BBCP.
\newblock \BBOQ Breaking the softmax bottleneck: {A} high-rank {RNN} language
  model\BBCQ\
\newblock In {\Bem 6th International Conference on Learning Representations,
  {ICLR} 2018, Vancouver, BC, Canada, April 30 - May 3, 2018, Conference Track
  Proceedings}. OpenReview.net.

\bibitem[\protect\BCAY{Zhang, Warstadt, Li,\ \BBA\ Bowman}{Zhang
  et~al.}{2020}]{zhang2020billions}
Zhang, Y., Warstadt, A., Li, H., \BBA\ Bowman, S.~R. \BBOP2020\BBCP.
\newblock \BBOQ When do you need billions of words of pretraining data?\BBCQ\
\newblock {\Bem CoRR}, {\Bem abs/2011.04946}.

\bibitem[\protect\BCAY{Zhao, Lin, Mi, Jaggi,\ \BBA\ Sch{\"u}tze}{Zhao
  et~al.}{2020}]{zhao-etal-2020-masking}
Zhao, M., Lin, T., Mi, F., Jaggi, M., \BBA\ Sch{\"u}tze, H. \BBOP2020\BBCP.
\newblock \BBOQ Masking as an efficient alternative to finetuning for
  pretrained language models\BBCQ\
\newblock In {\Bem Proceedings of the 2020 Conference on Empirical Methods in
  Natural Language Processing (EMNLP)}, \BPGS\ 2226--2241, Online. Association
  for Computational Linguistics.

\bibitem[\protect\BCAY{Zhu\ \BBA\ Rudzicz}{Zhu\ \BBA\
  Rudzicz}{2020}]{DBLP:conf/emnlp/ZhuR20}
Zhu, Z.\BBACOMMA\  \BBA\ Rudzicz, F. \BBOP2020\BBCP.
\newblock \BBOQ An information theoretic view on selecting linguistic
  probes\BBCQ\
\newblock In Webber, B., Cohn, T., He, Y., \BBA\ Liu, Y.\BEDS, {\Bem
  Proceedings of the 2020 Conference on Empirical Methods in Natural Language
  Processing, {EMNLP} 2020, Online, November 16-20, 2020}, \BPGS\ 9251--9262.
  Association for Computational Linguistics.

\end{thebibliography}
\bibliographystyle{theapa}

\end{document}